\pdfoutput=1
\documentclass[11pt]{article}
\usepackage[utf8]{inputenc}
\usepackage[left=2.5cm,right=2.5cm,vmargin=2.5cm]{geometry}

\title{\textbf{A Spectral Method for Joint Community Detection and Orthogonal Group Synchronization}}
\author{Yifeng Fan\thanks{Department of Electrical and Computer Engineering,  University of Illinois at Urbana-Champaign, Champaign, IL.}
	\and Yuehaw Khoo\thanks{Department of Statistics, University of Chicago, Chicago, IL.}
	\and Zhizhen Zhao\footnotemark[1]}
\date{}

\usepackage{graphicx}
\usepackage{amssymb}
\usepackage{amsmath}
\usepackage{amsfonts}
\usepackage{subfig}
\usepackage{multirow}
\usepackage{enumerate}
\usepackage{amsthm}
\usepackage{enumitem}
\usepackage{color}
\usepackage{booktabs}
\usepackage{epstopdf}
\usepackage{tablefootnote}
\usepackage{tabularx}
\usepackage[linesnumbered, ruled,vlined]{algorithm2e}
\usepackage{array}
\usepackage{bm}
\usepackage{caption}
\usepackage{subfig}
\usepackage{float}
\usepackage{hyperref}
\SetKwInput{kwInit}{Initialize}

\newtheorem{theorem}{Theorem}[section]
\newtheorem{lemma}[theorem]{Lemma}
\newtheorem{remark}{Remark}
\theoremstyle{definition}
\newtheorem{definition}{Definition}[section]
\newtheorem{assumption}{Assumption}

\newcommand{\SO}{\mathrm{SO}}
\newcommand{\mO}{\mathrm{O}}
\DeclareMathOperator*{\argmax}{arg\,max}
\DeclareMathOperator*{\argmin}{arg\,min}

\begin{document}
	
	\maketitle
	
	\begin{abstract}
		Community detection and orthogonal group synchronization are both fundamental problems with a variety of important applications in science and engineering. In this work, we consider the joint problem of community detection and orthogonal group synchronization which aims to recover the communities and perform synchronization simultaneously. To this end, we propose a simple algorithm that consists of a spectral decomposition step followed by a blockwise column pivoted QR factorization (CPQR). The proposed algorithm is efficient and scales linearly with the number of edges in the graph. We also leverage the recently developed `leave-one-out' technique to establish a near-optimal guarantee for exact recovery of the cluster memberships and stable recovery of the orthogonal transforms. Numerical experiments demonstrate the efficiency and efficacy of our algorithm and confirm our theoretical characterization of it.
	\end{abstract}
	
	
	\section{Introduction}
\label{sec:intro}

Community detection and synchronization are both fundamental problems in signal processing, machine learning, and computer vision. Recently, there is an increasing interest in their joint problem~\cite{fan2021joint, bajaj2018smac, lederman2019representation}. That is, in the presence of heterogeneous data where data points associated with random group elements (e.g. the orthogonal group $\text{O}(d)$ of dimension $d$) fall into multiple underlying clusters, the joint problem is to simultaneously recover the cluster structures as well as the group elements. A motivating example is the 2D class averaging process in \textit{cryo-electron microscopy single particle reconstruction}~\cite{frank2006,singer2011viewing,zhao2014rotationally}, whose goal is to align (with $\text{SO}(2)$ group synchronization) and average projection images of a single particle with similar viewing angles to improve their signal-to-noise ratio (SNR). Another application in computer vision is \textit{simultaneous permutation group synchronization and clustering} on heterogeneous object collections consisting of 2D images or 3D shapes~\cite{bajaj2018smac}. 

\begin{figure}[t!]
    \centering
    \includegraphics[width = 0.65\textwidth]{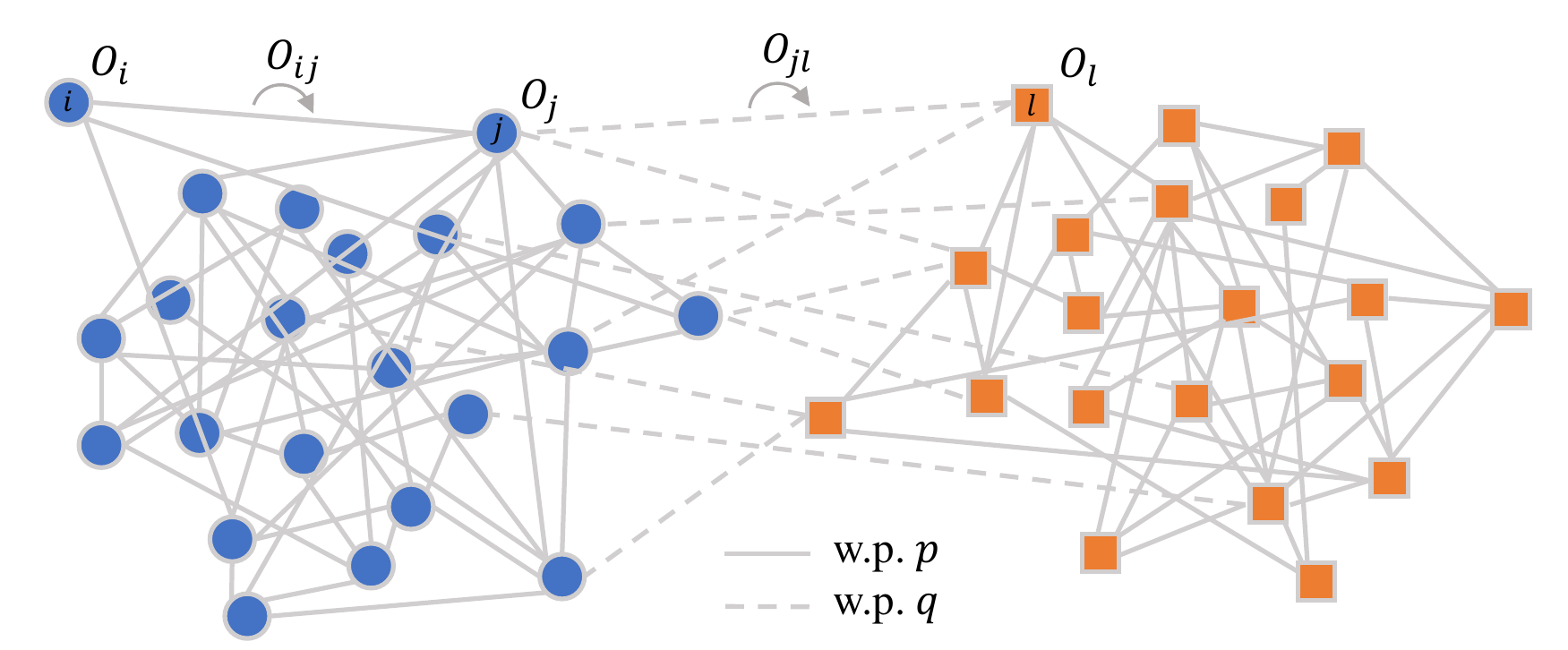}
    \caption{We present a network with two communities shown in circles and squares respectively. Each node is associated with an orthogonal group element. Each pair of nodes within the same cluster (resp.~across clusters) are independently connected with probability $p$ (resp.~$q$) as shown in solid (resp.~dash) lines. Also, a pairwise alignment $\bm{O}_{ij}$ is observed on each edge $(i,j)$.}
    \label{fig:intro_model}
\end{figure}

In this work, we study the joint problem based on the probabilistic model introduced in~\cite{fan2021joint} which extends the celebrated stochastic block model (SBM)~\cite{decelle2011asymptotic, doreian2005generalized, dyer1989solution, fienberg1985statistical, holland1983stochastic, karrer2011stochastic, massoulie2014community, mcsherry2001spectral, mossel2012stochastic, mossel2018proof} for community detection (see Fig.~\ref{fig:intro_model} for an illustration, which is slightly modified based on~\cite[Fig. 1]{fan2021joint}). In particular, we focus on the orthogonal group $\text{O}(d)$ that covers a wide range of applications mentioned above. Formally, given a network of $n$ nodes (data points) with $K$ underlying disjoint communities, each node $i$ is additionally associated with an unknown orthogonal group element $\bm{O}_i \in \text{O}(d)$. For each pair of nodes $(i,j)$, their orthogonal group transformation $\bm{O}_{ij}$ is independently observed with probability $p$ (resp. $q$) when node $i$ and node $j$ belong to the same cluster (resp. different clusters). In particular, the clean measurement $\bm{O}_{ij} = \bm{O}_i\bm{O}_j^\top$ is obtained if $i$ and $j$ are in the same cluster. Otherwise, $\bm{O}_{ij}$ is uniformly drawn from $\text{O}(d)$, implying the measurement is completely noisy. Notably, such model of corruption is widely considered in synchronization (e.g.~\cite{singer2011viewing,fan2019representation,fan2019multi, fan2019unsupervised, ling2020near}).

Under this probabilistic setting, we want to simultaneously recover the clusters and group elements by combining the idea for community detection and synchronization. That is, an optimization program can be formulated that maximizes not only the edge connections within each identified cluster, but also the consistency on the observed group transformations within each cluster. 
{\color{black}
As shown in~\cite{fan2021joint}, joint optimization approach is better than solving the clustering and synchronization problems in two stages: first solving a graph-clustering problem based on the graph connectivity and then synchronizing within each identified cluster.
} 
Directly solving such programs is usually NP-hard and computationally intractable, which gives rise to the convex relaxation methods such as semidefinite relaxation studied in \cite{fan2021joint} or spectral method in \cite{bajaj2018smac} that yield approximate solutions with polynomial time complexity. The algorithm proposed in this work is also based on a spectral method which first computes the top eigenvectors of an $n \times n$ block matrix of observed data. Then, different from \cite{bajaj2018smac}, a blockwise column-pivoted QR-factorization is performed on the top eigenvectors to identify the cluster structure and orthogonal group elements, which scales linearly with the number of data points. As a result, our method is able to achieve competitive performance with lower computational cost compared to \cite{bajaj2018smac}.

\subsection{Related work and our contributions}
Given the practical importance to a variety of applications, either community detection or group synchronization has been extensively studied over the past decades. Due to the vast volume of literature, we are not able to present a complete review of all previous works but only highlight the most related ones to this work. Community detection aims to find the underlying communities within a network by using the network topology information. It is commonly studied under the stochastic block model (SBM)~\cite{dyer1989solution, fienberg1985statistical, holland1983stochastic} where obtaining the maximum likelihood estimator for clustering is often NP-hard. Therefore, different approaches such as semidefinite programming~(SDP)~\cite{abbe2015exact, hajek2016achievinga, hajek2016achievingb, perry2017semidefinite, guedon2016community, amini2018semidefinite, bandeira2018random}, spectral method~\cite{abbe2020entrywise, vu2014simple, yun2014accurate, massoulie2014community, krzakala2013spectral, ng2002spectral}, and belief propagation~\cite{decelle2011asymptotic, abbe2015community} are considered for a practical solution. 
{\color{black}
In particular, semidefinite relaxation generally yields the state-of-the-art performance as it is able to achieve the information theoretic limits of SBM~\cite{abbe2015exact, abbe2015community, perry2017semidefinite, hajek2016achievinga, hajek2016achievingb}. However, solving large-scale SDPs is still computationally expensive. In contrast, algorithms based on spectral method are more efficient and also gives competitive result (e.g., achieving the information theoretic limits in the case of two equal-sized clusters~\cite{abbe2020entrywise, yun2014accurate, gao2017achieving}). This motivates the proposed method, which extends the spectral method in \cite{damle2016robust} for clustering, to solve our joint problem.
}

On the other hand, group synchronization wants to recover the underlying group elements $\{g_i\}_{i = 1}^n$ from a set of noisy pairwise measurements $\{g_i^{-1}g_j\}$. A common approach is the least square estimator which is usually NP-hard. Instead, similar to the development of community detection algorithms, convex relaxations such as semidefinite relaxation~\cite{singer2011angular, huang2013consistent} and spectral methods~\cite{singer2011angular, arrigoni2016spectral, chaudhury2015global, pachauri2013solving, shen2016normalized, gao2019multi} have shown to be powerful, along with many investigations of their theoretical properties~\cite{singer2011angular, zhong2018near, ling2020near, huang2013consistent, ling2020solving}. Again, spectral method based algorithms are generally more favorable and appealing than SDPs due to its computational efficiency.

The joint community detection and synchronization is a relatively new topic, motivated by recent scientific applications such as cryo-electron microscopy as mentioned before.  In~\cite{bajaj2018smac}, the authors addressed simultaneous permutation group synchronization and clustering via a spectral method for simultaneously mapping and clustering 3D object volumes. In~\cite{lederman2019representation}, as motivated by the cryo-EM single particle reconstruction problem, the authors proposed a harmonic analysis and SDP based approach for solving the rotational alignment and classification of 2D projection images simultaneously. The recent work~\cite{fan2021joint} by the authors of this paper proposed several SDPs and gave theoretical conditions for exact recovery for the probabilistic model considered. In this work, we propose an alternative spectral method based algorithm, which greatly reduces the computational complexity of SDP and obtains competitive performance compared to the existing methods.

Besides the algorithm itself, a significant contribution of this work is to provide a near-optimal performance guarantee for exact recovery under the probabilistic model. This requires analyzing the perturbation of the eigenvectors of a low rank matrix corrupted by random noise. Such problem is a classic topic in matrix perturbation theory~\cite{stewart1998perturbation, tao2012topics, anderson2010introduction}, where a naive $\ell_2$ or Frobenius norm error bound can be easily obtained by Davis-Kahan theorem~\cite{davis1970rotation}. However, 
such result is not sufficient for exact recovery since it measures the ``average'' error, and if the error is concentrated on several entries (or blocks if each node is represented by a matrix instead of a scalar), exact recovery is not guaranteed. Instead, an $\ell_\infty$ norm 
type error bound is necessary for exact recovery since it bounds the error entrywisely (or blockwisely). Fortunately, in the past years we have seen a surge of developments on $\ell_\infty$ norm bounds. Most of them~(e.g. \cite{abbe2020entrywise, zhong2018near, deng2020strong, chen2019spectral, fan2018eigenvector, ling2020near}) are based on the \textit{leave-one-out technique} proposed in~\cite{abbe2020entrywise, zhong2018near}.
In particular, our analysis greatly benefits from~\cite{abbe2020entrywise}, which provides an entrywise error bound of the leading eigenvector of low-rank matrices, and also \cite{ling2020near}, which extends~\cite{abbe2020entrywise} to consider block matrices and gives a blockwise bound on multiple eigenvectors. Results by other approaches exist such as~\cite{eldridge2018unperturbed, damle2020uniform}. Specifically, \cite{damle2020uniform} introduces deterministic row-wise perturbation bounds for orthonormal bases of invariant subspaces of symmetric matrices. Such bounds can be applied to general forms of the perturbation matrix. Compared to~\cite{damle2020uniform}, the leave-one-out technique can exploit the independence of random variables involved and thus achieves sharper bounds in our setting. Here, our contribution lies in handling the additional cluster structures and analyzing the QR factorization. 

We summarize our contributions in the following: (1) We introduce a novel algorithm for joint community detection and orthogonal group synchronization, which consists of three simple steps: a spectral decomposition followed by a \emph{blockwise} column-pivoted QR-factorization (CPQR), and a step for cluster assignment and group element recovery. (2) A variant of CPQR, called \textit{blockwise CPQR}, is designed to deal with the block matrix structure induced by the $\text{O}(d)$ group transformation. 
(3) Under the probabilistic model, a near-optimal performance guarantee is established for exact recovery of the clusters memberships and stable recovery of the orthogonal transforms. (4) We demonstrate the efficacy of our method and verify the theoretical characterization of the sharp phase transition for recovery, via a series of numerical experiments.

\subsection{Organization}
The rest of this paper is organized as follows: In Section~\ref{sec:pre} we introduce the probabilistic model and formulate the optimization program for the joint problem. Then in Section~\ref{sec:methods} we present our algorithm. Section~\ref{sec:analysis} is devoted to theoretical analysis. Numerical experiments are given in Section~\ref{sec:exp} for evaluating the performance and verifying our theory. We conclude with discussions and future directions in Section~\ref{sec:conclusion}. For clarity most of the technical proofs are deferred to Appendix.

	\subsection{Notations}
\label{sec:notation}
Throughout this paper we use the following notations: The transpose of a matrix $\bm{X}$ is denoted by $\bm{X}^\top$. An $m \times n$ matrix of all zeros is denoted by $\bm{0}_{m \times n}$ (or $\bm{0}$ for brevity). An identity matrix of size $n \times n$ is denoted by $\bm{I}_n$.  $\sigma_{\text{max}}(\bm{X})$, $\sigma_{\text{min}}(\bm{X})$, and $\sigma_{l}(\bm{X})$ stand for the maximum, the minimum, and the $l$-th largest singular value of $\bm{X}$ respectively. Similarly, $\lambda_l(\bm{X})$ denotes the $l$-th largest eigenvalue of $\bm{X}$.   $\|\bm{X}\| = \max_{\|v\| = 1}\|\bm{X}\bm{v}\|$ and $\|\bm{X}\|_\mathrm{F} = \sqrt{\text{Tr}(\bm{X}^\top\bm{X})}$ denote the operator norm and the Frobenius norm of $\bm{X}$ respectively. 

For a block matrix $\bm{X} \in \mathbb{R}^{md \times nd}$, without further specification, the $(i, j)$-th block is denoted by $\bm{X}_{ij} \in \mathbb{R}^{d \times d}$, for $i = 1, \dots, m$ and $j = 1, \dots, n$.  
In addition, the $i$-th \textit{block row} (resp. $j$-th \textit{block column}) of $\bm{X}$ is referred to as the sub-matrix that contains $\bm{X}_{ij}$ for all $j = 1, \ldots, n$ (resp. $i = 1, \ldots, m$), and is denoted as $\bm{X}_{i\cdot} \in \mathbb{R}^{d \times nd}$ (resp. $\bm{X}_{\cdot j} \in \mathbb{R}^{md \times d}$).

For two non-negative functions $f(n)$ and $g(n)$, $f(n) = O(g(n))$ or $f(n) \lesssim g(n)$ means there exists an absolute positive constant $C$ such that $f(n) \leq Cg(n)$ for all sufficiently large $n$; $f(n) = \Omega(g(n))$ or $f(n) \gtrsim g(n)$ means there exists an absolute positive constant $C$ such that $f(n) \geq Cg(n)$ for all sufficiently large $n$; and $f(n) = o(g(n))$ indicates for every positive constant $C$, the inequality $f(n) \leq Cg(n)$ holds for all sufficiently large $n$. 

\section{Preliminaries}
\label{sec:prelim}
We start from some important definitions for matrix factorization and decomposition that will be used for algorithm development and analysis. 
{\color{black}
\begin{definition}[Polar decomposition]
Given a squared matrix $\bm{X} \in \mathbb{R}^{d \times d}$, the \textup{polar decomposition} of $\bm{X}$ is given as 
\begin{equation}
    \bm{X} = \mathcal{P}(\bm{X})\bm{W}
    \label{eq:polar}
\end{equation}
where $\mathcal{P}(\bm{X}) \in \mathbb{R}^{d \times d}$ is orthogonal and $\bm{W} \in \mathbb{R}^{d \times d}$ is positive semi-definite.
\label{def:polar}
\end{definition}

Notably, such a decomposition always exists. Also, when $\bm{X}$ has full rank, $\mathcal{P}(\bm{X})$ is the closest orthogonal matrix to $\bm{X}$ such that $\mathcal{P}(\bm{X}) =  \argmin_{\bm{Y} \in \mO(d)} \|\bm{X} -\bm{Y}\|_\mathrm{F}$~(see \cite{fan1955some}), and $\bm{W}$ is guaranteed to be positive definite. In addition, by denoting $\bm{X} = \bm{U}\bm{\Sigma}\bm{V}^{\top}$ as its singular value decomposition, one can obtain $\mathcal{P}(\bm{X}) = \bm{U}\bm{V^\top}$ and $\bm{W} = \bm{V}\bm{\Sigma}\bm{V}^\top$, with computational cost $O(k^3)$. As a result, we also denote $\mathcal{P}(\bm{X})$ as the \textit{polar factor} of any matrix $\bm{X}$.

\begin{definition}[QR factorization]
Given any $\bm{X} \in \mathbb{R}^{m \times n}$, the \textup{QR factorization} of $\bm{X}$ is given as
\begin{equation*}
    \bm{X} = \bm{Q}\bm{R}
\end{equation*}
where $\bm{Q} \in \mathbb{R}^{m \times m}$ is orthogonal and $\bm{R} \in \mathbb{R}^{m \times n}$ is an upper triangular matrix.
\label{def:QR_def}
\end{definition}

Again, such a factorization always exists. In terms of computing it, the G\textit{ram-Schmidt process} and \textit{Householder transformation} are commonly used~(see e.g. 
\cite{trefethen1997numerical}), where the latter one enjoys better numerical stability.

\begin{definition}[Column-pivoted QR factorization]
Given any $\bm{X} \in \mathbb{R}^{m \times n}$ with $m \leq n$ and rank $m$, the \textup{column-pivoted QR factorization (CPQR)} of $\bm{X}$ is given as
\begin{equation*}
\bm{X} \bm{\Pi}_{n} = \bm{Q}\left[\bm{R}_1, \bm{R}_2\right] = \bm{Q}\bm{R}
\end{equation*}
where $\bm{\Pi}_n \in \mathbb{R}^{n \times n}$ is a permutation matrix, $\bm{Q} \in \mathbb{R}^{m \times m}$ is orthogonal, $\bm{R}_1 \in \mathbb{R}^{m \times m}$ is upper-triangular, and $\bm{R}_2 \in \mathbb{R}^{m \times (n-m)}$. 
\label{def:CPQR}
\end{definition}

The column-pivoted QR differs from the vanilla QR by introducing a permutation $\bm{\Pi}_n$ that ideally makes $\bm{R}_1$ as well-condition as possible given $\bm{X}$. In practice, the \textit{Golub-Businger algorithm}~\cite{businger1965linear}, which is based on the Householder transformation, chooses $\bm{\Pi}_n$ using a greedy heuristic: at each step, the column with the largest remaining norm in $\bm{X}$ is picked as the pivot for computing the new orthogonal basis in $\bm{Q}$. As a result, CPQR avoids selecting columns that are highly linearly dependent for determining $\bm{Q}$, which improves the numerical stability especially when $\bm{X}$ is rank deficient. Because of this, CPQR serves as the backbone of the so-called \textit{rank-revealing QR factorization}~\cite{gu1996efficient} which is used to determine the rank of a matrix.

In the following, we introduce an extension of CPQR which considers the case of a block matrix and perform the decomposition blockwisely. This serves as the core of our proposed algorithm in this paper.

\begin{definition}[blockwise CPQR] 
Given any $m \times n$ block matrix $\bm{X} \in \mathbb{R}^{md \times nd}$ with block size $d \times d$, $m < n$ and rank $md$ (full row rank), the \textup{blockwise CPQR} of $\bm{X}$ is given as $\bm{X}\bm{\Pi}_{nd} = \bm{Q}\left[\bm{R}_1, \bm{R}_2\right]$, where $\bm{\Pi}_{nd} \in \mathbb{R}^{nd \times nd}$ is a permutation matrix with a Kronecker product (denoted by $\otimes$) structure such that 
\begin{equation*}
    \bm{\Pi}_{nd} = \bm{\Pi}_n \otimes \bm{I}_d
\end{equation*}
for some permutation matrix $\bm{\Pi}_n \in \mathbb{R}^{n \times n}$, where $\bm{Q} \in \mathbb{R}^{md \times md}$ is orthogonal, $\bm{R}_1 \in \mathbb{R}^{md \times md}$ is upper-triangular, and $\bm{R}_2 \in \mathbb{R}^{md \times (n-m)d}$.
\label{def:block_CPQR}
\end{definition}

Here, one can view the blockwise CPQR as a special form of the vanilla CPQR, where the matrix $\bm{X}$ is decomposed blockwisely such that its block structure and the relative orders of columns within each block are always preserved, and the pivot becomes a block column instead. The details for computing it is deferred to Section~\ref{sec:methods}.
}

	\section{Problem setup}
\label{sec:pre}
Given a network with $n$ nodes and $K$ underlying communities, we assume each node $i$ has a clustering label $\kappa(i) \in \{1, \ldots, K\}$, and is associated with an orthogonal transform $\bm{O}_i \in \mO(d)$. We use $m_k$ and $C_k$ to denote the size of the $k$-th cluster and the set of nodes belonging to it respectively, such that $m_k = |C_k|$.

Formally, our probabilistic model generates a random graph $G = (V,E)$ with node set $V$ and edge set $E$. Each pair of nodes $(i,j)$ is independently connected with probability $p$ if $\kappa(j) = \kappa(i)$ that belong to the same community, otherwise they are connected with probability $q$ if $\kappa(j) \neq \kappa(i)$. Also, an orthogonal transformation $\bm{O}_{ij} \in \mO(d)$ is observed on each edge connection $(i,j) \in E$, when $\kappa(j) = \kappa(i)$ we obtain $\bm{O}_{ij} = \bm{O}_i\bm{O}_j^\top$ which equals to the true alignment from $j$ to $i$, otherwise we assume $\bm{O}_{ij} \sim \text{Unif}(\mO(d))$, which is a random orthogonal transformation uniformly drawn from $\mO(d)$ that carries no information but only noise. 

Given the above, our observation from the model can be represented by an \textit{observation matrix} $\bm{A} \in \mathbb{R}^{nd \times nd}$, which is an $n \times n$ symmetric block matrix whose $(i,j)$-th block $\bm{A}_{ij} \in \mathbb{R}^{d \times d}$ for any $i < j$ satisfies, 
\begin{equation}
  \bm{A}_{ij} = 
    \begin{cases}
    \bm{O}_i\bm{O}_j^\top,  &\; \text{with probability } p \text{ and when } \kappa(j) = \kappa(i), \\
    \bm{O}_{ij} \sim \text{Unif}(\mO(d)), &\; \text{with probability } q \text{ and when } \kappa(j) \neq \kappa(i),\\
    \bm{0}, &\; \text{otherwise}. 
    \end{cases} 
    \label{eq:clean_observation}
\end{equation}
Then $\bm{A}_{ij} = \bm{A}_{ji}^\top$, and we set the diagonal blocks $\bm{A}_{ii} = \bm{0}$ for $i = 1,\ldots, n$. In this way, unlike an adjacency matrix which only has $\{0, 1\}$-valued entries that indicate the connectivity, $\bm{A}$ defined in \eqref{eq:clean_observation} extends to include orthogonal transformation connected nodes as well. 
In addition, we define the \textit{clean observation matrix} $\bm{A}^{\text{clean}} \in \mathbb{R}^{nd \times nd}$, whose $(i,j)$-th block $\bm{A}_{ij}^{\text{clean}}$ satisfies
\begin{equation}
    \bm{A}_{ij}^{\text{clean}} = 
    \begin{cases}
    \bm{O}_i\bm{O}_j^\top, &\quad \kappa(j) = \kappa(i),\\
    \bm{0}, &\quad \text{otherwise}.
    \end{cases}
    \label{eq:clean_A}
\end{equation}
As a result, $\bm{A}^{\text{clean}}$ is equivalent to $\bm{A}$ in the clean case when $p = 1$ and $q = 0$. 

Given the observation $\bm{A}$, we formulate the following optimization program for recovery:
\begin{equation}
    \max_{\bm{\Theta}_i, \ldots, \bm{\Theta}_n \in \text{O}(d), \mathcal{C}_1, \ldots, \mathcal{C}_K} \sum_{k = 1}^K\sum_{i, j \in \mathcal{C}_k} \left\langle \bm{A}_{ij}, \; \frac{1}{|\mathcal{C}_k|}\bm{\Theta}_{i} \bm{\Theta}_j^\top \right\rangle,
    \label{eq:clustersync}
\end{equation}
where $\mathcal{C}_k$ denotes the set of nodes assigned to the $k$-th cluster, $\bm{\Theta}_i$ is the identified orthogonal transform for node $i$. As a result, \eqref{eq:clustersync} simultaneously recovers the cluster memberships and the orthogonal group elements, by maximizing not only the edge connectivity but also the consistency of transformations among nodes within each cluster. Notably, compared to the formulation in a previous work~\cite[Eq.~(5)]{fan2021joint}, the additional factor $1/|\mathcal{C}_k|$ in \eqref{eq:clustersync} is introduced to make the contribution of each cluster to the cost more balanced.
Such idea is in the same spirit as the ``RatioCut''~\cite{hagen1992new, von2007tutorial} studied in the graph partition problem.  

To proceed, we perform a change of optimization variables in \eqref{eq:clustersync}: for each cluster $\mathcal{C}_k$, let us define a block column vectors 
$\bm{V}^{(k)} = [\bm{V}_i^{(k)}]_{i = 1}^n \in \mathbb{R}^{nd \times d}$ 
of length $n$ whose $i$-th block $\bm{V}_i^{(k)} \in \mathbb{R}^{d \times d}$ satisfies
\begin{equation}
    \bm{V}_i^{(k)} := 
    \begin{cases}
    \frac{1}{\sqrt{|\mathcal{C}_k|}}\bm{\Theta}_i, &\; i \in \mathcal{C}_k,\\
    \bm{0}, &\; \text{otherwise}.
    \end{cases}
    \label{eq:hat_V_k_definition}
\end{equation}
As a result, $\bm{V}_i^{(k)}$ indicates the cluster membership of $\mathcal{C}_k$ and also includes the identified orthogonal group elements of nodes in $\mathcal{C}_k$. Then, \eqref{eq:clustersync} can be rewritten as
\begin{equation}
    \begin{aligned}
        \max_{\bm{V} \in \mathbb{R}^{nd \times Kd}} &\quad \left\langle \bm{A},\; \bm{V}\bm{V}^\top\right\rangle \\
        \mathrm{s.t.} &\quad \bm{V}\bm{V}^\top = \sum_{k = 1}^K \bm{V}^{(k)}(\bm{V}^{(k)})^\top 
        \text{ for $\{\bm{V}^{(k)}\}_{k = 1}^K$ satisfy the form in \eqref{eq:hat_V_k_definition}}.
    \end{aligned}
    \label{eq:clustersync_matrix_V}
\end{equation}
It is clear to see that \eqref{eq:clustersync_matrix_V} nonlinear and non-convex and is thus computationally intractable to be exactly solved. In~\cite{fan2021joint}, the authors use semidefinite relaxations to obtain an approximate solution with polynomial time complexity (on the program without the factor $1/|\mathcal{C}_k|$). However, solving large-scale SDPs is still highly nontrivial in general especially when $n$ is large. Therefore, here we propose a spectral method detailed in Section~\ref{sec:methods} to improve the efficiency while not sacrificing the accuracy.

	\section{Algorithm}
\label{sec:methods}

\begin{algorithm}[t!]
\SetAlgoLined
\SetNoFillComment
\LinesNotNumbered
\setcounter{AlgoLine}{0}
\vspace{0.1cm}
\KwIn{The observation matrix $\bm{A}$, the number of clusters $K$.}

\vspace{0.1cm}
\begin{minipage}{0.963\textwidth}
\begin{itemize}[leftmargin=12pt]
    \item [1.] (Spectral decomposition) Compute the top $Kd$ eigenvectors $\bm{\Phi} \in \mathbb{R}^{nd \times Kd}$ of $\bm{A}$ such that $\bm{\Phi}^\top \bm{\Phi} = \bm{I}_{Kd}$.
    \vspace{0.05cm}
    \item [2.] (Blockwise CPQR) Compute the \textit{blockwise} CPQR (detailed in Algorithm~\ref{alg:QR_blockwise}) of $\bm{\Phi}^\top$, which yields
    \begin{equation}
        \bm{\Phi}^\top \bm{\Pi}_{nd} = \bm{Q} \bm{R}\quad \Rightarrow \quad \bm{\Phi}^\top = \bm{Q}\bm{R}\bm{\Pi}_{nd}^\top.
        \label{eq:alg_QR}
    \end{equation}
    Update $\bm{R} \leftarrow \bm{R}\bm{\Pi}_{nd}^\top$.
    \vspace{0.05cm}
    \item [3.] (Recovery of the cluster memberships and orthogonal group elements)
    For each node $i = 1,\ldots, n$, assign its cluster as
    \begin{equation}
        \hat{\kappa}(i) \leftarrow \argmax_{k = 1,\ldots, K} \|\bm{R}_{ki}\|_\textrm{F}.
        \label{eq:cluster}
        \vspace{-0.2cm}
    \end{equation}
    Then  identify the orthogonal group element by the polar factor (Definition~\ref{def:polar}) as
    \begin{equation}
        \hat{\bm{O}}_{i} \leftarrow \mathcal{P}(\bm{R}_{ki})^\top \quad \text{where} \quad k = \hat{\kappa}(i).
        \label{eq:syncronization}
    \end{equation}
    \item [4.] {\color{black}
    (Optional) Perform the refinement step described in Section~\ref{sec:refinement}. 
    }
\end{itemize}
\vspace{-0.05cm}
\KwOut{\textup{Cluster assignments $\{\hat{\kappa}(i)\}_{i = 1}^n$ and orthogonal group elements $\{\hat{\bm{O}}_i\}_{i = 1}^n$.}}
\end{minipage}
\caption{Joint spectral clustering and synchronization}
\label{alg:spectral}
\end{algorithm}

Given the observation matrix $\bm{A}$, we propose the following algorithm for simultaneously finding the underlying clusters and synchronizing within each cluster. Our algorithm,  as summarized in Algorithm~\ref{alg:spectral}, is strikingly simple and only consists of three steps. First, we get the matrix $\bm{\Phi}$ which contains the top $Kd$ eigenvectors of $\bm{A}$ via  a spectral decomposition. Secondly, we get the matrix $\bm{R}$ through a blockwise CPQR of $\bm{\Phi}^\top$. We end up with cluster assignment and orthogonal group element recovery based on individual sub-block of $\bm{R}$, \textcolor{black}{followed by an optional step for refining the recovery result}. In particular, we refer to Section~\ref{sec:block_CPQR} for the details of the blockwise CPQR given in Definition~\ref{def:block_CPQR}.

We highlight that Algorithm~\ref{alg:spectral} is deterministic such that it has no dependency on any sort of random initialization.
In comparison, the performance of other classical algorithms such as $k$-means~\cite{arthur2006k} largely depends on the initial guess of the cluster centers. Also, in terms of the computational complexity, our algorithms scales linearly with the number of data points $n$ apart from the spectral decomposition step~(step 1), which is desirable especially when $n$ is large (see Section~\ref{sec:complexity} for a detailed discussion). In addition, from an implementation perspective, our algorithm is very easy to implement since it only consists of common and efficient matrix operations.

We point out that CPQR has been widely used in different scientific fields~\cite{zha2001spectral,damle2016robust, damle2017computing, damle2017scdm}, and our algorithm is mainly inspired by~\cite{damle2016robust}, which proposes using CPQR for clustering after a spectral embedding of the graph. The algorithm in~\cite{damle2016robust} achieves competitive and more robust performance against the classical Lloyd's algorithm~\cite{lloyd1982least} with $k$-means$++$ initialization~\cite{arthur2006k} under SBM model. For our problem, we introduce the blockwise CPQR which naturally extends such QR-based algorithm to handle the extra group transformations and block structures. 

\subsection{Algorithm motivation}
In this section, we provide motivations for  Algorithm~\ref{alg:spectral}. We start from the original problem \eqref{eq:clustersync_matrix_V}. By noticing that  $\{\bm{V}^{(k)}\}_{k = 1}^K$ in \eqref{eq:hat_V_k_definition} forms an orthogonal basis, the spectral method relaxes \eqref{eq:clustersync_matrix_V} by replacing the constraint in \eqref{eq:clustersync_matrix_V} with $\bm{V}^\top \bm{V} = \bm{I}_{Kd}$ and yields the following relaxed program
\begin{equation}
    \begin{aligned}
        \bm{\Phi} = \argmax_{\bm{{V}} \in \mathbb{R}^{nd \times Kd}} &\quad \langle \bm{A},\; \bm{{V}}\bm{{V}}^\top\rangle \quad
        \mathrm{s.t.} \quad \bm{{V}}^\top \bm{{V}} = \bm{I}_{Kd}
    \end{aligned}
    \label{eq:clustersync_matrix_spectral}
\end{equation}
whose global optimizer turns out to be the top $Kd$ eigenvectors of $\bm{A}$ denoted by $\bm{\Phi} \in \mathbb{R}^{nd \times Kd}$. This leads to the \emph{Step 1}~(spectral decomposition) in Algorithm~\ref{alg:spectral}. 

To proceed, we split $\bm{A}$ into deterministic and random parts, 
\begin{equation}
    \bm{A} = \mathbb{E}[\bm{A}] + (\bm{A} - \mathbb{E}[\bm{A}]) = \mathbb{E}[\bm{A}] + \bm{\Delta},
    \label{eq:A_decompose}
\end{equation}
where $\mathbb{E}[\bm{A}] = p\bm{A}^{\text{clean}}$ with the clean observation matrix $\bm{A}^{\text{clean}}$ defined in \eqref{eq:clean_A}, and the residual $\bm{\Delta}$ is a random perturbation with $\mathbb{E}[\bm{\Delta}] = \bm{0}$. 
Specifically, $\mathbb{E}[\bm{A}]$ is a low-rank matrix which satisfies the following decomposition:
\begin{equation*}
    \mathbb{E}[\bm{A}] = p\sum_{k = 1}^K m_k \bm{\Psi}^{(k)}(\bm{\Psi}^{(k)})^\top, \quad \text{with} \quad \bm{\Psi}_{i}^{(k)} := 
    \begin{cases}
    \frac{1}{\sqrt{m_k}}\bm{O}_i, &\; i \in C_k,\\
    \bm{0}, &\; \text{otherwise}, 
    \end{cases}
    \label{eq:V_k_definition}
\end{equation*}
where $\bm{\Psi}^{(k)} = [\bm{\Psi}_i^{(k)}]_{i = 1}^n \in \mathbb{R}^{nd \times d}$ is a block column vector of length $n$ that is defined in a similar manner to $\bm{V}^{(k)}$ in \eqref{eq:hat_V_k_definition}. Then  each $\bm{\Psi}^{(k)}$ indicates the true cluster memberships of $C_k$ and the exact orthogonal group element $\bm{O}_i$ of node $i$ within $C_k$.  
Therefore, the matrix $\bm{\Psi} = [\bm{\Psi}^{(1)}, \bm{\Psi}^{(2)}, \ldots, \bm{\Psi}^{(K)}]$ satisfies $\bm{\Psi}^\top \bm{\Psi} = \bm{I}_{Kd}$.

We first consider the clean case when $p = 1$ and $q = 0$, then we have $\bm{A} = \bm{A}^\text{clean}$, $\bm{\Delta} = \bm{0}$, and $\bm{\Phi} = \bm{\Psi}\bm{O}$, where $\bm{O} \in \mathbb{R}^{Kd \times Kd}$ is some orthogonal matrix. 
Then for recovering the communities and group elements, it suffices to extract $\{\bm{\Psi}^{(k)}\}_{k = 1}^K$ from $\bm{\Phi}$. However, such extraction is non-trivial since $\bm{O}$ is unknown to us. To resolve this, without loss of generality, we assume that the first $m_1$ nodes form the first community $C_1$, the following $m_2$ nodes form $C_2$, and so on, 
{\color{black}
then notice that $\bm{\Phi}^\top$ can be decomposed as 
\begin{equation}
    \begin{aligned}
    \bm{\Phi}^\top &= \bm{O}^\top 
    \begin{bmatrix}
    \bm{\Psi}^{(1)} &\cdots &\bm{\Psi}^{(k)}
    \end{bmatrix}^\top\\
    & =\bm{O}^\top
    \begin{bmatrix}
    \frac{1}{\sqrt{m_1}}\bm{O}_1^\top &\cdots &\frac{1}{\sqrt{m_1}}\bm{O}_m^\top &\cdots &\bm{0} &\cdots &\bm{0}\\
    \vdots &\vdots &\vdots &\ddots &\vdots &\vdots &\vdots\\
    \bm{0} &\cdots &\bm{0} &\cdots &\frac{1}{\sqrt{m_K}}\bm{O}_{n-m_K+1} &\cdots &\frac{1}{\sqrt{m_K}}\bm{O}_{n}
    \end{bmatrix}\\
    &= 
    \underbrace{
    \bm{O}^\top
    \begin{bmatrix}
    \bm{O}_1^\top &\cdots &\bm{0}\\
    \vdots &\ddots &\vdots \\
    \bm{0} &\cdots &\bm{O}_{n-m_K+1}^\top
    \end{bmatrix}}_{=:\bm{Q}}\\
    &\times
    \underbrace{
    \begin{bmatrix}
    \frac{1}{\sqrt{m_1}}\bm{I}_d &\cdots &\frac{1}{\sqrt{m_1}}\bm{O}_1\bm{O}_m^\top &\cdots &\bm{0} &\cdots &\bm{0}\\
    \vdots &\vdots &\vdots &\ddots &\vdots &\vdots &\vdots\\
    \bm{0} &\cdots &\bm{0} &\cdots &\frac{1}{\sqrt{m_K}}\bm{I}_d &\cdots &\frac{1}{\sqrt{m_K}}\bm{O}_{n-m_K+1}\bm{O}_{n}^\top
    \end{bmatrix}}_{ =:\bm{R}}\\
    &= \bm{Q}\bm{R}.
    \end{aligned}
    \label{eq:QR_general}
\end{equation}
The resulting decomposition $\bm{\Phi}^\top = \bm{Q}\bm{R}$ corresponds to the \emph{Step 2} (blockwise CPQR) in Algorithm~\ref{alg:spectral} up to some column permutation, and here we assume $\bm{\Pi}_{nd} = \bm{I}_{nd}$ such that no column pivoting is performed. In this way, $\bm{Q} \in \mathbb{R}^{Kd \times Kd}$ is an orthogonal matrix that includes the unknown orthogonal matrix $\bm{O}$, and $\bm{R} \in \mathbb{R}^{Kd \times nd}$ is a $K \times n$ block matrix that excludes $\bm{O}$. More importantly, $\bm{R}$ incorporates all the information needed for recovery, as shown in the following.
}

To recover the cluster memberships from $\bm{R}$ in \eqref{eq:QR_general}, one can see that for each node $i$, the $i$-th block column of $\bm{R}$ (i.e. $\bm{R}_{\cdot i}$) is sparse such that its $k$-th block $\bm{R}_{ki}$ is nonzero only if $k = \kappa(i)$, which indicates the cluster membership of node $i$. Also, the orthogonal group element $\bm{O}_i$ can be determined from the nonzero block (up to some global orthogonal transformation). This leads to the \emph{Step 3} in Algorithm~\ref{alg:spectral}, where the polar factor $\mathcal{P}(\cdot)$ in \eqref{eq:syncronization} ensures the orthogonality of the estimation $\hat{\bm{O}}_i$. 

In practice, when Algorithm~\ref{alg:spectral} is applied to the noisy observation $\bm{A}$, the exact recovery of the cluster assignments is still possible as long as the perturbation to $\mathbb{E} [\bm{A}]$, which is controlled by $p$ and $q$, is less than a certain threshold such that $\bm{\Phi}$ is still close to $\bm{\Psi} \bm{O}$. Indeed, this will be verified by both theoretical analysis in Section~\ref{sec:analysis} and numerical experiments in Section~\ref{sec:exp}.

\subsection{Blockwise column-pivoted QR}
\label{sec:block_CPQR} An important step of CPQR in Definition~\ref{def:CPQR} is selecting the pivots, which amounts to finding a subset of columns that are as linearly independent as possible and are used for determining the basis. In our setting as we present in \eqref{eq:QR_general}, we further require the QR factorization to always preserve the block structure which our recovery relies on, in other words, each $d \times d$ block should be treated as a single unit during the whole process. To handle this requirement, we propose a simple variant of CPQR, called \textit{blockwise CPQR} given in Definition~\ref{def:block_CPQR}, which preserves the block structure and can be computed by selecting a block column instead of a single column as the pivot. We present our blockwise CPQR in Algorithm~\ref{alg:QR_blockwise}, where the \textit{Householder transformation}~\cite{householder1958unitary} (see Algorithm~\ref{alg:Householder}) is adopted to ensure its numerical stability. Basically, Algorithm~\ref{alg:QR_blockwise} modifies the original CPQR algorithm~\cite{trefethen1997numerical, golub1996matrix} in the following three aspects: 

\vspace{0.05cm}
\begin{enumerate}[leftmargin=20pt]
\item At each iteration, we select a \textit{block column} instead of a column as pivot.
\item We determine the pivot by finding the block column with the largest Frobenius norm of its residual (see line 3 in Algorithm~\ref{alg:QR_blockwise}).
\item After each pivot selection, we compute $d$ orthonormal basis (instead of only one basis) from the pivot (see lines 7-11 in Algorithm~\ref{alg:QR_blockwise}) by using Householder transformation (Algorithm~\ref{alg:Householder}).
\end{enumerate}
\vspace{0.05cm}
\noindent As a result, the relative order of columns in each block is always preserved.
\begin{algorithm}[t!]

\DontPrintSemicolon
\SetAlgoLined
\SetNoFillComment
\KwIn{A block matrix $\bm{X} \in \mathbb{R}^{Kd \times nd}$ with $K \leq n$ and the block size $d$}
\kwInit{$\bm{\Pi}_{nd} \leftarrow \bm{I}_{nd}$,  $\bm{Q} \leftarrow \bm{I}_{Kd}$, $\bm{R} \leftarrow \bm{X}$}
\For{$t = 1, 2, \ldots, K$}{ 
\tcc{Pivot selection}
\For{$j = t, t+1, \ldots, n$}{

Compute the residual $\rho_j \leftarrow \|\bm{R}_{t\cdot, \; j}\|_\textrm{F}$, where $\bm{R}_{t\cdot, \; j} \in \mathbb{R}^{(K-t+1)d \times d}$ is the segment of the $j$-th block column from the $t$-th block to the end
}
Determine the pivot $ j^* \leftarrow \argmax_{j = t,\ldots,n} \rho_j$

For both $\bm{R}$ and $\bm{\Pi}_{nd}$, swap the $t$-th block column with the pivot ($j^*$-th) block column


\vspace{0.3cm}
\tcc{Determine orthonormal basis from the pivot block column}
\For{$j = 1, 2, \ldots, d$}{
Let $l \leftarrow (t-1)d + j$, apply Householder transformation in Algorithm~\ref{alg:Householder} on $\bm{r}_{l} = \bm{R}_{l:Kd , l} \in \mathbb{R}^{Kd - l+1}$, 
and  get the Householder matrix $\widetilde{\bm{Q}}_l$.  

\vspace{0.1cm}
 Update $\bm{Q}_l \leftarrow \begin{bmatrix}
 \bm{I}_{l-1} &\bm{0}\\
 \bm{0} &\widetilde{\bm{Q}}_l
 \end{bmatrix}$

\vspace{0.cm}
Update $\bm{R} \leftarrow \bm{Q}_l\bm{R}$ and $\bm{Q} \leftarrow \bm{Q}\bm{Q}_l$

}}
\KwOut{$\bm{Q}$, $\bm{R}$, \textup{and} $\bm{\Pi}_{nd}$.} 
\caption{Blockwise column-pivoted QR}
\label{alg:QR_blockwise}
\end{algorithm}

\begin{algorithm}[t!]
\DontPrintSemicolon
\SetAlgoLined
\SetNoFillComment
\KwIn{A column vector $\bm{x} \in \mathbb{R}^{n}$}
$\alpha \leftarrow -\text{sign}(x_1)\|\bm{x}\|$

$\bm{u} \leftarrow \bm{x} - \alpha\bm{e}_1$, where $\bm{e}_1 = [1, 0, \ldots, 0]^\top$

$\bm{v} \leftarrow \bm{u}/\|\bm{u}\|$

$\bm{Q} \leftarrow \bm{I}_n - 2\bm{v}\bm{v}^\top$

\KwOut{\textup{Householder matrix} $\bm{Q}$.}
\caption{Householder transformation}
\label{alg:Householder}
\end{algorithm}

\subsection{Refinements} 
\label{sec:refinement}
\paragraph{On cluster memberships} We propose the following step that further improves the clustering result by Algorithm~\ref{alg:spectral}. In \eqref{eq:cluster}, $\max_k\|\bm{R}_{ki}\|_\mathrm{F}$ can be interpreted a ``confidence score'' that node $i$ belongs to its assigned cluster. Then   our idea is to refine those cluster assignments with low confidence scores. Formally, given a threshold $\varepsilon \in (0, 1)$, we define
\begin{equation}
    S_{\varepsilon}= \left\{\, i \biggm| \max_k \frac{ \|\bm{R}_{ki}\|_\mathrm{F}}{\|\bm{R}_{\cdot i}\|_\mathrm{F}} \leq \varepsilon \,\right\}
    \label{eq:refine_S_k}
\end{equation}
as the set of ill-defined nodes, where $\bm{R}_{\cdot i}$ denotes the $i$-th block column.
{\color{black}
Notably, $\varepsilon$ is determined based on the distribution of $\{\max_k \|\bm{R}_{ki}\|_\mathrm{F} / \|\bm{R}_{\cdot i}\|_\mathrm{F}\}_{i = 1}^n$ such that $S_{\varepsilon}$ includes a small fraction of nodes and $|S_{\varepsilon}| \ll n$. As a result, this strategy enables us to  control the ``level of refinement'' instead of setting $\varepsilon$ directly. We find that in practice, refining on a small portion of nodes (e.g., $10\%$) usually yields satisfying result with a mild computational complexity.
} Then, for each node $i \in S_{\varepsilon}$, we refine its cluster to be
\begin{equation}
    \hat{\kappa}(i) = \argmax_{k = 1,\ldots, K} \frac{1}{|\hat{C}_k|}\sum_{j \in \hat{C}_k} \sqrt{|\hat{C}_k|}\left\|(\bm{R}_{\cdot i})^\top\bm{R}_{\cdot j} \right\|_\mathrm{F} = \argmax_{k = 1,\ldots, K} \frac{1}{\sqrt{|\hat{C}_k|}}\sum_{j \in \hat{C}_k} \left\|(\bm{R}_{\cdot i})^\top\bm{R}_{\cdot j} \right\|_\mathrm{F}
    \label{eq:cluster_refine}
\end{equation}
where $\hat{C}_k$ is the $k$-th cluster identified by Algorithm~\ref{alg:spectral}. To understand \eqref{eq:cluster_refine}, consider the clean case as shown in \eqref{eq:QR_general}: for each cluster, we have $\hat{C}_k = C_k$, $|\hat{C}_k| = m_k$, and \eqref{eq:cluster_refine} measures the averaged similarity between node $i$ and all nodes $j \in \hat{C}_k$ as 
\begin{equation*}
    \frac{1}{\sqrt{|\hat{C}_k|}}\sum_{j \in \hat{C}_k} \left\|(\bm{R}_{\cdot i})^\top\bm{R}_{\cdot j} \right\|_\mathrm{F} = 
    \begin{cases}
    \sqrt{d}, &\quad i \in \hat{C}_k \\
    0, &\quad i \notin \hat{C}_k
    \end{cases}
    \label{eq:cluster_refine_clean}
\end{equation*}
Then~\eqref{eq:cluster_refine} selects the cluster with the maximum similarity. Notably, the factor $\sqrt{|\hat{C}_k|}$ in \eqref{eq:cluster_refine} removes the dependency of \eqref{eq:cluster_refine} on cluster sizes. 

{\color{black}
\paragraph{On orthogonal transforms} We also have a step for refining orthogonal transforms as follows: for each cluster $\hat{C}_k$ identified by \eqref{eq:cluster}, we collect all the available pairwise transforms $\bm{O}_{ij}$ for nodes in $\hat{C}_k$, and form an observation matrix $\bm{A}^{(k)} \in \mathbb{R}^{\hat{m}_kd \times \hat{m}_kd}$ where $\hat{m}_k = |\hat{C}_k|$. In other words, $\bm{A}^{(k)}$ is a restricting of $\bm{A}$ on $\hat{C}_k$.
Then we compute the top-$d$ eigenvectors of $\bm{A}^{(k)}$ denoted by $\bm{\Phi}^{(k)} \in \mathbb{R}^{m_k d \times d}$. For each node $i$, we have
\begin{equation}
    \hat{\bm{O}}_i^{\text{refine}} = \mathcal{P}(\bm{\Phi}^{(k)}_i)
    \label{eq:refine_orth}
\end{equation}
as the refined orthogonal transform. Notably, under the probabilistic model in Section~\ref{sec:pre} where the pairwise transforms are noiseless within clusters, one can perfectly identify all the orthogonal transform $\{\bm{O}_i\}_{i = 1}^n$ by \eqref{eq:refine_orth}, as long as the cluster memberships are exactly recovered.\footnote{Here, we also assume the sub-graph that corresponds to each cluster is connected (a spanning tree exists). Otherwise, one can add an arbitrary global offset to a connected component without violating the observations.}
}

\subsection{Complexity}
\label{sec:complexity}
{\color{black}
We summarize the complexity of Algorithm~\ref{alg:spectral} in Table~\ref{tab:complexity}. Overall, the cost of Algorithm~\ref{alg:spectral} is largely dominated by the spectral decomposition step whose complexity depends on the sparsity of the network and is linear with the number of edges observed in the graph. The remaining steps of Algorithm~\ref{alg:spectral} all together is relatively efficient and scale linearly with $n$ and quadratically with $K$. As a result, when the data network $G$ is densely connected with $m = O(n^2)$ edges observed, Algorithm~\ref{alg:spectral} ends up with $O(n^2)$ complexity. While in practice, it is more common (see e.g.,~\cite{leskovec2008statistical}) to obtain a sparse network $G$ such as $m = O(n\log n)$ or $m = O(n)$, then Algorithm~\ref{alg:spectral} runs efficiently as the complexity reduces to $O(n\log n)$ or $O(n)$, respectively.

We also remark that in practice, the complexity of spectral decomposition can be further reduced from linearly scaling with $K$ i.e. $O(K)$ to $O(\log K)$, by using randomized algorithm described in, e.g.~\cite{halko2011finding, woolfe2008fast}. But notice that the resulting decomposition is an approximation and could lead to extra error of recovery (of the cluster memberships and the orthogonal transforms).
}

\aboverulesep=0ex
\belowrulesep=0ex
\renewcommand{\arraystretch}{1.25}
\begin{table}[t!]
\caption{The complexity of Algorithm~\ref{alg:spectral} in each step. We consider two cases where the network (graph) is densely or sparsely connected. For the case of a sparse network, $m$ denotes the number of edges observed in the graph.}
\centering
\begin{tabular}{p{0.02\textwidth}<{}| p{0.3\textwidth}<{} | p{0.25\textwidth}<{\centering} | p{0.25\textwidth}<{\centering}} 
\toprule
\multicolumn{2}{c|}{Steps} & A dense network &A sparse network \\ \midrule
1 &Spectral decomposition\tablefootnote{Note that we only need to compute the top $Kd$ eigenvectors, which is assumed to be done by Lanczos method~\cite{stewart2002krylov}.} &$O(Kd^3n^2)$ & $O(Kd^3m)$ \\ \midrule
2 &Blockwise CPQR &\multicolumn{2}{c}{$O(K^2d^2n + K^2d^3n)$}  \\ \midrule
\multirow{2}{*}{3} &Clustering by \eqref{eq:cluster} &\multicolumn{2}{c}{$O(Kd^2n)$}  \\ \cmidrule{2-4}
 &Synchronization by \eqref{eq:syncronization} &\multicolumn{2}{c}{$O(d^3n)$} \\ \midrule
\multicolumn{2}{c|}{Total complexity} & $O(Kd^3n^2 + K^2d^3n)$ & $O(Kd^3m + K^2d^3n)$ \\
\bottomrule
\end{tabular}
\label{tab:complexity}
\end{table}

	\section{Analysis}
\label{sec:analysis}

In this section, we investigate the performance of Algorithm~\ref{alg:spectral} under the probabilistic model described in Section~\ref{sec:pre}, by finding the condition that guarantees the clusters $\{C_k\}_{k = 1}^K$ are exactly recovered and the orthogonal transforms $\{\bm{O}_i\}_{i = 1}^n$ are estimated with bounded error. For simplicity, we focus on the case of two underlying clusters with equal cluster sizes, namely $K = 2$ and $m_1 = m_2 = m = n/2$.

To begin with, recall \eqref{eq:A_decompose} that $\bm{A} = \mathbb{E}[\bm{A}] + \bm{\Delta}$, Davis-Kahan theorem~\cite{davis1970rotation} (see Theorem~\ref{the:davis_kahan}) bounds the difference between the noisy eigenvectors $\bm{\Phi}$ of $\bm{A}$ and the clean ones $\bm{\Psi}$ of $\mathbb{E}[\bm{A}]$ in terms of the spectral norm as 
\begin{equation}
    \|\bm{\Phi} - \bm{\Psi}\bm{O}\| \lesssim \frac{\|\bm{\Delta}\bm{\Psi}\|}{\delta} \leq \frac{\|\bm{\Delta}\|\|\bm{\Psi}\|}{\delta},
    \label{eq:Davis_kahan}
\end{equation}
with a global orthogonal transformation $\bm{O} = \mathcal{P}(\bm{\Psi}^\top\bm{\Phi})$ and a certain spectral gap $\delta$~(see Theorem~\ref{the:davis_kahan} in Appendix~\ref{sec:tech_ingreidient} for the detail).
Then one may expect that exact recovery can be achieved as long as the error in \eqref{eq:Davis_kahan} is shown to be sufficiently small. However, even a tight bound of $\|\bm{\Phi} - \bm{\Psi}\bm{O}\|$ cannot guarantee exact recovery, since if a few blocks of $\bm{\Phi}$ have much larger error than others, then exact recovery of every cluster membership could still be failed. 
Therefore, a more appealing way to show exact recovery is to obtain a blockwise error bound between $\bm{\Phi}$ and $\bm{\Psi}$ as 
\begin{equation}
    \max_{1 \leq i\leq n} \|\bm{\Phi}_{i\cdot} - \bm{\Psi}_{i\cdot}\bm{O}\| \leq \epsilon
    \label{eq:block_wise}
\end{equation}
for some uniform $\epsilon$, where $\bm{\Phi}_{i\cdot}, \bm{\Psi}_{i\cdot} \in \mathbb{R}^{d \times Kd}$ denote the $i$-th block row of $\bm{\Phi}$ and $\bm{\Psi}$ respectively. As a result, the error on each node (block) is uniformly bounded and one can show that exact recovery is guaranteed as long as $\epsilon$ is sufficiently small. 

\subsection{Main theorems} Our first theoretical result provides a condition that guarantees~\eqref{eq:block_wise} is satisfied for a sufficiently small $\epsilon$, which further leads to the performance guarantee that Algorithm~\ref{alg:spectral} achieves exact recovery under such condition.

\vspace{0.05cm}
\begin{theorem}[Blockwise error bound]
Under the setting of two equal-sized clusters with model parameters $(n, p, q, d)$, for a sufficiently large $n$, suppose
    \begin{equation}
    \eta := \frac{\sqrt{(p(1-p)+q)\log (nd)}}{p\sqrt{n}} \leq c_0
    \label{eq:p_q_assumption}
\end{equation}
for some small constant $c_0 \geq 0$. Then with probability $1 - O(n^{-1})$,
\begin{equation}
    \max_{1 \leq i \leq n} \|\bm{\Phi}_{i\cdot} - \bm{\Psi}_{i\cdot}\bm{O}\| \lesssim \frac{\eta}{\sqrt{n}}
    \label{eq:bound_max_block}
\end{equation}
where $\bm{O} = \mathcal{P}(\bm{\Psi}^\top\bm{\Phi})$. 
\label{the:phi_psi_i}
\end{theorem}

\begin{theorem}[Performance guarantee]
Under the assumption of Theorem~\ref{the:phi_psi_i}, for $i = 1,\ldots, n$, with probability $1 - O(n^{-1})$, Algorithm~\ref{alg:spectral} exactly recovers the cluster memberships $\kappa(i)$ defined in Section~\ref{sec:pre}, and $\hat{\bm{O}}_i$ satisfies
\begin{equation}
    \|\hat{\bm{O}}_i - \bm{O}_i\bar{\bm{O}}_{\kappa(i)}\| \lesssim \eta
    \label{eq:sync_bounded}
\end{equation}
where $\bar{\bm{O}}_{\kappa(i)}$ is orthogonal and only depends on the cluster that $i$ belongs to.
\label{the:cond}
\end{theorem}

As we can see, \eqref{eq:p_q_assumption} is the condition that 
leads to exact recovery of the cluster memberships, and \eqref{eq:sync_bounded} guarantees the estimation error of $\hat{\bm{O}}_i$ for each node $i$ is uniformly bounded. Also, by letting $q = 0$ and $d = O(1)$, \eqref{eq:p_q_assumption} reduces to
\begin{equation}
    \frac{p}{1-p} \gtrsim \frac{\log n}{n}.
    \label{eq:cond_rewritten}
\end{equation}
As a result, \eqref{eq:cond_rewritten} implies that \eqref{eq:p_q_assumption} holds and exact recovery is possible if $p \gtrsim \log n/n$. Notably, such threshold lies in the same regime by using the SDPs studied in~\cite{fan2021joint}, which indicates that the spectral method yields competitive results against SDP. 

In the following, we outline the key steps for proving Theorems~\ref{the:phi_psi_i} and \ref{the:cond}. The complete proof is deferred to Appendix~\ref{sec:proof_main_theorems}.
Our proof follows from two important ingredients, one is the \textit{leave-one-out} technique presented in \cite{abbe2020entrywise, zhong2018near}
which enables us to have a tight, entrywise analysis on the eigenvectors of low-rank matrices; the second one is \cite{ling2020near} which extends the entrywise analysis to a blockwise error bound for studying group synchronization. Here, our main contribution lies in handling the difficulty introduced by the combination of community structures and orthogonal group elements, as well as the QR-factorization for clustering. 

\subsection{The sketch of proof}
\label{sec:proof_sketch}
Throughout the analysis we assume the first $m = n/2$ nodes belong to  $C_1$ and the remaining $m$ nodes belong to $C_2$. Without loss of generality, we assume $\bm{O}_i = \bm{I}_d, \; i = 1, \ldots, n$. For brevity, we use ``w.h.p" (with high probability) to represent ``with probability $1 - O(n^{-1})$". 

\vspace{0.1cm}
\paragraph{$\bullet$ Step 1: Bound $\|\bm{\Phi}_{i\cdot} - \bm{\Phi}_{j\cdot}\|$} We first bound the distance $\|\bm{\Phi}_{i\cdot} - \bm{\Phi}_{j\cdot}\|$ for any pair of nodes within the same cluster. The keypoint lies in finding a suitable surrogate of $\bm{\Phi}$ such that the difference between $\bm{\Phi}$ and its surrogate can be tightly bounded. To this end, let $\bm{\Lambda} \in \mathbb{R}^{2d \times 2d}$ be a diagonal matrix that contains the top $2d$ eigenvalues of $\bm{A}$ as diagonal entries, then we have $\bm{\Phi} = \bm{A}\bm{\Phi}\bm{\Lambda}^{-1}$. To proceed, inspired by \cite{ling2020near, abbe2020entrywise}, let us define the following function of $\bm{\Phi}$:
\begin{equation}
    f(\bm{\Phi}) := \bm{A}\bm{\Phi}\bm{\Lambda}^{-1},
    \label{eq:eigen_decomp}
\end{equation}
then $\bm{\Phi}$ is a fixed point of $f(\cdot)$ such that $f(\bm{\Phi}) = \bm{\Phi}$. Since \eqref{eq:Davis_kahan} indicates that $\bm{\Phi}$ is close to $\bm{\Psi}\bm{O}$ with $\bm{O} = \mathcal{P}(\bm{\Psi}^\top\bm{\Phi})$, then we hope the following choice of surrogate
\begin{equation}
    f(\bm{\Psi}\bm{O}) = \bm{A}\bm{\Psi}\bm{O}\bm{\Lambda}^{-1}
    \label{eq:def_f}
\end{equation}
serves as a good estimation of $\bm{\Phi}$ \textit{blockwisely} such that the block error  $\|\bm{\Phi}_{i\cdot} - f(\bm{\Psi}\bm{O})_{i\cdot}\|$ are uniformly bounded for any $i$.
If this holds, we can further bound $\|\bm{\Phi}_{i\cdot} - \bm{\Phi}_{j\cdot}\|$ as 
\begin{equation}
\begin{aligned}
        \|\bm{\Phi}_{i\cdot} - \bm{\Phi}_{j\cdot}\| &= \|\bm{\Phi}_{i\cdot} - f(\bm{\Psi}\bm{O})_{i\cdot} - (\bm{\Phi}_{j\cdot} - f(\bm{\Psi}\bm{O})_{j\cdot}) + f(\bm{\Psi}\bm{O})_{i\cdot} - f(\bm{\Psi}\bm{O})_{j\cdot}\| \\
        &\leq \|\bm{\Phi}_{i\cdot} - f(\bm{\Psi}\bm{O})_{i\cdot}\| + \|\bm{\Phi}_{j\cdot} - f(\bm{\Psi}\bm{O})_{j\cdot}\| + \|f(\bm{\Psi}\bm{O})_{i\cdot} - f(\bm{\Psi}\bm{O})_{j\cdot}\|.
\end{aligned}
\label{eq:bound_phi_i_phi_j_main}
\end{equation}

To bound $\|\bm{\Phi}_{i\cdot} - f(\bm{\Psi}\bm{O})_{i\cdot}\|$, by definition it satisfies 
\begin{align}
    \|\bm{\Phi}_{i\cdot} - f(\bm{\Psi}\bm{O})_{i\cdot}\| &= \|[\bm{\Phi} - \bm{A}\bm{\Psi}\bm{O}\bm{\Lambda}^{-1}]_{i \cdot}\| = \|[\bm{A}(\bm{\Phi} - \bm{\Psi}\bm{O})\bm{\Lambda}^{-1}]_{i \cdot}\| \nonumber \\ &\leq \|\bm{\Lambda}^{-1}\|\|[\bm{A}(\bm{\Phi} - \bm{\Psi}\bm{O})]_{i \cdot}\|
    = \|\bm{\Lambda}^{-1}\|\|[(\mathbb{E}[\bm{A}] + \bm{\Delta})(\bm{\Phi}- \bm{\Psi}\bm{O})]_{i \cdot}\| \nonumber \\
    &\leq \|\bm{\Lambda}^{-1}\|\left(\|[\mathbb{E}[\bm{A}]]_{i \cdot}(\bm{\Phi} - \bm{\Psi}\bm{O})\|+ \|\bm{\Delta}_{i \cdot}(\bm{\Phi} - \bm{\Psi}\bm{O})\| \right) \label{eq:epsilon_1_decomp_main}
\end{align}
Here, $\|\bm{\Lambda}^{-1}\|$ and $\|[\mathbb{E}[\bm{A}]]_{i\cdot}(\bm{\Phi} - \bm{\Psi}\bm{O})\|$ are tightly bounded by Weyl's inequality (see Theorem~\ref{the:weyl}) and Davis-Kahan theorem (see Theorem~\ref{the:davis_kahan}), respectively. For $\|\bm{\Delta}_{i \cdot}(\bm{\Phi} - \bm{\Psi}\bm{O})\|$, it is natural to consider applying concentration inequalities for getting a tight bound since $\bm{\Delta}_i$ consists of independent noisy blocks. However, this is impossible since $\bm{\Delta}_{i \cdot}$ and $\bm{\Phi} - \bm{\Psi}\bm{O}$ are not statistically independent, which only allows Cauchy-Schwarz inequality such that $\|\bm{\Delta}_{i \cdot}(\bm{\Phi} - \bm{\Psi}\bm{O})\| \leq \|\bm{\Delta}_{i \cdot}\|\|\bm{\Phi} - \bm{\Psi}\bm{O}\|$ that is bounded loosely. To resolve this, we resort to the \textit{leave-one-out} technique introduced in~\cite{zhong2018near, abbe2020entrywise, ling2020near}, the idea is to define an auxiliary matrix $\bm{A}^{(i)}$ that ``leaves out'' the $i$-th block row and column as
\begin{equation}
    \bm{A}^{(i)} := \mathbb{E}[\bm{A}] + \bm{\Delta}^{(i)}, \quad \bm{\Delta}^{(i)}_{kl} := 
    \begin{cases}
    \bm{\Delta}_{kl}, & \text{ if  $k \neq i$ and  $l \neq i$}, \\
    \bm{0}, & \text{ if $k = i$ or $l = i$}.
    \end{cases}
    \label{eq:A_i_definition}
\end{equation}
In other words, $\bm{A}^{(i)}$ only differs from $\bm{A}$ on its $i$-th block row and $i$-th block column. Because of this tiny difference, the noise perturbation $\bm{\Delta}_{i \cdot}$ is not included in $\bm{A}^{(i)}$ and thus $\bm{\Delta}_{i \cdot}$ is independent of $\bm{\Phi}^{(i)}$, which denotes the top $Kd$ eigenvectors of $\bm{A}^{(i)}$. This enables us to tightly bound $\|\bm{\Delta}_{i \cdot}(\bm{\Phi} - \bm{\Psi}\bm{O})\|$ by replacing $\bm{\Phi}$ with $\bm{\Phi}^{(i)}$ and applying concentration inequalities. To this end, by defining
\begin{equation*}
    \begin{aligned}
    \bm{O}^{(i)} := \mathcal{P}((\bm{\Phi}^{(i)})^\top\bm{\Phi}), \quad 
    \bm{S}^{(i)} := \mathcal{P}(\bm{\Psi}^{\top}\bm{\Phi}^{(i)}),
    \end{aligned}
\end{equation*}
then $\|\bm{\Delta}_{i \cdot}(\bm{\Phi} - \bm{\Psi}\bm{O})\|$ can be decomposed as
\begin{align}
    \|\bm{\Delta}_{i \cdot}(\bm{\Phi} - \bm{\Psi}\bm{O})\| &= \|\bm{\Delta}_{i \cdot}(\bm{\Phi} - \bm{\Phi}^{(i)}\bm{O}^{(i)} + \bm{\Phi}^{(i)}\bm{O}^{(i)} - \bm{\Psi}\bm{S}^{(i)}\bm{O}^{(i)} + \bm{\Psi}\bm{S}^{(i)}\bm{O}^{(i)} - \bm{\Psi}\bm{O})\| \nonumber \\
    &\leq \underbrace{\|\bm{\Delta}_{i \cdot}(\bm{\Phi} \!-\! \bm{\Phi}^{(i)}\bm{O}^{(i)})\|}_{=: T_1} \!+\! \underbrace{\|\bm{\Delta}_{i \cdot}(\bm{\Phi}^{(i)} \!-\! \bm{\Psi}\bm{S}^{(i)})\|}_{=: T_2} \!+\! \underbrace{\|\bm{\Delta}_{i \cdot}\bm{\Psi}(\bm{S}^{(i)}\bm{O}^{(i)} \!-\! \bm{O})\|}_{=: T_3}    \label{eq:bound_T_1_T_2_T_3}\\
    &= T_1 + T_2 + T_3. \nonumber 
\end{align}
{\color{black} 
Then, under the condition \eqref{eq:p_q_assumption}, we bound $T_1$, $T_2$ and $T_3$ separately with details given in Appendix~\ref{sec:proof_the_1}, where particularly $T_2$ is bounded by matrix Bernstein inequality~\cite[Theorem 1.6.2]{tropp2015introduction}. Here, we remark that the condition \eqref{eq:p_q_assumption} is necessary for bounding $T_1, T_2$, and $T_3$, as $\eta $ should be sufficiently small (i.e. $\eta \leq c_0$ for some small $c_0$) for obtaining the following intermediate result (see \eqref{eq:phi_j_bound} in the proof of Lemma~\ref{lemma:Phi_Phi_i} for details):
\begin{equation*}
    \max_j\|\bm{\Phi}_{j \cdot}^{(i)}\| \lesssim \max_j\|\bm{\Phi}_{j \cdot}\|
    \label{eq:phi_j_bound_main}
\end{equation*}
which implies that after leave-one-out, the norm of the block row of $\bm{\Phi}^{(i)}$ is universally (at most) at the same order as the original $\bm{\Phi}$. Then, \eqref{eq:phi_j_bound_main} gives the bounds on $\{T_i\}_{i = 1}^3$, and further leads to the following bound on $\|\bm{\Phi}_{i\cdot} - f(\bm{\Psi}\bm{O})_{i\cdot}\|$:}
\begin{lemma}
Under the condition \eqref{eq:p_q_assumption}, for all $i = 1, \ldots, n$,
\begin{equation*}
    \|\bm{\Phi}_{i\cdot} - f(\bm{\Psi}\bm{O})_{i\cdot}\| \lesssim \eta\max_j \|\bm{\Phi}_{j\cdot}\|
\end{equation*}
with probability  $1 - O(n^{-1})$.
\label{the:Phi_f_Phi}
\end{lemma}

Lemma~\ref{the:Phi_f_Phi} confirms our previous hypothesis that $\bm{\Phi}_{i\cdot}$ is well-approximated by its surrogate $f(\bm{\Psi}\bm{O})_{i\cdot}$ uniformly, as long as $\eta$ defined in \eqref{eq:p_q_assumption} is sufficiently small. Back to \eqref{eq:bound_phi_i_phi_j_main}, 
$\|f(\bm{\Psi}\bm{O})_{i\cdot} - f(\bm{\Psi}\bm{O})_{j\cdot}\|$ can be bounded as
\begin{equation*}
\begin{aligned}
    \|f(\bm{\Psi}\bm{O})_{i\cdot} \!-\! f(\bm{\Psi}\bm{O})_{j\cdot}\| \!=\! \|(\bm{\Delta}_{i \cdot} \!-\! \bm{\Delta}_{j \cdot})\bm{\Psi}\bm{O}\bm{\Lambda}^{-1}\| \!\leq\! \|(\bm{\Delta}_{i \cdot} \!-\! \bm{\Delta}_{j \cdot})\bm{\Psi}\|\|\bm{\Lambda}^{-1}\| \!\lesssim\! \eta \max_j \|\bm{\Phi}_{j\cdot}\|.
\end{aligned}
\end{equation*}
with the detail deferred to Appendix~\ref{sec:proof_the_1}. Combining the results above yields

\begin{lemma}
Under the condition \eqref{eq:p_q_assumption}, for any pair of nodes $(i,j)$ that belong to the same cluster, it satisfies
\begin{equation*}
    \begin{aligned}
        \|\bm{\Phi}_{i\cdot} - \bm{\Phi}_{j\cdot}\| \lesssim \eta\max_j \|\bm{\Phi}_{j\cdot}\| \lesssim \frac{\eta}{\sqrt{n}}
    \end{aligned}
\end{equation*}
with probability $1 - O(n^{-1})$.
\label{the:phi_i_phi_j}
\end{lemma}

\vspace{0.1cm}
\paragraph{$\bullet$  Step 2: Bound $\|\bm{\Phi}_{i\cdot} - \bm{\Psi}_{i\cdot}\bm{O}\|$}
Now we use the results from Step 1 to prove Theorem~\ref{the:phi_psi_i}, which gives the blockwise error bound $\|\bm{\Phi}_{i\cdot} - \bm{\Psi}_{i\cdot}\bm{O}\|$. The main idea is to combine Lemma~\ref{the:phi_i_phi_j} with the bound on $\|\bm{\Phi} - \bm{\Psi}\bm{O}\|$ given by Davis-Kahan. To this end, let us define $\bm{\Phi} - \bm{\Psi}\bm{O} = [\bm{N}_{C_1}^\top, \bm{N}_{C_2}^\top ]^\top$ 
where $\bm{N}_{C_1}$ is defined as
\begin{equation}
    \bm{N}_{C_1} :=
    \begin{bmatrix}
    \bm{\Phi}_{1\cdot} - \bm{\Psi}_{1\cdot}\bm{O}\\
    \vdots\\
    \bm{\Phi}_{m\cdot} - \bm{\Psi}_{m\cdot}\bm{O}\\
    \end{bmatrix}
    = 
    \underbrace{\begin{bmatrix}
    \bm{\Phi}_{i\cdot} - \bm{\Psi}_{1\cdot}\bm{O}\\
    \vdots\\
    \bm{\Phi}_{i\cdot}- \bm{\Psi}_{m\cdot}\bm{O}\\
    \end{bmatrix}}_{=:\bm{N}_i} + 
    \underbrace{
    \begin{bmatrix}
    \bm{\Phi}_{1\cdot} - \bm{\Phi}_{i\cdot}\\
    \vdots\\
    \bm{\Phi}_{m\cdot} - \bm{\Phi}_{i\cdot}\\
    \end{bmatrix}}_{=: \bm{N}_{\bm{\Delta}, i}}
    = \bm{N}_i + \bm{N}_{\bm{\Delta}, i}, 
    \label{eq:def_N_C_1}
\end{equation}
for all $i \in C_1$, and $\bm{N}_{C_2}$ is defined in a similar manner for block indices that belong to $C_2$.  
Then we can derive a lower bound of $\|\bm{\Phi} - \bm{\Psi}\bm{O}\|$ as 
\begin{align*}
    \|\bm{\Phi} - \bm{\Psi}\bm{O}\| & = \max_{\|\bm{x}\| = 1}\|[\bm{N}_{C_1}^\top, \bm{N}_{C_2}^\top] \bm{x}\|  \geq \max_{\|\bm{y}\| = 1} \|\bm{N}_{C_1}^\top\bm{y}\| = \|\bm{N}_{C_1}\| \\ 
    & \geq \|\bm{N}_i\| - \|\bm{N}_{\bm{\Delta}, i}\| = \sqrt{m}\|\bm{\Phi}_{i\cdot} - \bm{\Psi}_{i\cdot}\bm{O}\| - O(\eta)
\end{align*}
w.h.p. On the other hand, by Davis-Kahan we can obtain an upper bound as $\|\bm{\Phi} - \bm{\Psi}\bm{O}\| \lesssim \eta/\sqrt{\log (n)}$ w.h.p (see Lemma~\ref{lemma:Psi_Phi} for the detail).
Combining the lower bound and the upper bound yields
\begin{equation*}
    \|\bm{\Phi}_{i\cdot} - \bm{\Psi}_{i\cdot}\bm{O}\| \lesssim \frac{\eta}{\sqrt{n}}, \quad \text {for } i \in C_1
\end{equation*}
w.h.p. In a similar manner, we get the same bound for $i \in C_2$. Then applying the union bound for $i = 1, \ldots, n$ completes the proof of Theorem~\ref{the:phi_psi_i}.

\vspace{0.1cm}
\paragraph{$\bullet$ Step 3: Find the performance guarantee of Algorithm~\ref{alg:spectral}}
We start from showing the exact recovery of the cluster memberships by Algorithm~\ref{alg:spectral}. Recall the matrix $\bm{R}$ output from the blockwise CPQR in \eqref{eq:alg_QR}, with out loss of generality, we assume $p_1 \in C_1$ is the first pivot block column selected. Then for any node $j$, the two blocks $\bm{R}_{1j}$ and $\bm{R}_{2j}$ satisfy
\begin{equation*}
    \|\bm{R}_{1j}\|_\mathrm{F}^2 = \|\mathcal{P}(\bm{\Phi}_{p_1\cdot})(\bm{\Phi}_{j\cdot})^\top\|^2_\mathrm{F}, \quad
    \|\bm{R}_{2j}\|_\mathrm{F}^2 = \|(\bm{\Phi}_{j\cdot})^\top\|_\mathrm{F}^2 - \|\mathcal{P}(\bm{\Phi}_{p_1\cdot})(\bm{\Phi}_{j\cdot})^\top\|^2_\mathrm{F},
\end{equation*}
which correspond to the projection of $(\bm{\Phi}_{j\cdot})^\top$ onto the column space $\mathcal{R}((\bm{\Phi}_{p_1\cdot})^\top)$ and the complement of $\mathcal{R}((\bm{\Phi}_{p_1\cdot})^\top)$ respectively. According to Algorithm~\ref{alg:spectral}, node $j$ would be assigned to $C_1$ if $\|\bm{R}_{1j}\|_\mathrm{F} > \|\bm{R}_{2j}\|_\mathrm{F}$, which is equivalent to
\begin{equation}
    \|\mathcal{P}(\bm{\Phi}_{p_1\cdot})(\bm{\Phi}_{j\cdot})^\top\|_\mathrm{F} > \frac{\sqrt{2}}{2}\|\bm{\Phi}_{j\cdot}\|_\mathrm{F}
    \label{eq:assign_cluster_equal_main}.
\end{equation}
We first show that for any $j \in C_1$, \eqref{eq:assign_cluster_equal_main} is satisfied w.h.p. To this end, we estimate the LHS and the RHS of \eqref{eq:assign_cluster_equal_main} by using $\bm{\Psi}_{i\cdot}\bm{O}$ as a surrogate of $\bm{\Phi}_{i\cdot}$ for $i = 1,\ldots n$, as they have shown to be similar in Theorem~\ref{the:phi_psi_i}. Then we obtain that w.h.p
\begin{equation}
\begin{aligned}
\|\mathcal{P}(\bm{\Phi}_{p_1\cdot})(\bm{\Phi}_{j\cdot})^\top\|_{\mathrm{F}} 
&= \left(\sqrt{2} - O(\eta)\right)\sqrt{\frac{d}{n}}, \quad \frac{\sqrt{2}}{2}\|\bm{\Phi}_{j\cdot}\|_\mathrm{F} =
    \left(1 + O(\eta)\right)\sqrt{\frac{d}{n}}.
\end{aligned}
\end{equation}
This implies when the condition \eqref{eq:p_q_assumption} is satisfied such that $\eta$ is sufficiently small, \eqref{eq:assign_cluster_equal_main} holds and $j$ is correctly classified.
{\color{black}
Similar analysis applies to any $j \in C_2$ and we leave the details in Appendix~\ref{sec:proof_main_theorems}.
}

For recovering the orthogonal transforms $\{\bm{O}_i\}_{i = 1}^n$, we follow the assumption that $p_1 \in C_1$ is the first pivot. In the case of two clusters, the orthogonal matrix $\bm{Q}$ in the block CPQR \eqref{eq:alg_QR} satisfies $\bm{Q} = [
\bm{Q}_{\cdot 1}, \bm{Q}_{\cdot 2}]$, where $\bm{Q}_{\cdot 1} \in \mathbb{R}^{2d \times d}$ is the polar factor of $(\bm{\Phi}_{p_1 \cdot})^\top$ up to some orthogonal transform, and $\bm{Q}_{\cdot 2} \in \mathbb{R}^{2d \times d}$ is orthogonal to $\bm{Q}_{\cdot 1}$, i.e.,
\begin{equation}
     \bm{Q}_{\cdot 1} = \mathcal{P}((\bm{\Phi}_{p_1\cdot})^\top) \bar{\bm{O}}_1, \quad \text{and} \quad (\bm{Q}_{\cdot 1})^\top \bm{Q}_{\cdot 2} = \bm{0},
 \label{eq:def_Q_12} 
\end{equation}
where $\bar{\bm{O}}_1 \in \mathbb{R}^{d \times d}$ is some orthogonal matrix. 
As a result, for any node $j \in C_1$, our estimation $\hat{\bm{O}}_j$ by \eqref{eq:syncronization} can be written as 
\begin{equation*}
    \hat{\bm{O}}_j = \mathcal{P}(\bm{R}_{1j}) = \mathcal{P}((\bm{Q}_{\cdot 1})^\top (\bm{\Phi}_{j\cdot})^\top).
\end{equation*}
To proceed, by applying Theorem~\ref{the:phi_psi_i} we can show the estimation error $\|\hat{\bm{O}}_j - \bm{O}_j \bar{\bm{O}} _1\| \lesssim \eta$. Similar analysis applies to any $j \in C_2$ and we leave the details in Appendix~\ref{sec:proof_main_theorems}.
	
	\begin{figure}[t!]
    \centering
    \subfloat[\scriptsize{Success rate of exact recovery, $d = 2$}]{\includegraphics[width = 0.40\textwidth]{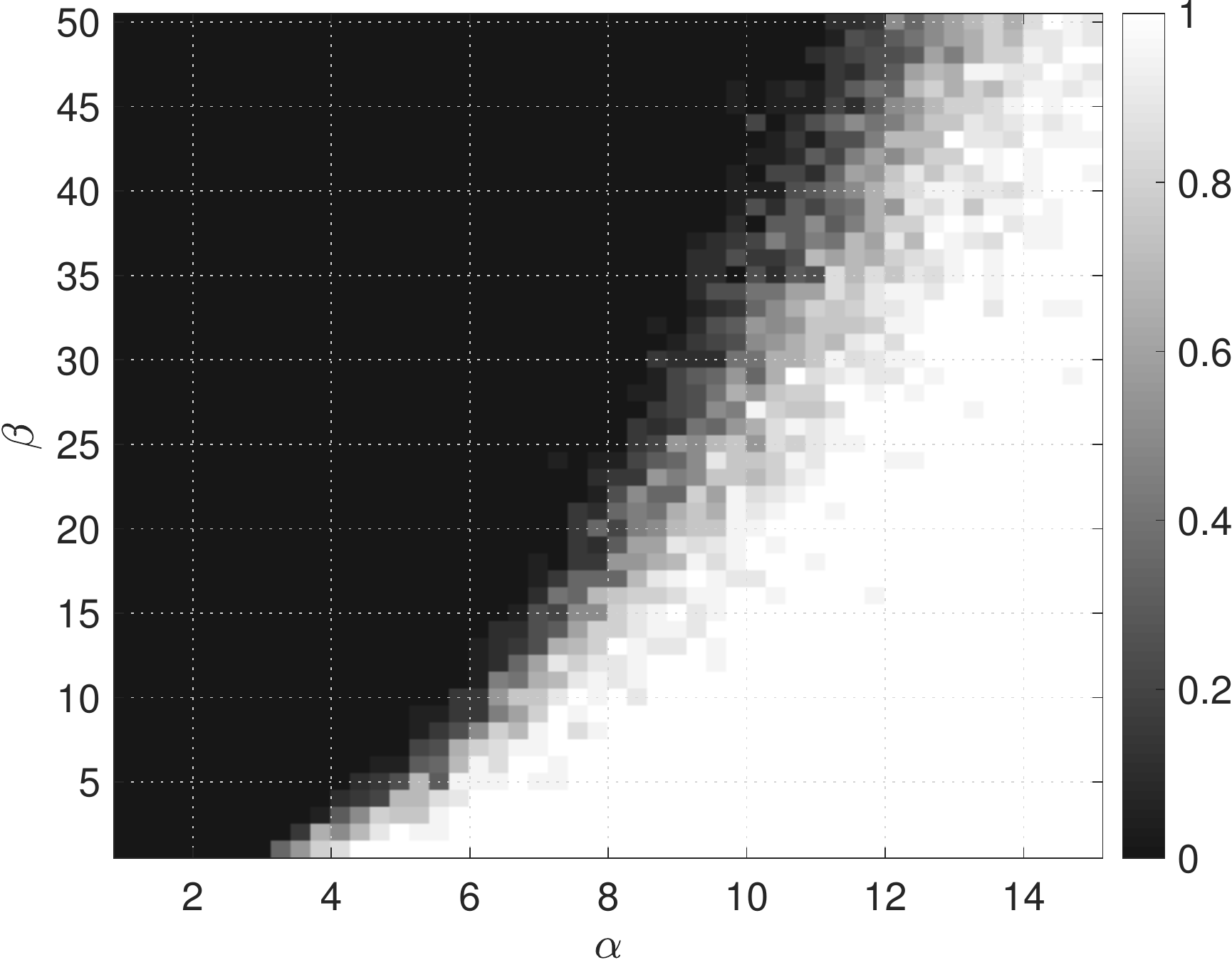}
    \label{fig:exp_K_2_a}
    }\hskip 0.2cm
    \subfloat[\scriptsize{Error of synchronization, $d = 2$}]{\includegraphics[width = 0.405\textwidth]{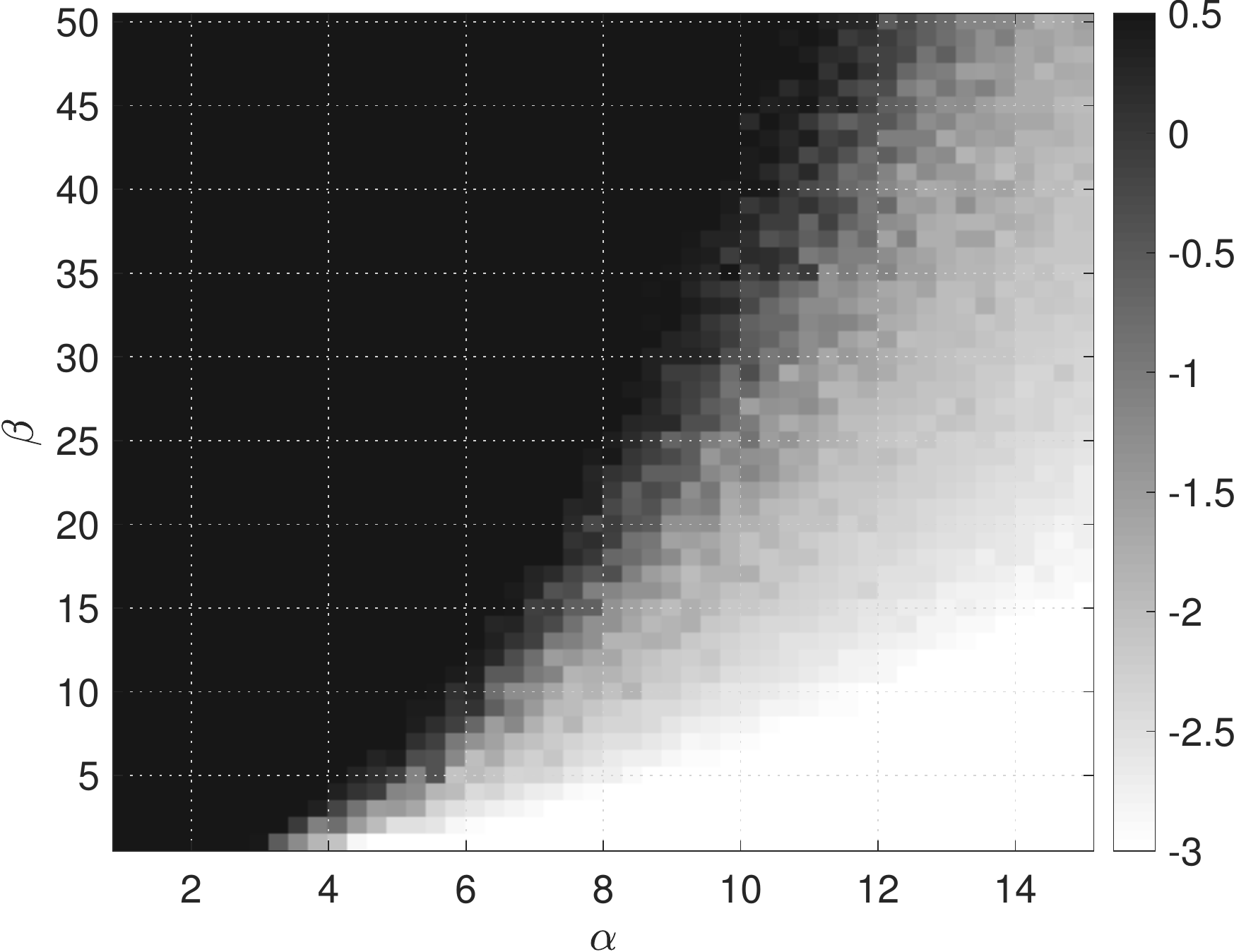}
    \label{fig:exp_K_2_b}
    }\\[-2pt]
    \hskip -0.2cm
    \subfloat[\scriptsize{Success rate of exact recovery, $d = 3$}]{\includegraphics[width = 0.40\textwidth]{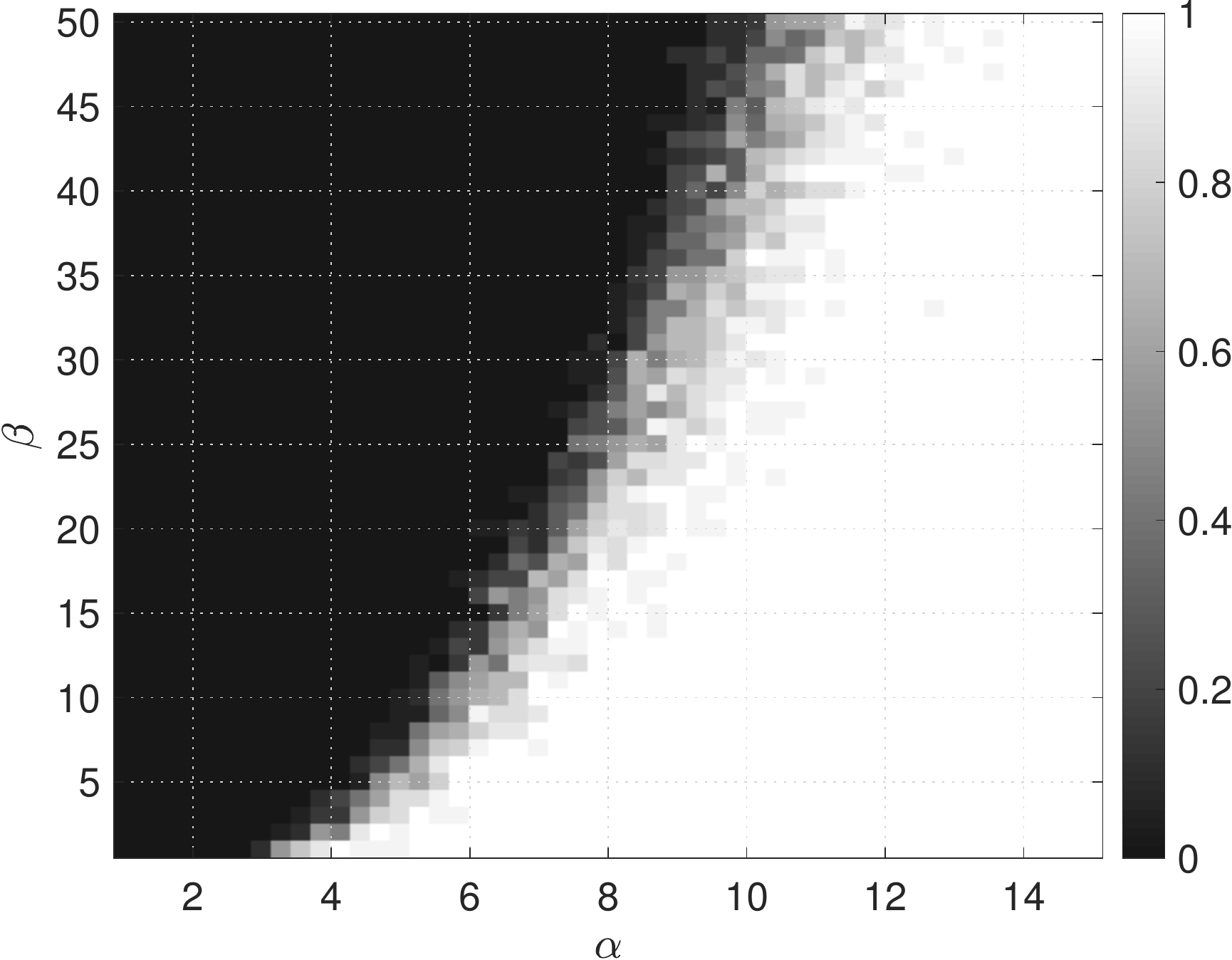}\label{fig:exp_two_cluster_d_2_cluster}}\hskip 0.45cm
    \subfloat[\scriptsize{Error of synchronization, $d = 3$}]{\includegraphics[width = 0.405\textwidth]{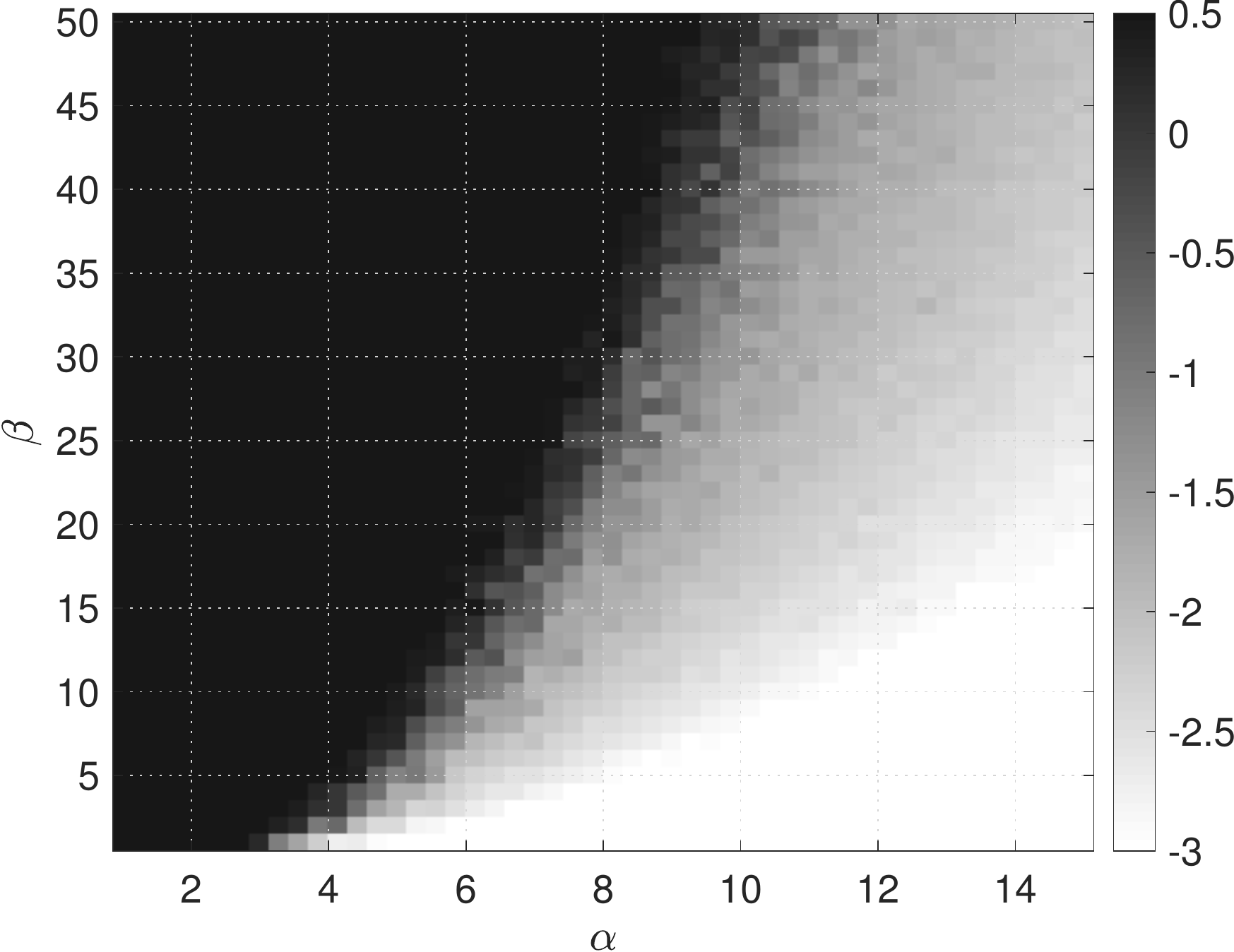}\label{fig:exp_two_cluster_d_2_sync}}
    \caption{\textit{Results on two clusters by Algorithm~\ref{alg:spectral}}. We test under the setting $m_1 = m_2 = 500$, $d = 2$ or $d = 3$.  
    \protect \subref{fig:exp_K_2_a} and \protect \subref{fig:exp_two_cluster_d_2_cluster}: the success rate of exact recovery by \eqref{eq:def_failure_rate}, under varying $\alpha$ in $p = \alpha\log n/n$ and $\beta$ in $q = \beta\log n/n$; 
    \protect \subref{fig:exp_K_2_b} and \protect \subref{fig:exp_two_cluster_d_2_sync}: the synchronization error by \eqref{eq:def_error_sync}.
    }
    \label{fig:exp_K_2_spec}
\end{figure}

\section{Numerical experiments}
\label{sec:exp}

\begin{figure}[t!]
    \centering
    \subfloat[\scriptsize{Success rate with varying $\eta$ and $\alpha$, $d = 2$}]{\includegraphics[width = 0.4\textwidth]{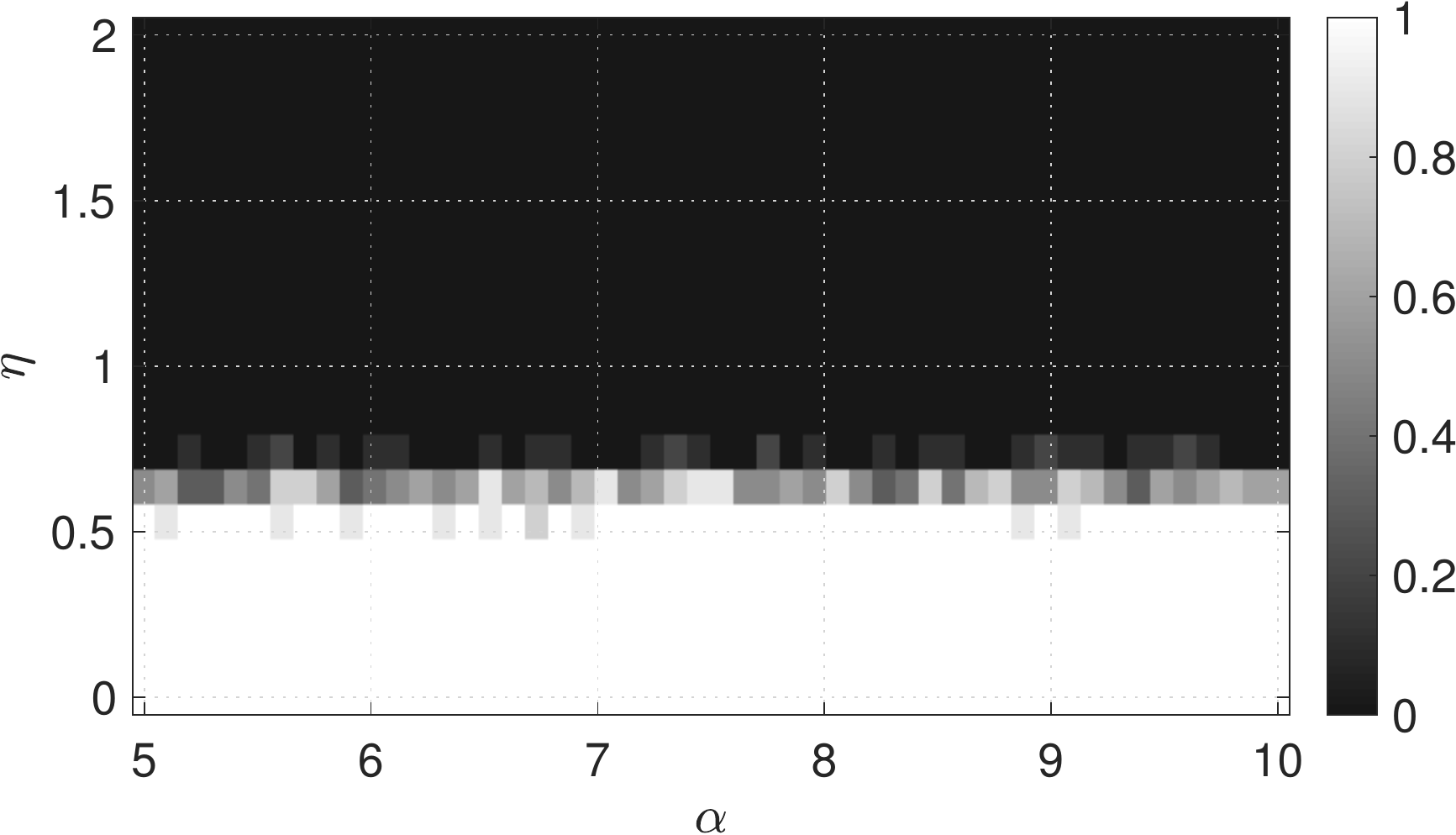}
    \label{fig:exp_scale_a}
    }\hskip 0.2cm 
    \subfloat[\scriptsize{Success rate with varying $\eta$ and $\alpha$, $d = 2$}]{\includegraphics[width = 0.4\textwidth]{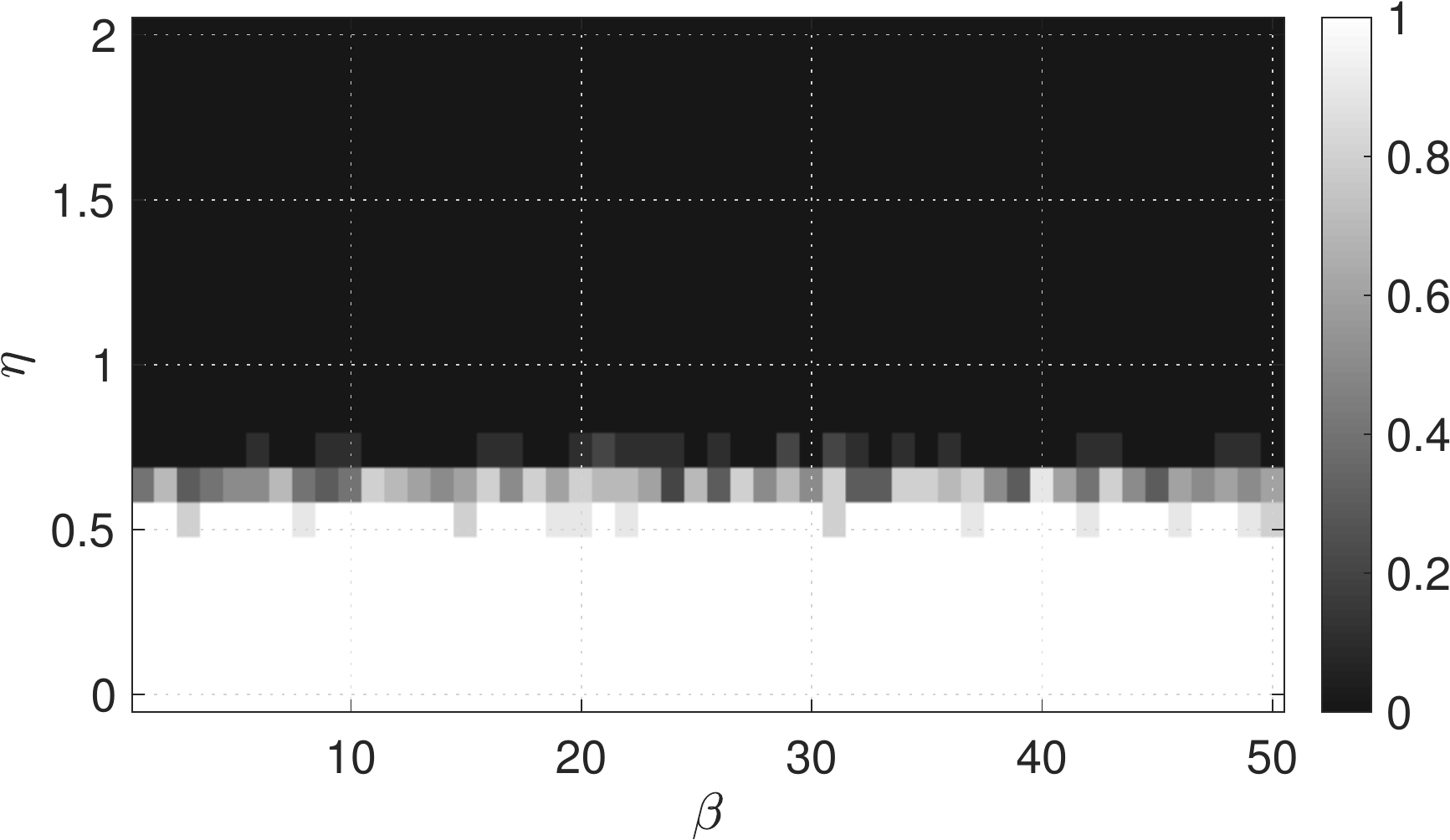}
    \label{fig:exp_scale_b}}\\[-3pt]
    \subfloat[\scriptsize{Success rate with varying $\eta$ and $\beta$, $d = 3$}]{\includegraphics[width = 0.4\textwidth]{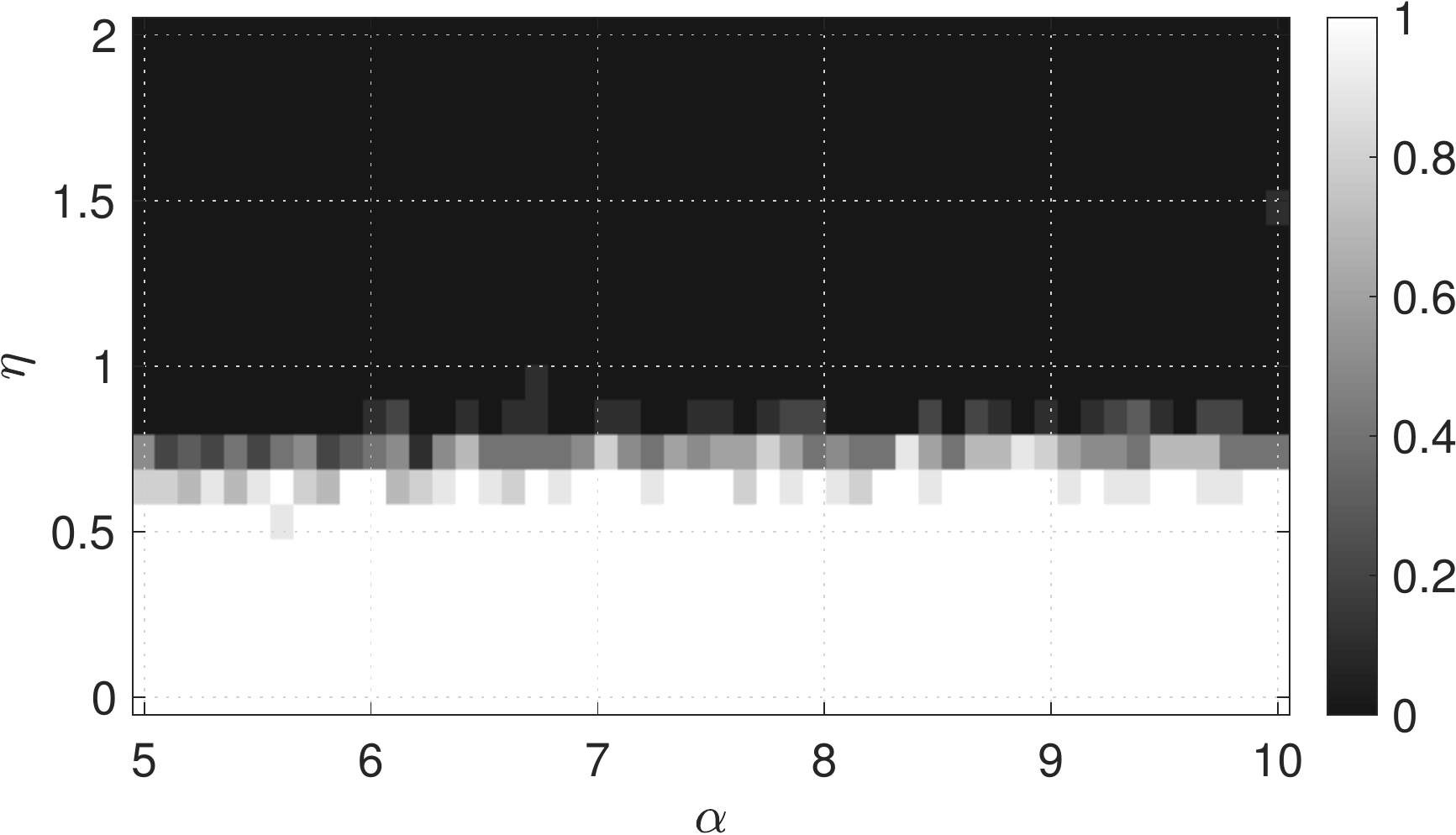}
    \label{fig:exp_scale_c}} \hskip 0.3cm
    \subfloat[\scriptsize{Success rate with varying $\eta$ and $\beta$, $d = 3$}]{\includegraphics[width = 0.4\textwidth]{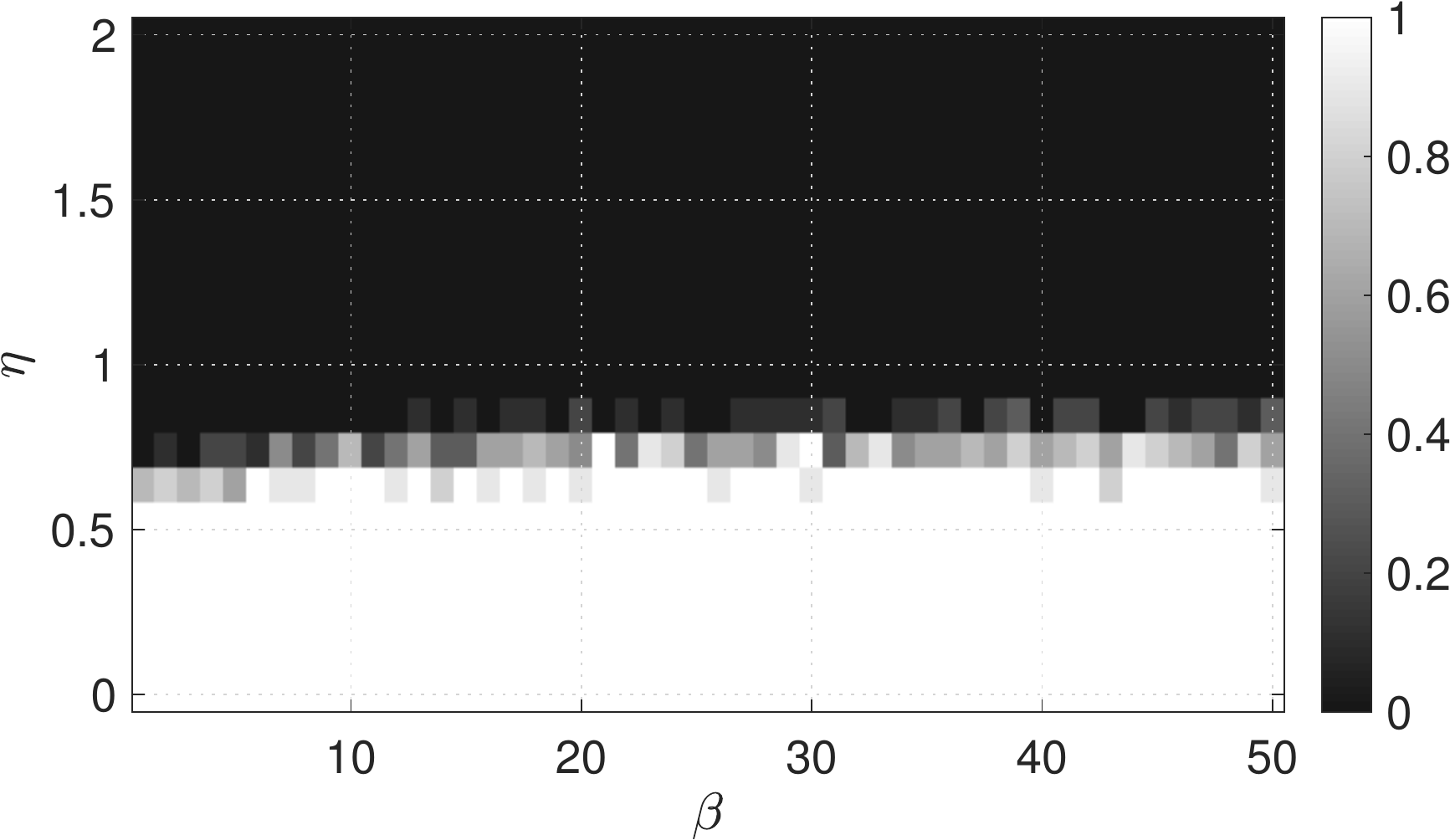}
    \label{fig:exp_scale_d}}
    \caption{\textit{Results on two clusters by Algorithm~\ref{alg:spectral}}. We test under the setting $m_1 = m_2 = 200$, $d = 2$ or $d = 3$. \protect \subref{fig:exp_scale_a} and \protect \subref{fig:exp_scale_c}: the success rate of exact recovery under varying $\eta$ defined in \eqref{eq:p_q_assumption} and $\alpha$ in $p = \alpha \log n/n$; \protect \subref{fig:exp_scale_b} and \protect \subref{fig:exp_scale_d}: the success rate of exact recovery under varying $\eta$ and $\beta$ in $q = \beta \log n/n$. For a fixed $\alpha$ (resp. $\beta$), we adjust $\eta$ by changing $\beta$ (resp. $\alpha$).}
    \label{fig:exp_scale}
\end{figure}

This section is devoted to numerically investigating the performance of our algorithm. 
{\color{black}
All experiments\footnote{The code is available in \url{https://github.com/frankfyf/joint_cluster_sync_spectral}.} are performed in MATLAB on a machine with 60 Intel Xeon CPU cores, running at 2.3GHz with 512GB RAM in total, and only one core is used for each experiment.
}
In each experiment, we generate the observation matrix $\bm{A}$ based on the probabilistic model in Section~\ref{sec:pre} and estimate the cluster memberships and the group elements by Algorithm~\ref{alg:spectral}. To evaluate the result, for clustering, let $\hat{C}_k = \{i | \hat{\kappa}(i) = k \}$ be the set of nodes identified to the $k$-th cluster by Algorithm~\ref{alg:spectral}, then we compute
\begin{equation}
    \textit{Success rate of exact recovery} = \text{ the rate } \{\hat{C}_k\}_{k = 1}^K \text{ is identical to } \{C_k\}_{k = 1}^K,
    \label{eq:def_failure_rate}
\end{equation}
that is the rate that Algorithm~\ref{alg:spectral} exactly recovers all the clusters memberships. 
After that, in order to evaluate the quality of identified orthogonal transformations, we define $\bm{O}^{(k)} = [\bm{O}_i]_{i \in C_k} \in \mathbb{R}^{m_kd \times d}$
for each cluster $C_k$ as the matrix that concatenates the ground truth $\bm{O}_i$ for all $i \in C_k$, and similarly define $\hat{\bm{O}}^{(k)} = [\hat{\bm{O}}_i]_{i \in C_k}$, the estimated orthogonal transformations. Then we remove the orthogonal ambiguity by aligning $\hat{\bm{O}}^{(k)}$ with $\bm{O}^{(k)}$ within each cluster $C_k$ as the following:
$$\bm{G}^{(k)} = \argmin_{\bm{G} \in \mathrm{O}(d)} \|\hat{\bm{O}}^{(k)} - \bm{O}^{(k)}\bm{G}\|_\mathrm{F}, \quad k = 1,\ldots, K$$
whose analytical solution is 
$\bm{G}^{(k)} = \mathcal{P}((\bm{O}^{(k)})^\top \hat{\bm{O}}^{(k)})$. In this way, the error of synchronization is defined as
\begin{equation}
    \textit{Error of synchronization} = \log\left(\frac{1}{\sqrt{d}}\max_{k = 1,\ldots, K} \max_{i \in C_k} \|\hat{\bm{O}}_i - \bm{O}_i\bm{G}^{(k)}\|_\mathrm{F}\right)
    \label{eq:def_error_sync}
\end{equation}
which is the maximum error of our estimation $\hat{\bm{O}}_i$ over all nodes. As a result, \eqref{eq:def_error_sync} is small only if the estimation error of each $\bm{O}_i$ is bounded. Both \eqref{eq:def_failure_rate} and \eqref{eq:def_error_sync} are averaged over 20 different realizations for each experiment.

\begin{figure}[t!]
    \centering
    \subfloat[\scriptsize{Success rate, before refinement}]{\includegraphics[width = 0.4\textwidth]{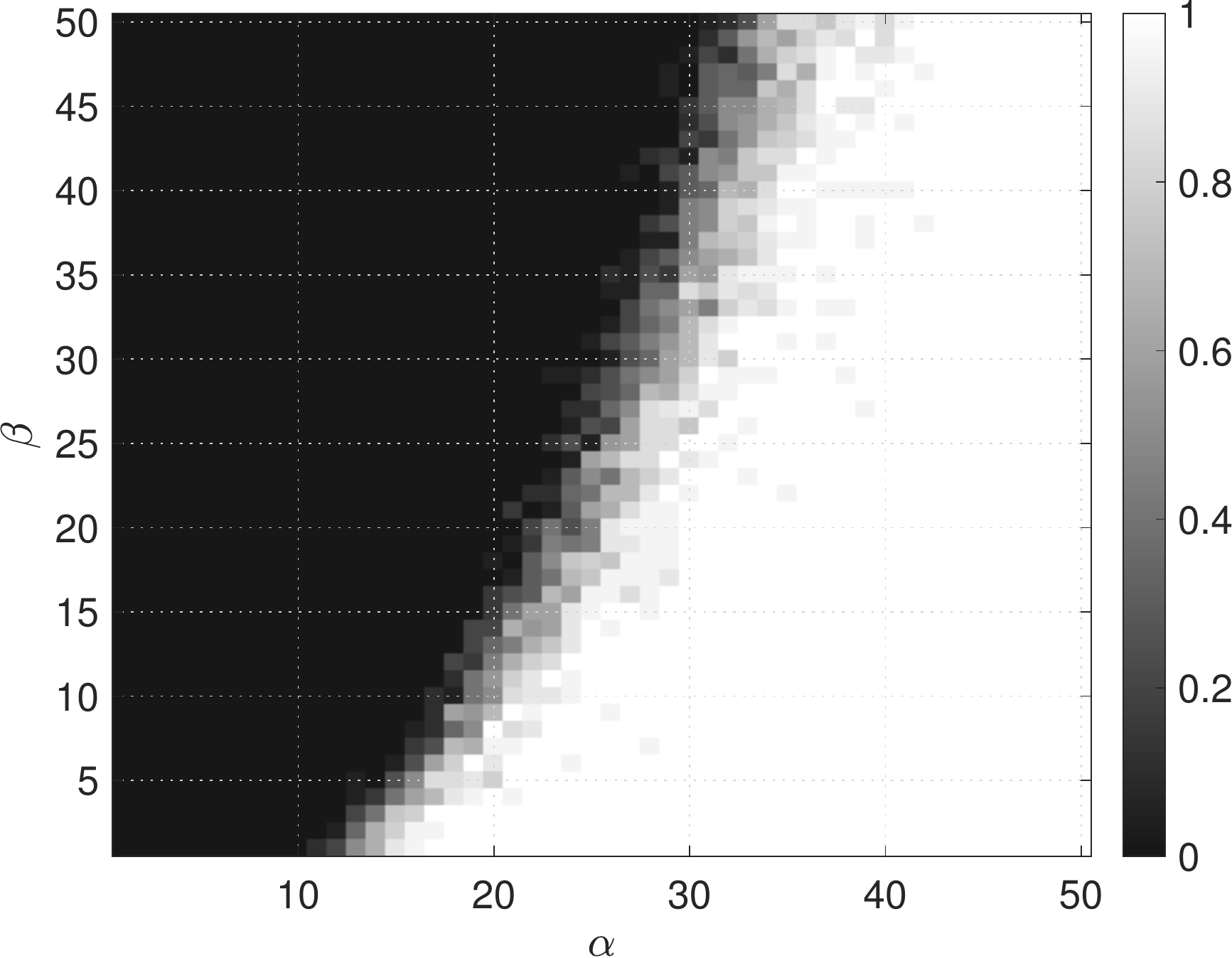}
    \label{fig:exp_three_a}
    } \hskip 0.2cm
    \subfloat[\scriptsize{Error of sync., before refinement}]{\includegraphics[width = 0.405\textwidth]{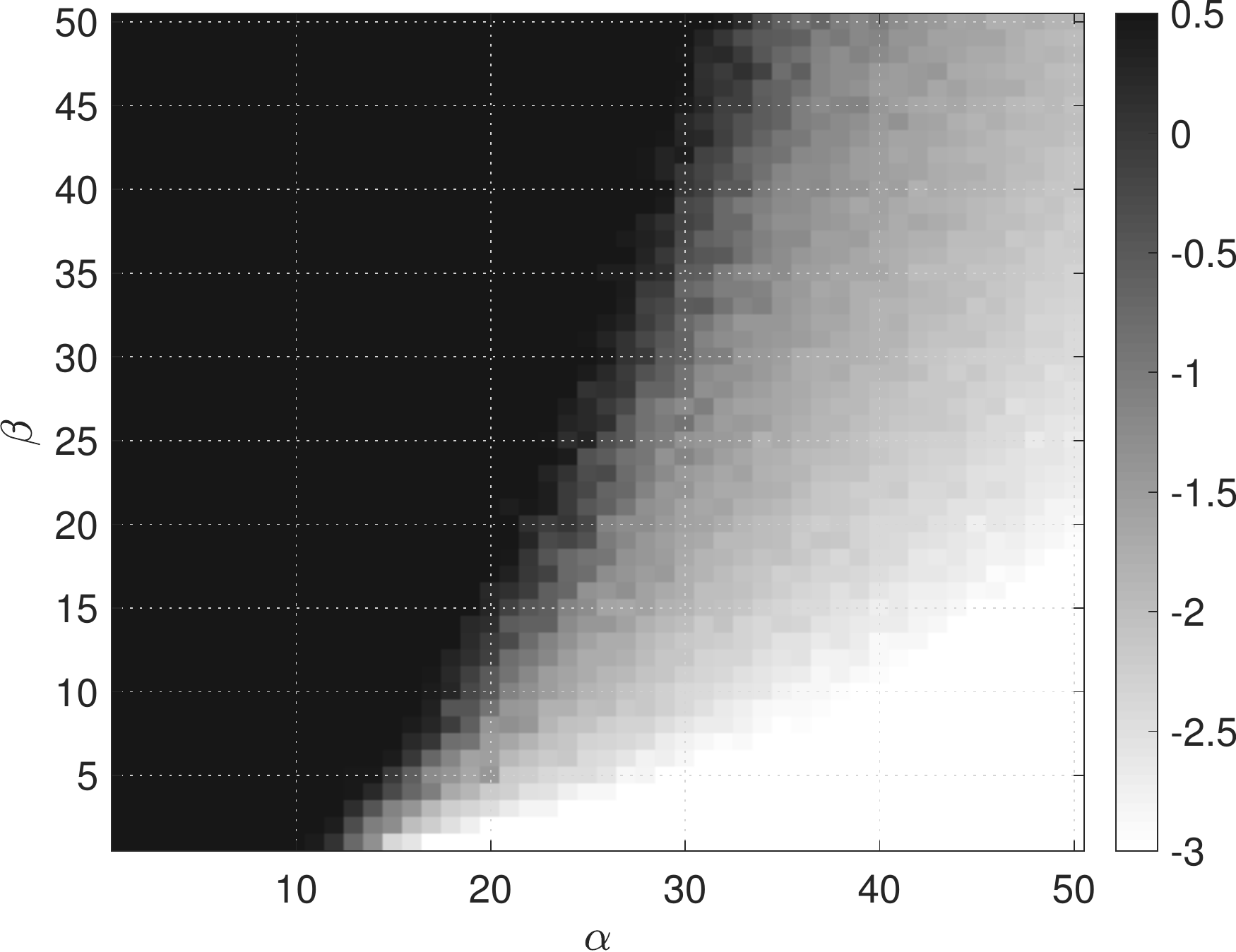}
    \label{fig:exp_three_b}}\\[-3pt]
    \subfloat[\scriptsize{Success rate, after refinement}]{\includegraphics[width = 0.4\textwidth]{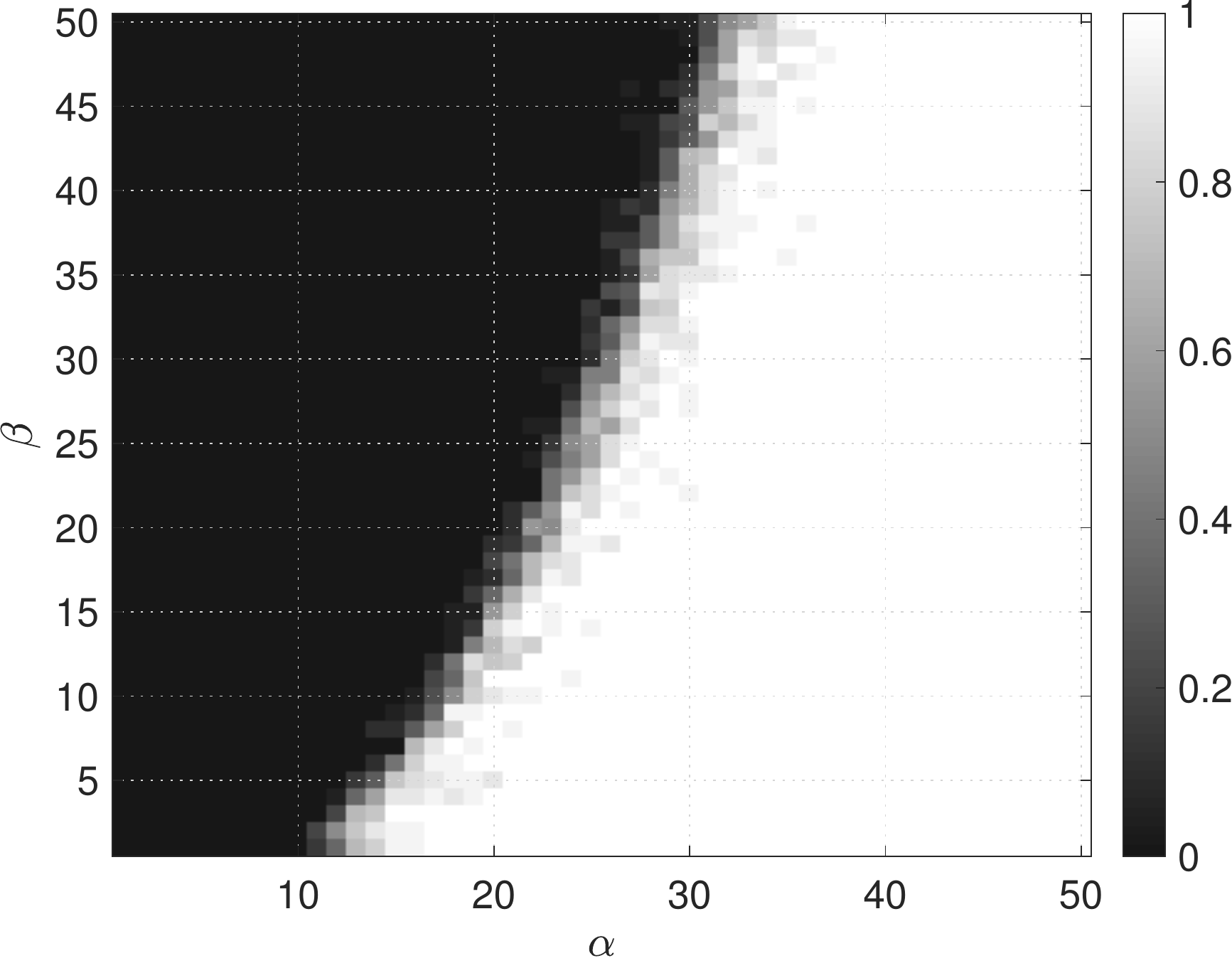}
    \label{fig:exp_three_c}} \hskip 0.33cm
    \subfloat[\scriptsize{Error of sync., after refinement}]{\includegraphics[width = 0.405\textwidth]{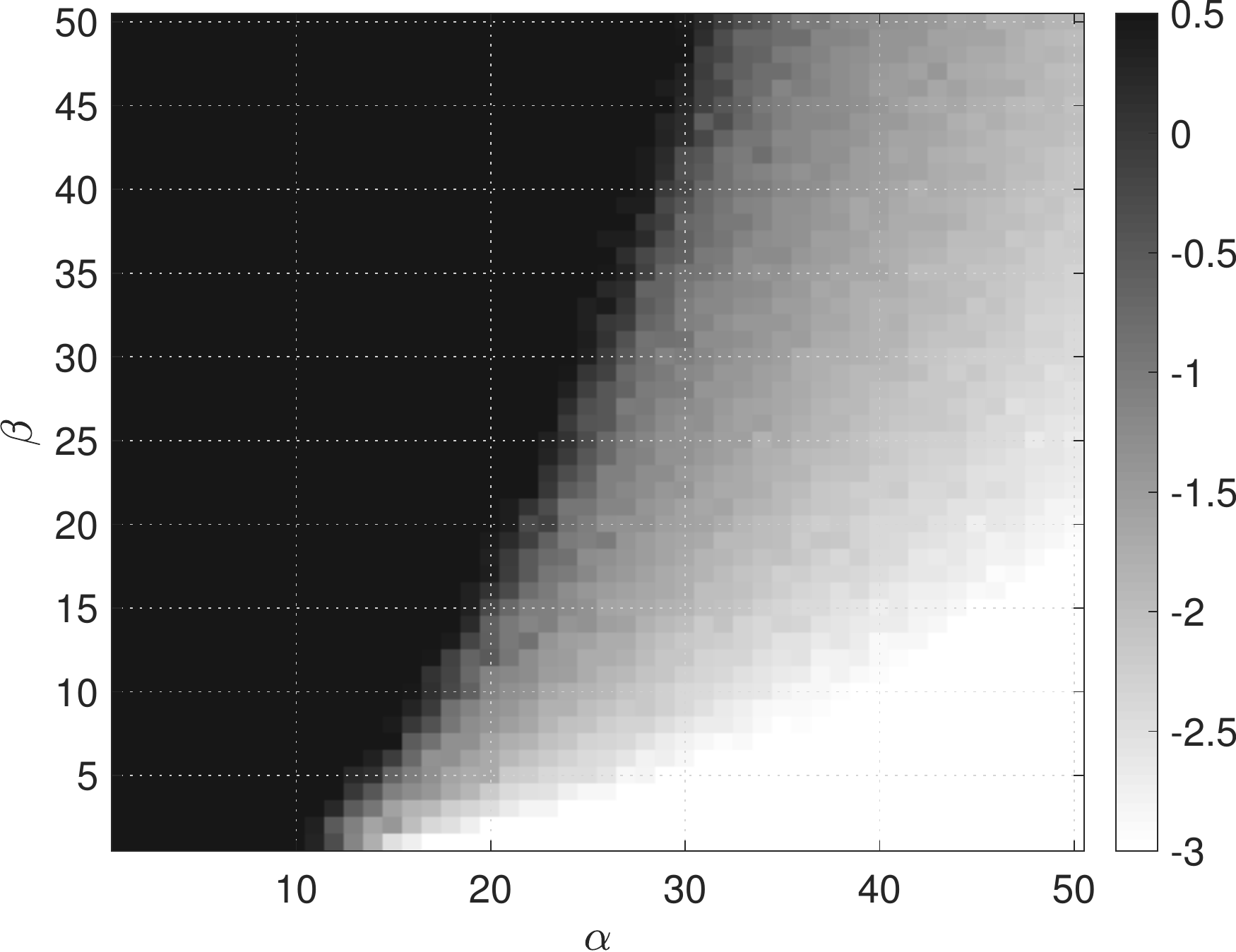}
    \label{fig:exp_three_d}}
    \caption{\textit{Results on five clusters by Algorithm~\ref{alg:spectral}.} We test under the setting $m_1 = 100$, $m_2 = m_3 = m_4 = 200$, $m_5 = 300$, and $d = 2$. We plot the success rate of exact recovery by \eqref{eq:def_failure_rate} and the synchronization error by \eqref{eq:def_error_sync}. \protect \subref{fig:exp_three_a} and \protect \subref{fig:exp_three_b}: result by Algorithm~\ref{alg:spectral}; \protect \subref{fig:exp_three_c} and \protect \subref{fig:exp_three_d}: result after the refinement step \eqref{eq:cluster_refine} in Section~\ref{sec:refinement}.
    }
    \label{fig:exp_three}
\end{figure}

We first test the performance of Algorithm~\ref{alg:spectral}. We consider the case of two clusters with equal cluster sizes, where we fix $n = 1000$ and test under two settings with $d = 2$ or $3$. In particular, since Theorem~\ref{the:cond} implies that exact recovery is possible at the regime $p, q = O(\log n/n)$, we measure the recovery performance on different $p = \alpha\log n/n$ and $q = \beta \log n/n$ with varying $\alpha$ and $\beta$. In Fig.~\ref{fig:exp_K_2_spec} we plot the success rate of exact recovery \eqref{eq:def_failure_rate} and the error of synchronization \eqref{eq:def_error_sync}. As a result, in both Fig.~\ref{fig:exp_K_2_a} and Fig.~\ref{fig:exp_two_cluster_d_2_cluster} we observe sharp phase transitions on the success rate of exact recovery. In Fig.~\ref{fig:exp_K_2_b} and Fig.~\ref{fig:exp_two_cluster_d_2_sync} the error of synchronization follows a similar pattern such that when exact recovery fails we observe a large error, and the error dramatically decreases as exact recovery is achieved. Such observations agree with our theory in Theorem~\ref{the:cond}. 

To better visualize the scaling of the phase transition curve,
in Fig.~\ref{fig:exp_scale} we plot the success rate of exact recovery, under different $\eta$ defined in \eqref{eq:p_q_assumption} with varying $p = \alpha \log n/n$ or $q = \beta \log n/n$. Specifically, we set $m_1 = m_2 = 200$, and for a fixed $\alpha$ (resp. $\beta$), we adjust $\eta$ from $0$ to $2$ by changing $\beta$ (resp. $\alpha$) accordingly. As we can see, Fig.~\ref{fig:exp_scale} implies that exact recovery can be achieved with high possibility as $\eta \leq 0.5$, which agrees with the theoretical condition \eqref{eq:p_q_assumption} in Theorem~\ref{the:cond} that exact recovery is possible as long as $\eta \leq c_0$ for some constant $c_0$. Such observation indicates the sharpness of our condition \eqref{eq:p_q_assumption}.

We further test our algorithm on a more general scenario with five clusters such that $(m_1, m_2, m_3, m_4, m_5) = (100, 200, 200, 200, 300)$ and $d = 2$. We report the result of our algorithm in terms of the metrics \eqref{eq:def_failure_rate} and \eqref{eq:def_error_sync} in Fig.~\ref{fig:exp_three_a} and Fig.~\ref{fig:exp_three_b}.
As a result, we still observe a clear phase transition boundary, which verifies our algorithm is able to handle arbitrary underlying cluster structures. 

\begin{figure}[t!]
\vspace{-0.1cm}
    \centering
    \subfloat[\scriptsize{(Ours) Before refinement}]{\includegraphics[width = 0.32\textwidth]{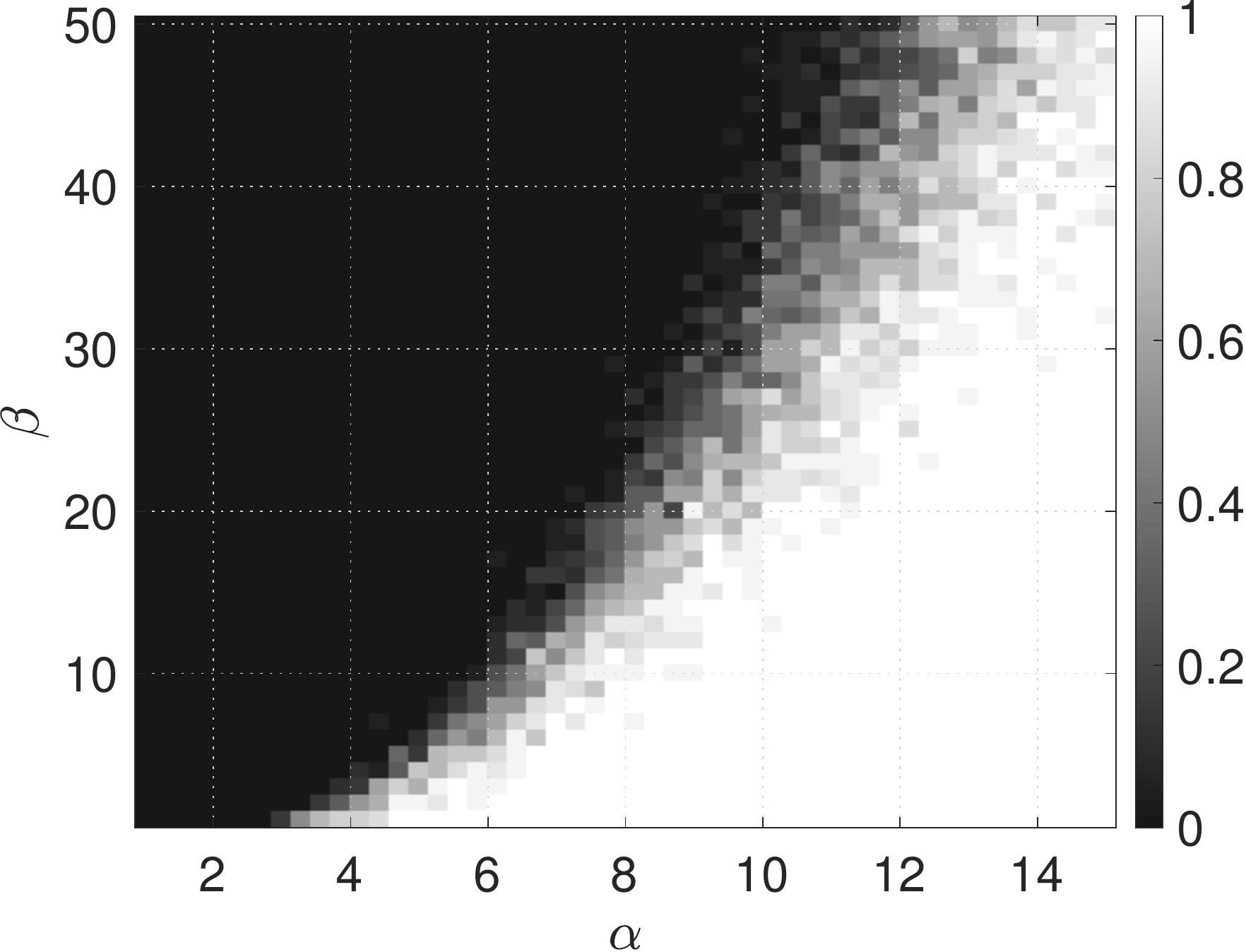}\label{fig:before_refine}} \hskip 0.17cm
    \subfloat[\scriptsize{(Ours) After refinement }]{\includegraphics[width = 0.32\textwidth]{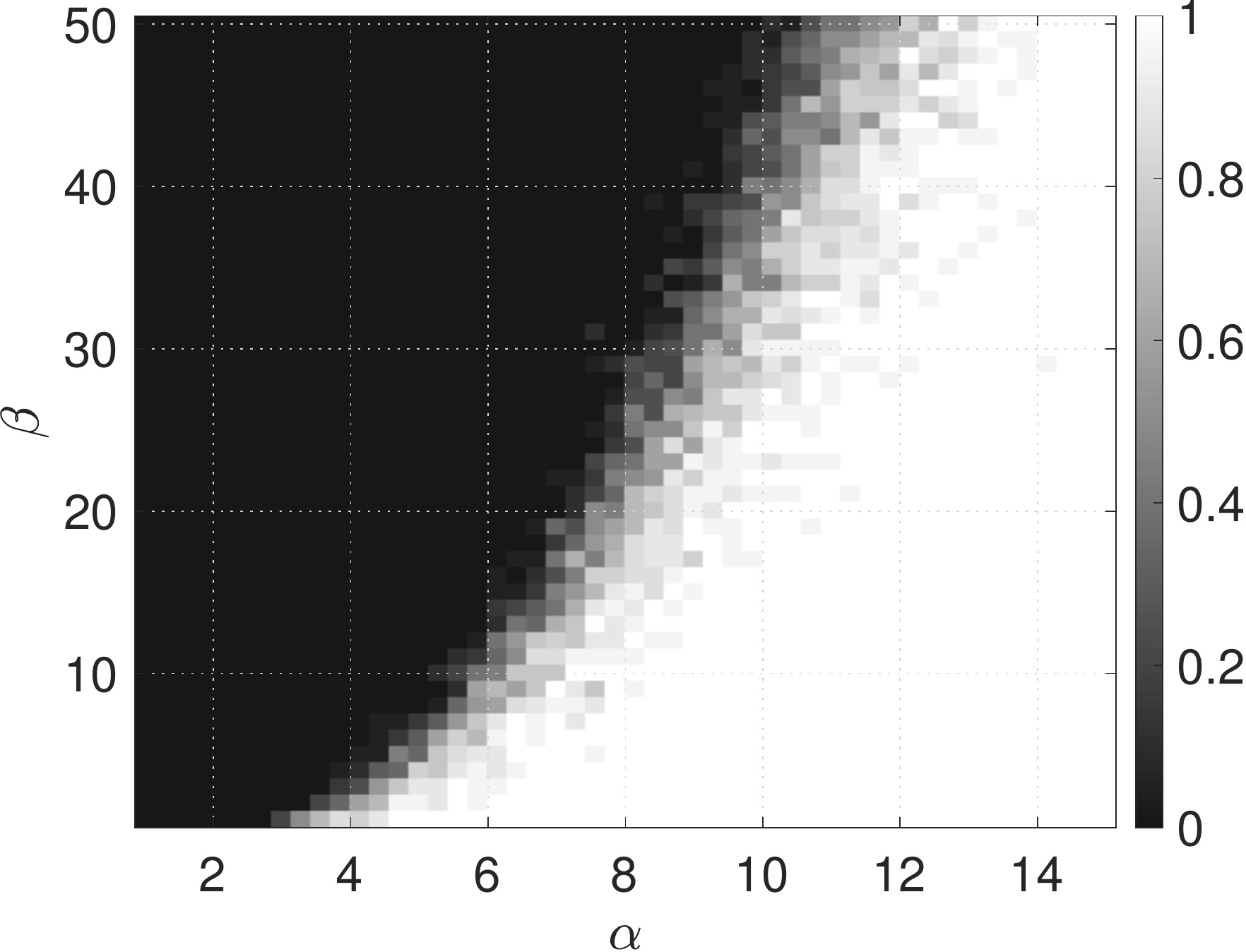}\label{fig:after_refine}} \hskip 0.17cm
    \subfloat[\scriptsize{Result by SMAC~\cite{bajaj2018smac}}]{\includegraphics[width = 0.32\textwidth]{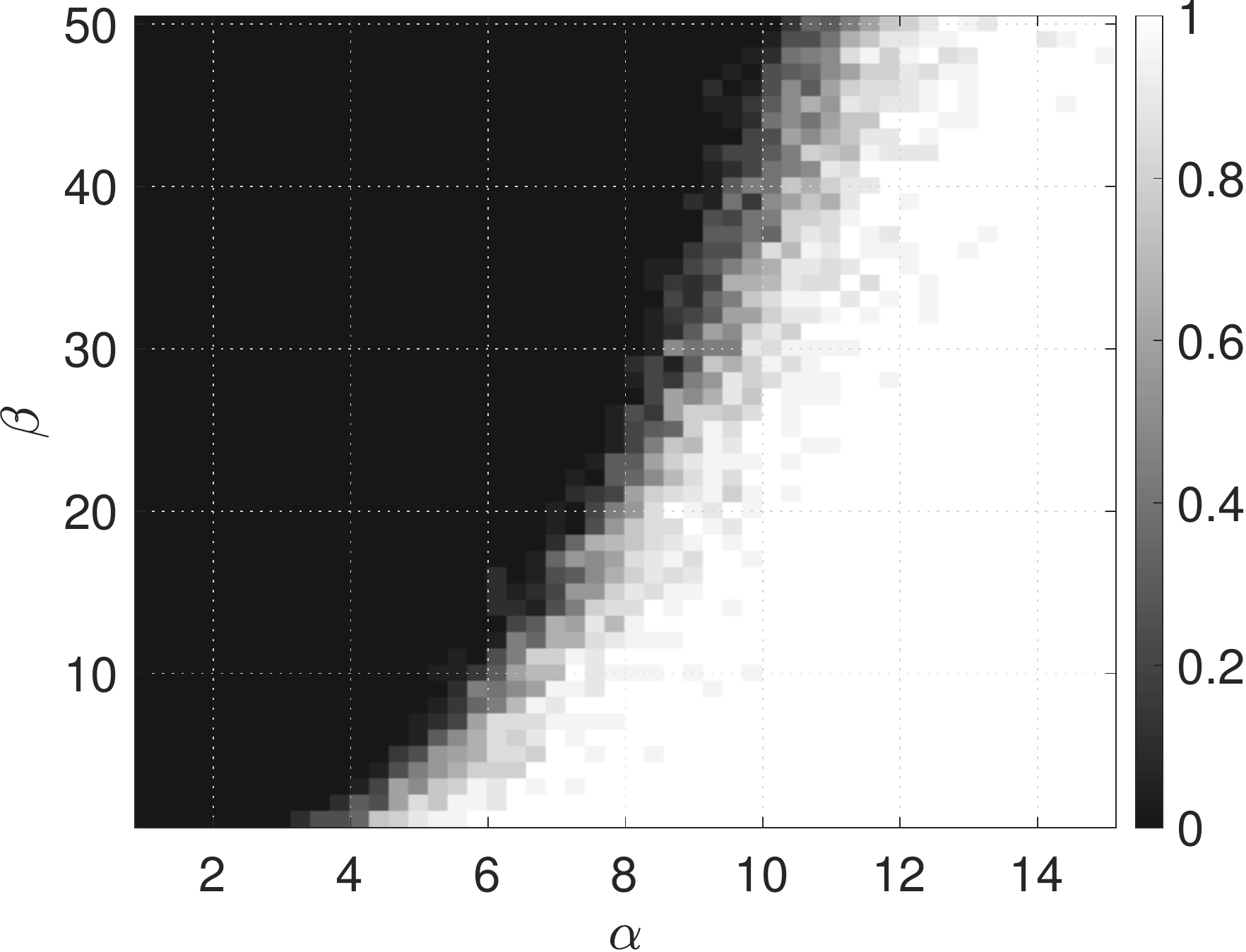}\label{fig:SMAC_refine}} 
    \caption{We test under the setting $m_1 = m_2 = 200$ and $d = 2$. We show the success rate of exact recovery of result by \protect \subref{fig:before_refine}: Algorithm~\ref{alg:spectral}; \protect \subref{fig:after_refine}: the refinement step in \eqref{eq:cluster_refine}; \protect \subref{fig:SMAC_refine}: SMAC proposed in~\cite{bajaj2018smac}. 
    }
    \label{fig:SMAC}
\end{figure}

\begin{figure}[t!]
\vspace{-0.15cm}
    \centering
    \subfloat[\scriptsize{Without spectral decomposition}]{\includegraphics[width = 0.43\textwidth]{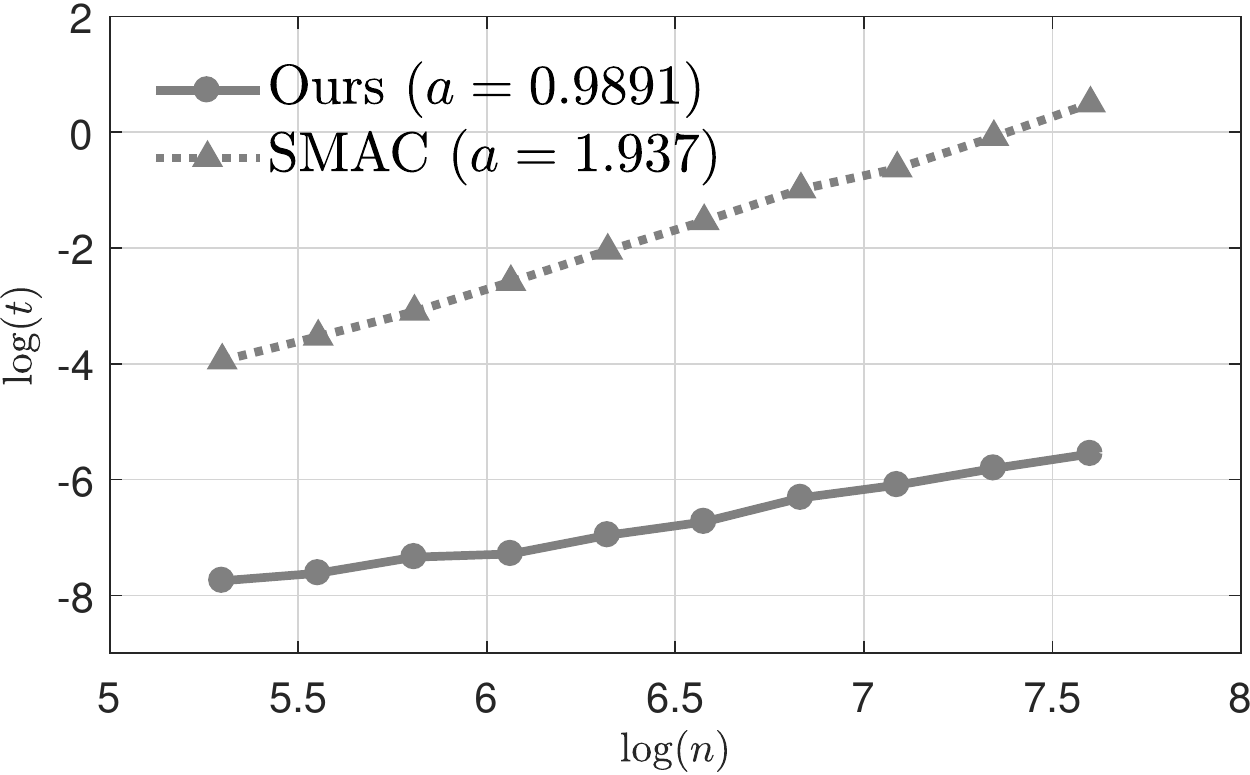}\label{fig:runtime_noeigen}} \hskip 0.2cm
    \subfloat[\scriptsize{With spectral decomposition}]{\includegraphics[width = 0.43\textwidth]{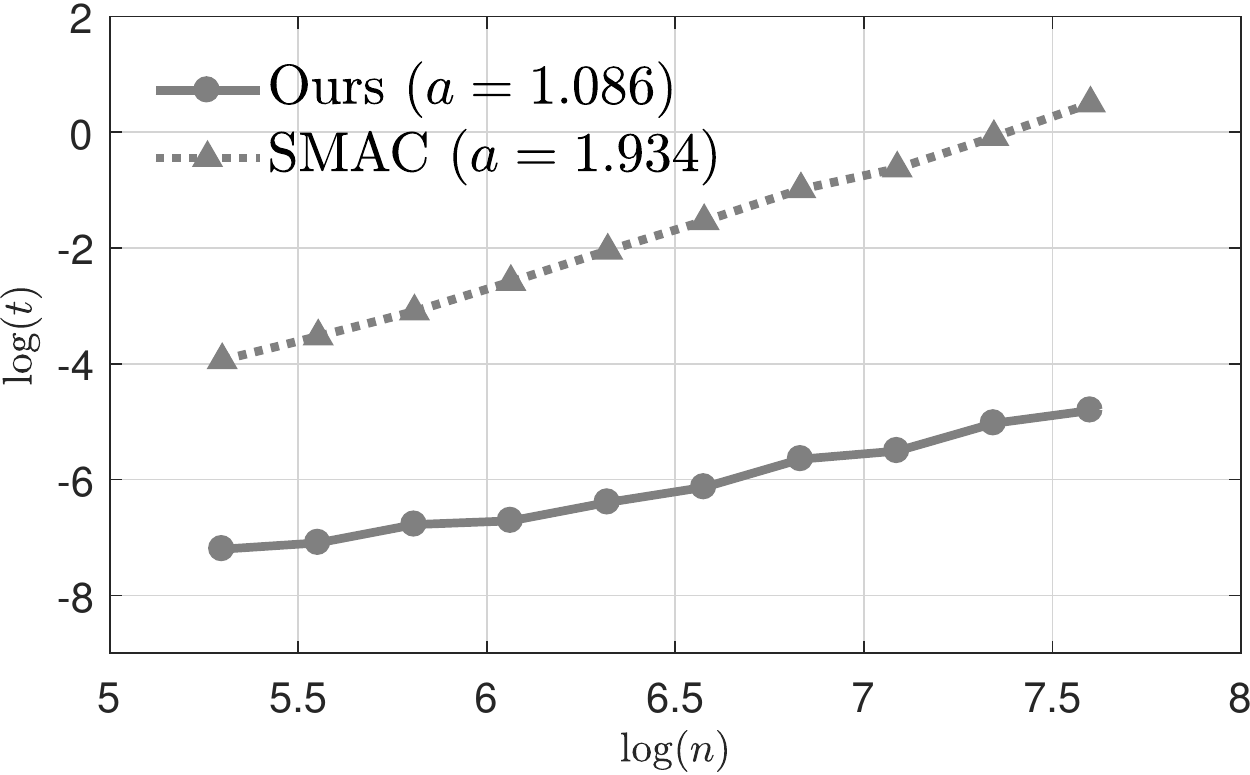}\label{fig:runtime_eigen}}\\
    \caption{\textit{Runtime test}. 
    We plot the runtime (denoted by $t$) of different algorithms v.s. different sizes of the adjacency matrix $\mathbf{A}$ (denoted by $n$) in $\log$-scale. The slope of each run-time v.s. problem size curve is noted by $a$. \protect \subref{fig:runtime_noeigen} and \protect \subref{fig:runtime_eigen} represent the runtime that include and exclude the step 1 (spectral decomposition) in Algorithm~\ref{alg:spectral}, respectively. {\color{black} We set $p = q = 10\log n/n$, $K = 2$ and $d = 2$.} The refinement step in Algorithm~\ref{alg:spectral} is excluded.
    }
    \label{fig:runtime}
\end{figure}

{\color{black}
We also test the optional refinement step for cluster memberships described in Section~\ref{sec:refinement}, where the threshold $\epsilon$ in \eqref{eq:refine_S_k} is specified in a way that 10$\%$ nodes are included in the set $S_{\epsilon}$. The result is then displayed in Fig.~\ref{fig:exp_three_c} and Fig.~\ref{fig:exp_three_d}. We also show another one in Fig.~\ref{fig:SMAC} under the setting $m_1 = m_2 = 200$ and $d = 2$.
As a result, on both examples, we observe a clear improvement on the phase transition boundary of exact recovery of the cluster memberships after the refinement step is applied, which demonstrates the efficacy of refinement.
}

In addition, we compare our algorithm with an existing method SMAC proposed in~\cite{bajaj2018smac}, which is also based on spectral decomposition. Although the original version of SMAC is designed for permutation group synchronization, it can be easily extended to handle the orthogonal group. We test under the setting of $m_1 = m_2 = 200$ and $d = 2$. In Fig~\ref{fig:SMAC} we present the success rate of recovery by our algorithm and SMAC. As we can see, the result by our algorithm after the refinement has a similar phase transition boundary as SMAC. However, our method has much less computational complexity than SMAC, as we shall see in the following. 

{
\color{black}
To investigate the computational complexity, we further test the runtime of the proposed Algorithm~\ref{alg:spectral} (excluding the refinement step) and SMAC~\cite{bajaj2018smac} in Fig.~\ref{fig:runtime}, under the setting $p = q = 10\log n/n$, $K = 2$, and $d = 2$, and we let $n$ vary from $200$ to $2000$. All the experiments are performed using the same computational resource mentioned at the beginning of Section~\ref{sec:exp}, and we obtain the average timing over 50 trials. Fig.~\ref{fig:runtime_noeigen} and Fig.~\ref{fig:runtime_eigen} display the runtime results including and excluding the step 1 (spectral decomposition) in Algorithm~\ref{alg:spectral}, respectively. From the slopes of the curves, we observe that our algorithm without refinement scales almost linearly with the data size $n$, and the slope slightly increases as the spectral decomposition is included (recall that the order is of $O(n\log n)$ in theory due to the sparsity). In contrast, SMAC scales quadractically in $n$, since it performs synchronization pairwisely. Such observation agrees with our complexity analysis in Table~\ref{tab:complexity} and demonstrates the efficiency of Algorithm~\ref{alg:spectral}.
}

\begin{figure}[t!]
\vspace{-0.15cm}
    \centering
    \subfloat[\scriptsize{$d = 2$}]{\includegraphics[width = 0.32\textwidth]{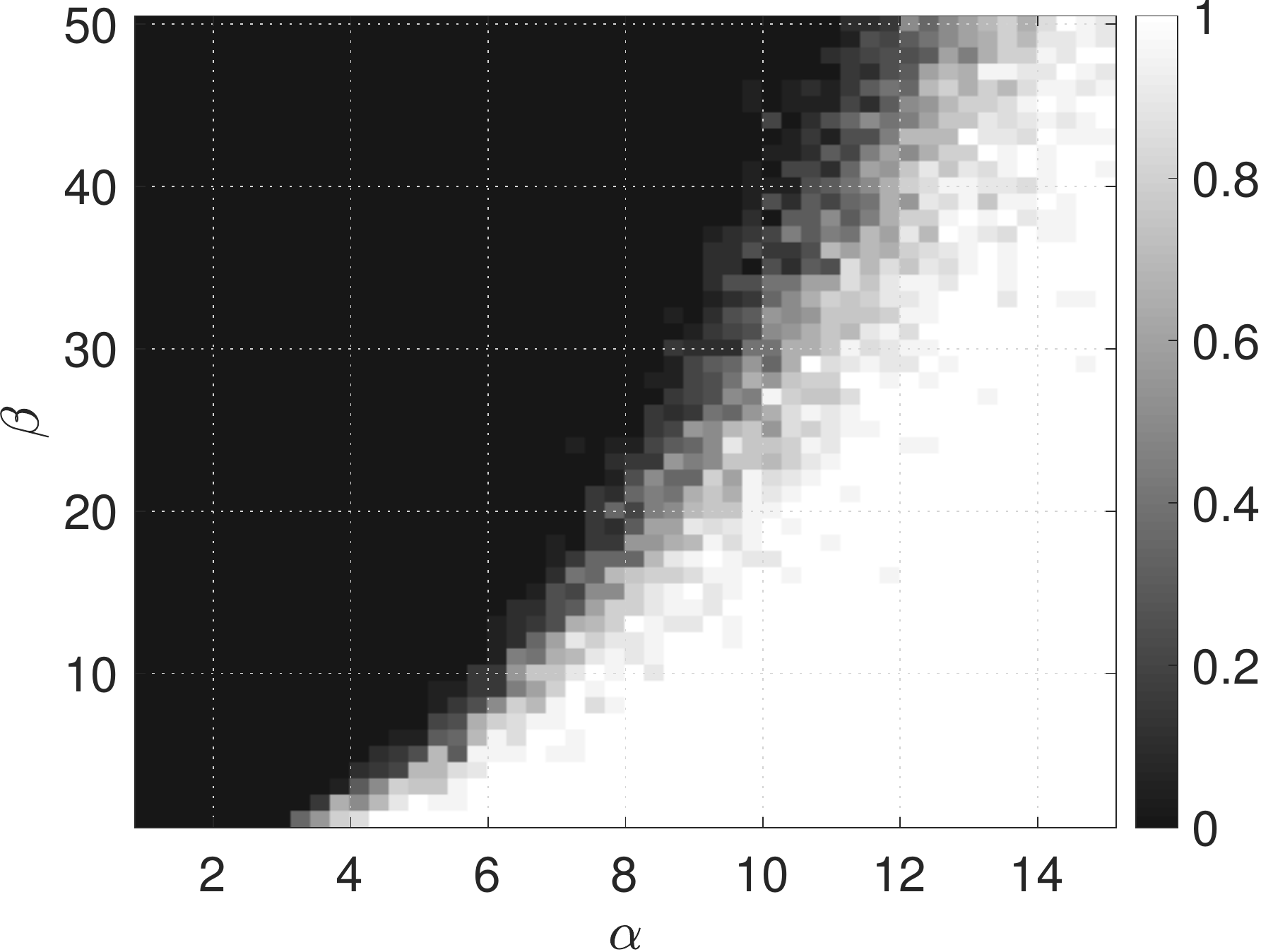}\label{fig:compare_d_2}}\hskip 0.17cm
    \subfloat[\scriptsize{$d = 10$}]{\includegraphics[width = 0.32\textwidth]{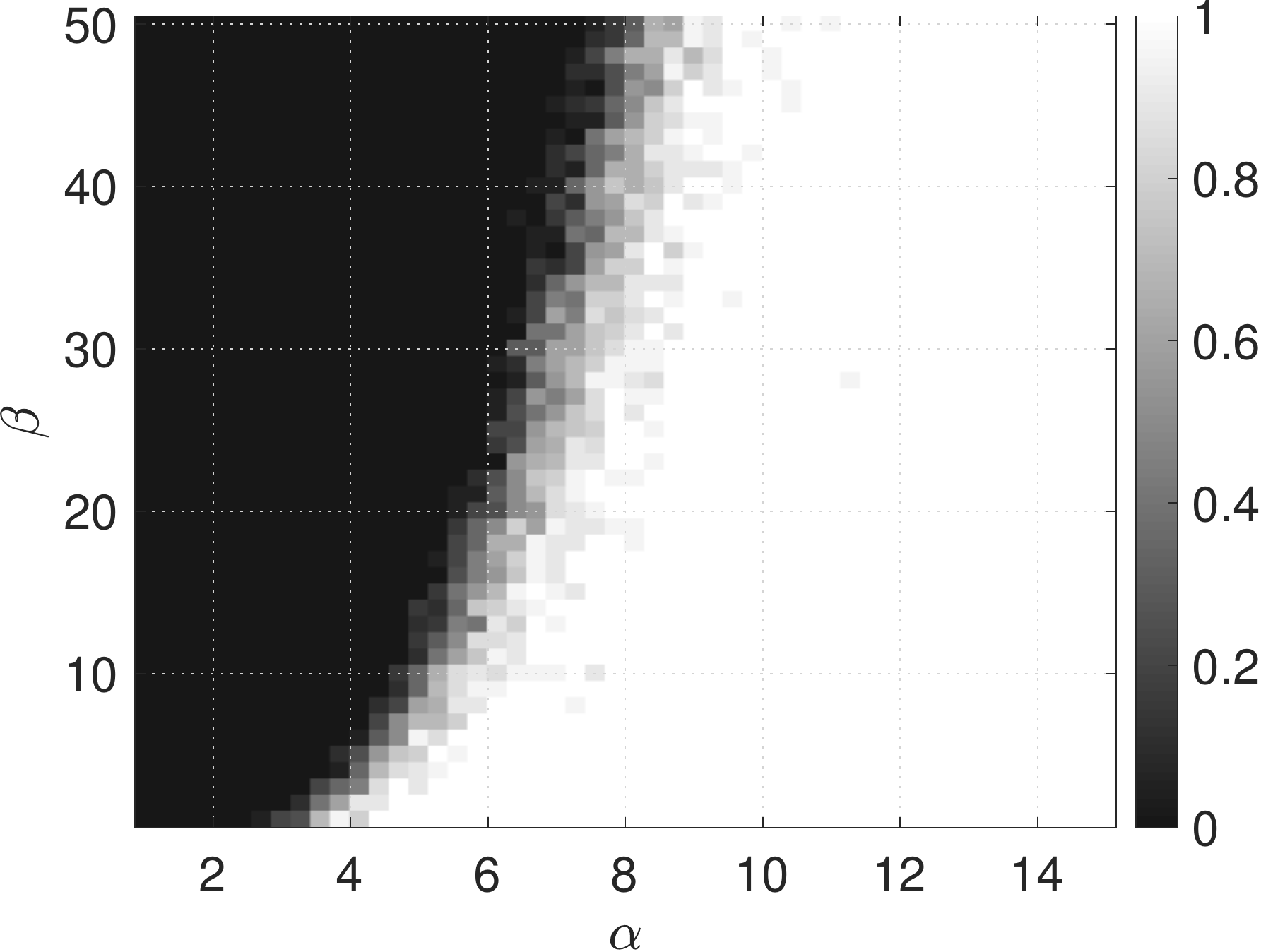}\label{fig:compare_d_10}} \hskip 0.17cm
    \subfloat[\scriptsize{$\min_{i \in C_1} \frac{\|\bm{R}_{1i}\|_\mathrm{F}}{ \|\bm{R}_{2i}\|_\mathrm{F}}$ }]{\includegraphics[width = 0.315\textwidth]{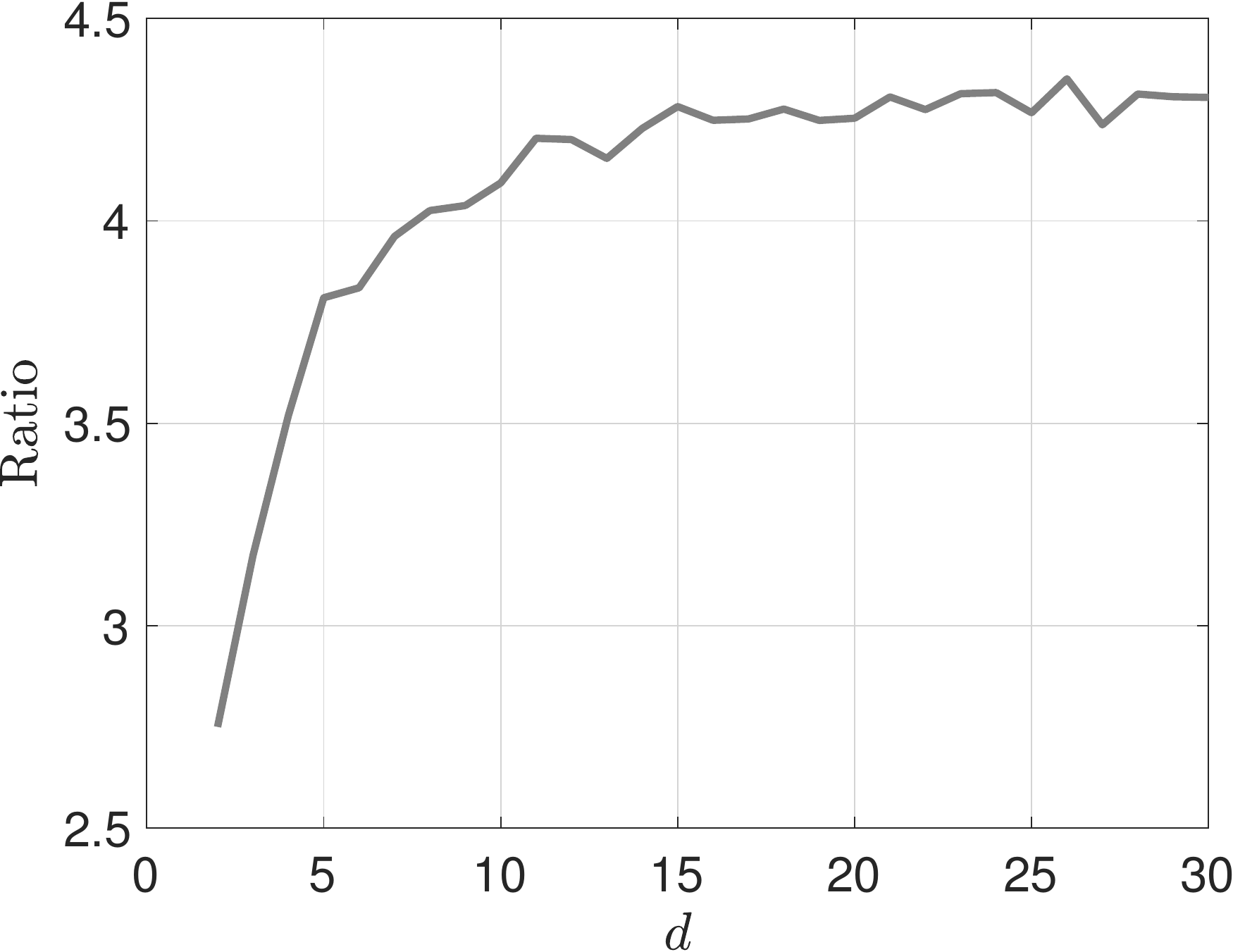}\label{fig:ratio}}\\
    \caption{\protect \subref{fig:compare_d_2} and \protect \subref{fig:compare_d_10}: We compare the success rate of exact recovery between $d = 2$ and $d = 10$, under $m_1 = m_2 = 500$; \protect \subref{fig:ratio}: We fix $p = q = 0.5$ and plot the minimum ratio $\min_{i \in C_1} \|\bm{R}_{i1}\|_\mathrm{F} / \|\bm{R}_{i2}\|_\mathrm{F}$ under different $d$ from $2$ to $30$.}
    \label{fig:compare_ratio}
\end{figure}

Besides, empirically we observe the phase transition boundary of exact recovery changes as the dimension of orthogonal group $d$ increases, as shown in Fig.~\ref{fig:compare_d_2} and Fig.~\ref{fig:compare_d_10} where we compare the result when $d = 2$ with $d = 10$. One can see that under the same setting of $(p,q,n)$, exact recovery gets easier as $d$ increases. To investigate the reason behind,
we randomly pick a node $i \in C_1$ and compute the ``signal-to-noise ratio'' $\|\bm{R}_{1i}\|_\mathrm{F} / \|\bm{R}_{2i}\|_\mathrm{F}$, where $\|\bm{R}_{1i}\|_\mathrm{F}$ and $\|\bm{R}_{2i}\|_\mathrm{F}$ 
can be seen as the signal and noise level respectively. Clearly, a larger ratio indicates a lower error probability of clustering. Then, we fix the parameters $p = q = 0.5, m_1 = m_2 = 500$, and plot the minimum ratio $\min_{i \in C_1} \|\bm{R}_{1i}\|_\mathrm{F}/\|\bm{R}_{2i}\|_\mathrm{F}$ 
among all nodes $i \in C_1$ under different $d$ varying from $2$ to $30$ in Fig.~\ref{fig:ratio}. As we can see, the ratio increases and converges as $d$ increases, which indicates that clustering becomes easier on a larger dimension $d$. Unfortunately, such phenomenon is not characterized by our theory in Section~\ref{sec:analysis}, and we leave the theoretical investigation as a future study.

	\section{Discussion and conclusion}
\label{sec:conclusion}

In this work, we study the joint community detection and orthogonal synchronization by proposing a spectral method based algorithm. The proposed method is extremely convenient to use which only consists of a spectral decomposition step followed by a blockwise column pivoted QR-factorization~(CPQR). As a simple variant of CPQR, blockwise CPQR is designed to ensure that the blockwise nature of the matrices involved is captured and the blockwise structure is always preserved. Such QR variant is flexible and can be applied to other applications that require QR-factorization on a block matrix. In terms of the time complexity, our algorithm scales linearly with the number of data points, which exhibits a great advantage over other existing methods that at least requires $O(n^2)$ complexity. In addition, under the scenario of two equal-sized clusters, we provide a near-optimal condition which guarantees the underlying cluster memberships are exactly recovered, and the orthogonal transforms are stably recovered. In particular, such condition is obtained by deriving a blockwise error bound on each block of eigenvectors, using the leave-one-out technique~\cite{zhong2018near, abbe2020entrywise} rather than Davis-Kahan theorem~\cite{davis1970rotation}.
Also, we point out that our theory can be extended to the more general case when having arbitrary number of clusters of different sizes. To evaluate our algorithm, we perform a series of numerical experiments that demonstrate the efficacy of our algorithm and confirm our theoretical characterization of the sharp phase transition of recovery.

{\color{black}
In addition, we have considered an extension of the current result to cover a ``noisy'' version of the problem by considering additive Gaussian noise model on the pairwise group transformation. We perform some initial study on this noise model which is given in Appendix~\ref{sec:new_noise_model}, where we (both theoretically and empirically) show that our proposed Algorithm~\ref{alg:spectral} is still able to robustly recover the cluster memberships and the orthogonal transforms under mild additive noise levels.

There are several directions that can be further explored. First, it is natural to expect that the proposed algorithm can be applied to other groups such as permutation group. Second, one should be able to extend the theoretical results to more general scenarios with unbalanced cluster sizes and larger number of clusters,\footnote{Here, we still assume the number of clusters scales as a constant i.e. $K = O(1)$ and each cluster size scales in the same order as the total number of nodes i.e. $m_i = \Theta(n), \; n=1,\ldots,K$.} but several major changes are necessary: for example, in general the spectral norm of the residual $\|\bm{\Delta}\|$ given in Lemma~\ref{lemma:Delta_M} should depend on $K$ but the current proof (see Appendix~\ref{sec:proof_lemma_delta_norm_2}) only considers two clusters; also, the analysis of exact recovery by CPQR certainly becomes more challenging as more than one pivot are needed to be considered when $K > 2$. 
}

	\section*{Acknowledgments}
	ZZ and YF acknowledge the support from NSF grant DMS-1854791 and Alfred P. Sloan foundation. YK acknowledges support from NSF grant DMS-2111563.
	
	\appendix
	
	\section{Proof of the main theorems}
\label{sec:proof_main_theorems}
This section is devoted to the proof of the main theorems given in Section~\ref{sec:analysis}. 

\subsection{Important technical ingredients}
\label{sec:tech_ingreidient}
We first introduce several important technical ingredients that will be frequently used in our analysis. 

\vspace{0.1cm}
\paragraph{$\bullet$ Polar decomposition} 
Recall the polar decomposition and the polar factor $\mathcal{P}(\cdot)$ in Definition~\ref{def:polar}, 
then the following bound will be frequently used in our analysis.
\begin{lemma}[\textup{\cite[Lemma 5.2]{ling2020near}}]
Let $\bm{X}, \bm{Y} \in \mathbb{R}^{d \times d}$ be two invertible matrices, then
\begin{equation*}
    \|\mathcal{P}(\bm{X}) - \mathcal{P}(\bm{Y})\| \leq 2\sqrt{2}\min\{\sigma^{-1}_{\mathrm{min}}(\bm{X}), \sigma^{-1}_{\mathrm{min}}(\bm{Y})\}\|\bm{X} - \bm{Y}\|,
\end{equation*}
where $\sigma_{\mathrm{min}}(\cdot)$ denotes the smallest singular value of a matrix.
\label{lemma:PX_PY}
\end{lemma}

Lemma~\ref{lemma:PX_PY} implies that $\|\mathcal{P}(\bm{X}) - \mathcal{P}(\bm{Y})\|$ is bounded as long as $\|\bm{X} - \bm{Y}\|$ is bounded and all the singular values of $\bm{X}$ and $\bm{Y}$ are bounded away from zero.

\vspace{0.1cm}
\paragraph{$\bullet$ Matrix perturbation theory} 
We include the following classic results in matrix perturbation theory, where Theorem~\ref{the:weyl} and Theorem~\ref{the:davis_kahan} study the perturbation on singular values and eigenvectors, respectively.
\begin{theorem}[Weyl's inequality~\cite{stewart1998perturbation}]
Let $\bm{X}$ and $\bm{Y}$ be two matrices of the same size, then
\begin{equation*}
    |\sigma_i(\bm{X}) - \sigma_i(\bm{Y})| \leq \|\bm{X} - \bm{Y}\|, \quad \forall i 
\end{equation*}
where $\sigma_i(\cdot)$ denotes the $i$-th largest singular value of a matrix.
\label{the:weyl}
\end{theorem}

\begin{theorem}[Davis-Kahan theorem~\cite{davis1970rotation}]
Let $\bm{X}$ and $\hat{\bm{X}} = \bm{X} + \bm{\Delta}$ be two $n \times n$ symmetric matrices with eigenvalues $\{\lambda_i\}_{i = 1}^n$ and $\{\hat{\lambda}_i\}_{i = 1}^n$ respectively, 
with the following eigen-decompositions:
\begin{equation*}
    \bm{X} = 
    \begin{bmatrix}
    \bm{\Phi}_0 &\bm{\Phi}_1
    \end{bmatrix}
    \begin{bmatrix}
    \bm{\Lambda}_0 &\bm{0}\\
    \bm{0} &\bm{\Lambda}_1
    \end{bmatrix}
    \begin{bmatrix}
    \bm{\Phi}_0 &\bm{\Phi}_1
    \end{bmatrix}^\top, \quad
    \hat{\bm{X}} = 
    \begin{bmatrix}
    \hat{\bm{\Phi}}_0 &\hat{\bm{\Phi}}_1
    \end{bmatrix}
    \begin{bmatrix}
    \hat{\bm{\Lambda}}_0 &\bm{0}\\
    \bm{0} &\hat{\bm{\Lambda}}_1
    \end{bmatrix}
    \begin{bmatrix}
    \hat{\bm{\Phi}}_0 &\hat{\bm{\Phi}}_1
    \end{bmatrix}^\top
\end{equation*}
where $\bm{\Lambda}_0$ and $\bm{\Lambda}_1$ $( \text{resp. } \hat{\bm{\Lambda}}_0 \text{ and } \hat{\bm{\Lambda}}_1)$ are diagonal matrices that contain the top $r$ eigenvalues $\{\lambda_i\}_{i = 1}^r$ $( \text{resp. } \{\hat{\lambda}_i\}_{i = 1}^r)$ and the remaining eigenvalues, respectively. $\bm{\Phi}_k, \hat{\bm{\Phi}}_k$ for $k = 0, 1$ are the normalized eigenvectors. Then,
\begin{equation*}
    \|\hat{\bm{\Phi}}_1^\top\bm{\Phi}_0\| \leq \frac{\|\bm{\Delta}\bm{\Phi}_0\|}{\delta}
\end{equation*}
where $\delta := |\lambda_{\mathrm{min}}(\hat{\bm{\Lambda}}_0) - \lambda_{\mathrm{max}}(\bm{\Lambda}_1)|$ denotes the spectral gap between $\hat{\bm{\Lambda}}_0$ and $\bm{\Lambda}_1$.
\label{the:davis_kahan}
\end{theorem}

\begin{remark}
{
\color{black}
A common technique for analyzing $\delta$ above is by obtaining the following lower bound:
\begin{equation*}
\begin{aligned}
    \delta &= |\lambda_{\mathrm{min}}(\hat{\bm{\Lambda}}_0) - \lambda_{\mathrm{max}}(\bm{\Lambda}_1)| = |\lambda_r(\hat{\bm{X}}) - \lambda_{r+1}(\bm{X})| \\
    &= |\lambda_r(\hat{\bm{X}}) - \lambda_r(\bm{X}) + \lambda_r(\bm{X}) - \lambda_{r+1}(\bm{X})| \\
    &\overset{(a)}{\geq} |\lambda_r(\bm{X}) - \lambda_{r+1}(\bm{X})| - |\lambda_r(\hat{\bm{X}}) - \lambda_r(\bm{X})|\\
    &\overset{(b)}{\geq} |\lambda_r(\bm{X}) - \lambda_{r+1}(\bm{X})| - \|\bm{\Delta}\|
\end{aligned}
\end{equation*}
where $(a)$ comes from the triangular inequality and $(b)$ applies Theorem~\ref{the:weyl}. Here, we assume that $|\lambda_r(\bm{X}) - \lambda_{r+1}(\bm{X})| > \|\bm{\Delta}\|$ such that the spectral gap is greater than the norm of perturbation (otherwise the lower bound becomes trivial). Indeed, this is the case in our analysis because $\|\bm{\Delta}\| = O(\sqrt{\log n})$ under the condition that $p = \alpha \frac{\log n}{n}$ and $q = \beta \frac{\log n}{n}$ and it is always of lower order compared to $|\lambda_r(\bm{X}) - \lambda_{r+1}(\bm{X})| = \Omega(\log n)$. Therefore we have $\delta \approx |\lambda_r(\bm{X}) - \lambda_{r+1}(\bm{X})|$. For details of the analysis on $\|\bm{\Delta}\|$, see Lemma~\ref{lemma:delta_norm_2} and its proof.
}
\end{remark}

\begin{lemma}[\textup{\cite[Lemma 5.6]{ling2020near}}]
Suppose $\bm{X}, \bm{Y} \in \mathbb{R}^{n \times d}$ are two tall orthogonal matrices, i.e. $\bm{X}^\top\bm{X} = \bm{Y}^\top\bm{Y} = \bm{I}_d$. Then,
\begin{equation*}
    \left\|\bm{Y} - \bm{X}\bm{O}\right\| \leq 2\left\|(\bm{I}_{n} - \bm{X}\bm{X}^\top)\bm{Y}\right\|,
\end{equation*}
where $\bm{O} = \mathcal{P}(\bm{X}^\top\bm{Y})$.
\label{lemma:error2inner}
\end{lemma}

As a result, plugging Lemma~\ref{lemma:error2inner} into Theorem~\ref{the:davis_kahan} yields
\begin{equation}
    \|\bm{\Phi}_0 - \hat{\bm{\Phi}}_0\bm{O}\| \leq 2\|\hat{\bm{\Phi}}_1^\top\bm{\Phi}_0\| \leq \frac{2\|\bm{\Delta}\bm{\Phi}_0\|}{\delta},
    \label{eq:Davis_kahan_bound}
\end{equation}
where $\bm{O} = \mathcal{P}(\bm{\Phi}_0^\top\hat{\bm{\Phi}}_0)$. \eqref{eq:Davis_kahan_bound} will be frequently used in our analysis.

\vspace{0.1cm}
\paragraph{$\bullet$ Analysis of $\bm{\Delta}$} 
Recall the residual $ \bm{\Delta} = \bm{A} - \mathbb{E}[\bm{A}]$ defined in \eqref{eq:A_decompose}. Then under the setting of two equal-sized clusters and the assumption $\bm{O}_i = \bm{I}_d, i = 1,\ldots, n$, by definition each block $\bm{\Delta}_{ij} \in \mathbb{R}^{d \times d}$ for $j \neq i$ satisfies
\begin{equation}
    \begin{aligned}
    &\text{when $\kappa(j) = \kappa(i)$: } \; \bm{\Delta}_{ij} = 
    \begin{cases}
    (1-p)\bm{I}_d, &\; \text{w.p. } p,\\
    -p \bm{I}_d, &\; \text{w.p. } 1-p,
    \end{cases} \\
    &\text{when $\kappa(j) \neq \kappa(i)$: } \; \bm{\Delta}_{ij} = 
    \begin{cases}
    \bm{O}_{ij}, &\; \text{w.p. } q,\\
    \bm{0},  &\; \text{w.p. } 1-q, 
    \end{cases}
    \end{aligned}
    \label{eq:Delta_ij}
\end{equation}
where $\bm{O}_{ij} \sim \text{Unif}(\mO(d))$. 
Also, from the setting $\bm{A}_{ii} = \bm{0}$ we have $\bm{\Delta}_{ii} = \bm{0}, i = 1,\ldots, n$. Given the above, we obtain the following inequalities for $\bm{\Delta}$.
\begin{lemma}
Given $\bm{\Delta}$ defined in \eqref{eq:Delta_ij}, suppose the model parameters $p, q$ satisfy $p = \Omega(\log n/n), 1-p = \Omega(\log n/n)$ and $q = \Omega(\log n/n)$, then 
\begin{equation*}
    \begin{aligned}
    \|\bm{\Delta}\| &\lesssim \sqrt{p(1-p)n} + \sqrt{qn}
    \end{aligned}
\end{equation*}
with probability $1 - O(n^{-1})$. Also, for the $i$-th block row $\bm{\Delta}_{i \cdot}$ and $\bm{\Delta}^{(i)}$ defined in \eqref{eq:A_i_definition}, we have $\|\bm{\Delta}_{i \cdot}\| \leq  \|\bm{\Delta}\|$ and $\|\bm{\Delta}^{(i)}\| \leq  \|\bm{\Delta}\|$ for $i = 1,\ldots, n$.
\label{lemma:delta_norm_2}
\end{lemma}

\begin{lemma}
Given any block matrix $\bm{M} \in \mathbb{R}^{nd \times r}$ with $n$ block rows, suppose $p = \Omega(\log n/n), 1-p = \Omega(\log n/n)$ and $q = \Omega(\log n/n)$, then for each block row $\bm{\Delta}_{i \cdot}$ it satisfies
\begin{equation*}
    \|\bm{\Delta}_{i \cdot}\bm{M}\| \lesssim \sqrt{(p(1-p) + q)n\log(nd)}\max_j\|\bm{M}_{j\cdot }\|
\end{equation*}
with probability $1 - O(n^{-1})$.
\label{lemma:Delta_M}
\end{lemma}

The proofs of both Lemma~\ref{lemma:delta_norm_2} and Lemma~\ref{lemma:Delta_M} are deferred to Appendix~\ref{sec:proof_lemmas}.

\subsection{Proof of Theorem~\ref{the:phi_psi_i}} 
\label{sec:proof_the_1}
In this section, we prove Theorem~\ref{the:phi_psi_i} which provides the blockwise error bound in \eqref{eq:bound_max_block}. 
We first introduce a series of necessary inequalities, followed by the proofs of Lemmas~\ref{the:Phi_f_Phi} and \ref{the:phi_i_phi_j} that are presented in Section~\ref{sec:proof_sketch}. Then we end up with the proof of Theorem~\ref{the:phi_psi_i}. All the lemmas introduced in this section are proved in Appendix~\ref{sec:proof_lemmas_the_1}.

To begin with, recall that $\bm{\Phi}$, $\bm{\Psi}$, and $\bm{\Phi}^{(i)}$ are defined as the top $Kd$ eigenvectors of $\bm{A}$, $\mathbb{E}[\bm{A}]$, and the auxiliary matrix $\bm{A}^{(i)}$ defined in \eqref{eq:A_i_definition},  respectively. We first bound the difference between $\bm{\Phi}$~(or $\bm{\Phi}^{(i)}$) and $\bm{\Psi}$ in the following:

\begin{lemma}
For a sufficiently large $n$, let $\tau := \frac{\sqrt{p(1-p)} + \sqrt{q}}{p\sqrt{n}}$, then the following satisfies with probability $1 - O(n^{-1})$.
\begin{alignat*}{2}
    \|\bm{\Phi} - \bm{\Psi}\bm{O}\| &\lesssim \tau, \\
    \|\bm{\Phi}^{(i)} - \bm{\Psi}\bm{O}^{(i)}\| &\lesssim \tau, \\
    1 - \sigma_{\mathrm{min}}(\bm{\Psi}^\top\bm{\Phi}) &\lesssim \tau, \\
    1 - \sigma_{\mathrm{min}}(\bm{\Psi}^\top\bm{\Phi}^{(i)}) &\lesssim \tau,
\end{alignat*}
where $\bm{O} = \mathcal{P}(\bm{\Psi}^\top\bm{\Phi})$ and $\bm{O}^{(i)} = \mathcal{P}(\bm{\Psi}^\top\bm{\Phi}^{(i)})$.
\label{lemma:Psi_Phi}
\end{lemma}

The following lemma bounds the difference between $\bm{\Phi}$ and $\bm{\Phi}^{(i)}$. One can expect that such difference is tiny since $\bm{A}$ only differs from $\bm{A}^{(i)}$ on its $i$-th block column and $i$-th block row.

\begin{lemma}
Under the condition \eqref{eq:p_q_assumption}, as $n$ is sufficiently large, we have
\begin{equation*}
    \begin{aligned}
    \|\bm{\Phi} - \bm{\Phi}^{(i)}\bm{O}^{(i)}\| &\lesssim \eta \max_j \|\bm{\Phi}_{j\cdot}\|, \\
    \max_j \|\bm{\Phi}_{j\cdot}^{(i)}\| &\lesssim \max_j\|\bm{\Phi}_{j\cdot}\|
    \end{aligned}
\end{equation*}
with probability $1 - O(n^{-1})$, where $\bm{O}^{(i)} = \mathcal{P}((\bm{\Phi}^{(i)})^\top\bm{\Phi})$. 
\label{lemma:Phi_Phi_i}
\end{lemma}

The following lemma further bounds the blockwise difference between $\bm{\Phi}^{(i)}$ and $\bm{\Psi}$.
\begin{lemma}
Under the condition \eqref{eq:p_q_assumption}, for a sufficiently large $n$ we have
\begin{equation*}
\begin{aligned}
    \max_{j}\|\bm{\Phi}_{j\cdot}^{(i)} - \bm{\Psi}_{j\cdot}\bm{S}^{(i)}\| &\lesssim \max_j\|\bm{\Phi}_{j\cdot}\|, \\
    \max_j\|\bm{\Phi}_{j\cdot}\| &\geq \frac{1}{\sqrt{n}}
\end{aligned}
\end{equation*}
with probability $1 - O(n^{-1})$, where $\bm{S}^{(i)} = \mathcal{P}(\bm{\Psi}^\top\bm{\Phi}^{(i)})$. 
\label{lemma:phi_i_psi_j}
\end{lemma}

\begin{lemma}
Under the condition \eqref{eq:p_q_assumption}, for a sufficiently large $n$ we have
\begin{equation*}
    \|\bm{S}^{(i)}\bm{O}^{(i)} - \bm{O}\| \lesssim \eta \max_j \|\bm{\Phi}_{j\cdot}\|,
\end{equation*}
where $\bm{O} = \mathcal{P}(\bm{\Psi}^\top \bm{\Phi})$, $\bm{O}^{(i)} = \mathcal{P}(\bm{(\Phi}^{(i)})^\top \bm{\Phi})$ and $\bm{S}^{(i)} = \mathcal{P}(\bm{\Psi}^\top \bm{\Phi}^{(i)})$.
\label{lemma:S_i_O_i_O}
\end{lemma}

As we can see, most of the statistics in the previous lemmas involves $\max_{j}\|\bm{\Phi}_{j\cdot}\|$, and from Lemma~\ref{lemma:Phi_Phi_i} we have $\max_{j}\|\bm{\Phi}_{j\cdot}\| \geq 1/\sqrt{n}$.
Then in the following lemma we further show that $\max_{j}\|\bm{\Phi}_{j\cdot}\| = O(1/\sqrt{n})$. Therefore, this enables us to replace all $\max_{j}\|\bm{\Phi}_{j\cdot}\|$ in the previous bounds with $1/\sqrt{n}$.

\begin{lemma}
Under the condition \eqref{eq:p_q_assumption}, for a sufficiently large $n$ we have
\begin{equation*}
    \max_{j}\|\bm{\Phi}_{j\cdot}\| \lesssim \frac{1}{\sqrt{n}},
\end{equation*}
with probability $1 - O(n^{-1})$.
\label{lemma:upper_bound_phi_j}
\end{lemma}

Based on the results above, now we are able to prove Lemma~\ref{the:Phi_f_Phi} which bounds the difference between $\bm{\Phi}_{i\cdot}$ and its surrogate $f(\bm{\Psi}\bm{O})_{i\cdot}$.
\begin{proof}[Proof of Lemma~\ref{the:Phi_f_Phi}]
We start from \eqref{eq:epsilon_1_decomp_main} in our proof sketch Section~\ref{sec:proof_sketch}, which is given as
\begin{equation}
    \|\bm{\Phi}_{i\cdot} - f(\bm{\Psi}\bm{O})_{i\cdot}\| \leq \|\bm{\Lambda}^{-1}\|\left(\|[\mathbb{E}[\bm{A}]]_{i \cdot}(\bm{\Phi} - \bm{\Psi}\bm{O})\|+ \|\bm{\Delta}_{i \cdot}(\bm{\Phi} - \bm{\Psi}\bm{O})\| \right) 
    \label{eq:epsilon_1_decomp}
\end{equation}
where the three terms $\|\bm{\Lambda}^{-1}\|$, $\|[\mathbb{E}[\bm{A}]]_{i \cdot}(\bm{\Phi} - \bm{\Psi}\bm{O})\|$ and $\|\bm{\Delta}_{i \cdot}(\bm{\Phi} - \bm{\Psi}\bm{O})\|$ are bounded separately. For $\|\bm{\Lambda}^{-1}\|$, by definition $\|\bm{\Lambda}^{-1}\| = \sigma_{2d}^{-1}(\bm{A})$, and by applying Weyl's inequality (Theorem~\ref{the:weyl}) we have 
\begin{equation*}
    \sigma_{2d}(\bm{A}) \geq \sigma_{2d}(\mathbb{E}[\bm{A}]) - \|\bm{\Delta}\| = \frac{pn}{2} - \|\bm{\Delta}\| = \Omega(pn)
\end{equation*}
w.h.p., where $\|\bm{\Delta}\|$ is bounded by Lemma~\ref{lemma:delta_norm_2}. This leads to 
\begin{equation}
    \|\bm{\Lambda}^{-1}\| \lesssim (pn)^{-1}
    \label{eq:bound_lambda}
\end{equation}
w.h.p. For $\|[\mathbb{E}[\bm{A}]]_{i \cdot}(\bm{\Phi} - \bm{\Psi}\bm{O})\|$, it satisfies
\begin{align*}
   \|[\mathbb{E}[\bm{A}]]_{i \cdot}(\bm{\Phi} - \bm{\Psi}\bm{O})\| &\leq \|[\mathbb{E}[\bm{A}]]_{i \cdot}\| \|\bm{\Phi} - \bm{\Psi}\bm{O}\| \lesssim p\sqrt{m} \cdot \frac{\sqrt{p(1-p)} + \sqrt{q}}{p\sqrt{n}} \\
   &\overset{(a)}{\lesssim} p\sqrt{m} \cdot \underbrace{\frac{\sqrt{p(1-p) + q}}{p\sqrt{n}}}_{= \eta/\log(nd)} \leq \frac{p\eta\sqrt{n}}{\log(nd)}
\end{align*}
w.h.p., where $\|\bm{\Phi} - \bm{\Psi}\bm{O}\|$ is bounded by Lemma~\ref{lemma:Psi_Phi} and $(a)$ uses the fact $\sqrt{x} + \sqrt{y} \leq \sqrt{2(x + y)}$ for any $x, y \geq 0$. From \eqref{eq:bound_T_1_T_2_T_3}, it remains to bound $\|\bm{\Delta}_{i \cdot}(\bm{\Phi} - \bm{\Psi}\bm{O})\|$ as
\begin{equation}
    \begin{aligned}
    \|\bm{\Delta}_{i \cdot}(\bm{\Phi} - \bm{\Psi}\bm{O})\| &\leq \underbrace{\|\bm{\Delta}_{i \cdot}(\bm{\Phi} \!-\! \bm{\Phi}^{(i)}\bm{O}^{(i)})\|}_{=: T_1} \!+\! \underbrace{\|\bm{\Delta}_{i \cdot}(\bm{\Phi}^{(i)} \!-\! \bm{\Psi}\bm{S}^{(i)})\|}_{=: T_2} \!+\! \underbrace{\|\bm{\Delta}_{i \cdot}\bm{\Psi}(\bm{S}^{(i)}\bm{O}^{(i)} \!-\! \bm{O})\|}_{=: T_3}  \\
    &= T_1 + T_2 + T_3, 
    \end{aligned}
    \label{eq:bound_T_1_T_2_T_3_appendix}
\end{equation}
w.h.p. Here, $T_1$, $T_2$ and $T_3$ satisfy
\begin{equation}
    \begin{aligned}
    T_1 &= \|\bm{\Delta}_{i \cdot}(\bm{\Phi} - \bm{\Phi}^{(i)}\bm{O}^{(i)})\| \leq \|\bm{\Delta}_{i \cdot}\| \|\bm{\Phi} - \bm{\Phi}^{(i)}\bm{O}^{(i)}\|
    \\
    &\lesssim (\sqrt{p(1-p)n} + \sqrt{qn}) \cdot \eta\max_j \|\bm{\Phi}_{j\cdot}\| \overset{(a)}{\lesssim} \sqrt{(p(1-p) + q)n}\eta \max_j \|\bm{\Phi}_{j\cdot}\|,\\
    T_2 &= \|\bm{\Delta}_{i \cdot}(\bm{\Phi}^{(i)} - \bm{\Psi}\bm{S}^{(i)})\| \lesssim \sqrt{(p(1-p)+q)n\log(nd)}\max_j\|\bm{\Phi}_{j\cdot}^{(i)} - \bm{\Psi}_{j\cdot}\bm{S}^{(i)}\| \\
    &\lesssim \sqrt{(p(1-p)+q)n\log(nd)}\max_j \|\bm{\Phi}_{j\cdot}\|, \\
    T_3 &= \|\bm{\Delta}_{i \cdot}\bm{\Psi}(\bm{S}^{(i)}\bm{O}^{(i)} - \bm{O})\| \leq \|\bm{\Delta}_{i \cdot}\bm{\Psi}\|\|\bm{S}^{(i)}\bm{O}^{(i)} - \bm{O}\| \\
    &\lesssim \sqrt{(p(1-p)+q)n\log(nd)}\max_j\|\bm{\Psi}_{j\cdot}\| \cdot  \eta \max_j \|\bm{\Phi}_{j\cdot}\| \\
    &\lesssim \sqrt{(p(1-p)+q)\log(nd)}\eta \max_j \|\bm{\Phi}_{j\cdot}\|.
    \end{aligned}
    \label{eq:bound_T_123}
\end{equation}
For $T_1$, $\|\bm{\Delta}_{i \cdot}\|$ and $\|\bm{\Phi} - \bm{\Phi}^{(i)}\bm{O}^{(i)}\|$ are bounded by Lemma~\ref{lemma:delta_norm_2} and Lemma~\ref{lemma:Phi_Phi_i} respectively; for $T_2$, $\|\bm{\Delta}_{i \cdot}(\bm{\Phi}^{(i)} - \bm{\Psi}\bm{S}^{(i)})\|$ is bounded by Lemma~\ref{lemma:Delta_M} with $\bm{M} = \bm{\Phi}^{(i)} - \bm{\Psi}\bm{S}^{(i)}$, and $\max_j\|\bm{\Phi}_{j\cdot}^{(i)} - \bm{\Psi}_{j\cdot}\bm{S}^{(i)}\|$ is bounded by Lemma~\ref{lemma:phi_i_psi_j}; for $T_3$, $\|\bm{\Delta}_{i \cdot}\bm{\Psi}\|$ and $\|\bm{S}^{(i)}\bm{O}^{(i)} - \bm{O}\|$ are bounded by Lemma~\ref{lemma:Delta_M} and Lemma~\ref{lemma:S_i_O_i_O} respectively.
As a result, one can see that $T_2$ is the dominant term 
among $\{T_i\}_{i = 1}^3$ and \eqref{eq:bound_T_1_T_2_T_3_appendix} becomes  
\begin{equation*}
    \|\bm{\Delta}_{i \cdot}(\bm{\Phi} - \bm{\Psi}\bm{O})\| \lesssim \sqrt{(p(1-p)+q)n\log(nd)}\max_j \|\bm{\Phi}_{j\cdot}\| = \eta pn\max_j \|\bm{\Phi}_{j\cdot}\|
\end{equation*}
w.h.p. Putting the above together into \eqref{eq:epsilon_1_decomp} yields 
\begin{equation*}
    \begin{aligned}
        \|\bm{\Phi}_{i\cdot} - f(\bm{\Psi}\bm{O})_{i\cdot}\| &\lesssim \|\bm{\Lambda}^{-1}\|\left(\|[\mathbb{E}[\bm{A}]]_{i \cdot}(\bm{\Phi} - \bm{\Psi}\bm{O})\|+ \|\bm{\Delta}_{i \cdot}(\bm{\Phi} - \bm{\Psi}\bm{O})\| \right)\\
        &\lesssim (pn)^{-1} \cdot \left(\frac{p\eta\sqrt{n}}{\log(nd)} + \eta pn \max_j \|\bm{\Phi}_{j\cdot}\|\right) \overset{(a)}{\lesssim} \eta \max_j \|\bm{\Phi}_{j\cdot}\|
    \end{aligned}
\end{equation*}
w.h.p., where $(a)$ uses $\max_j \|\bm{\Phi}_{j\cdot}\| \geq 1/\sqrt{n}$ shown in the proof of Lemma~\ref{lemma:phi_i_psi_j}. 
\end{proof}

In the next, we prove Lemma~\ref{the:phi_i_phi_j} which bounds the difference between $\bm{\Phi}_{i\cdot}$ and $\bm{\Phi}_{j\cdot}$ for $\kappa(i) = \kappa(j)$.
\begin{proof}[Proof of Lemma~\ref{the:phi_i_phi_j}]
We start from \eqref{eq:bound_phi_i_phi_j_main} such that 
\begin{equation*}
\begin{aligned}
        \|\bm{\Phi}_{i\cdot} - \bm{\Phi}_{j\cdot}\| &\leq \|\bm{\Phi}_{i\cdot} - f(\bm{\Psi}\bm{O})_{i\cdot}\| + \|\bm{\Phi}_{j\cdot} - f(\bm{\Psi}\bm{O})_{j\cdot}\| + \|f(\bm{\Psi}\bm{O})_{i\cdot} - f(\bm{\Psi}\bm{O})_{j\cdot}\|.
\end{aligned}
\end{equation*}
where $\|\bm{\Phi}_{i\cdot} - f(\bm{\Psi}\bm{O})_{i\cdot}\|$ and $\|\bm{\Phi}_{j\cdot} - f(\bm{\Psi}\bm{O})_{j\cdot}\|$ have been bounded by Lemma~\ref{the:Phi_f_Phi}. For $\|f(\bm{\Psi}\bm{O})_{i\cdot} - f(\bm{\Psi}\bm{O})_{j\cdot}\|$, by definition 
\begin{equation}
\begin{aligned}
    &\|f(\bm{\Psi}\bm{O})_{i\cdot} - f(\bm{\Psi}\bm{O})_{j\cdot}\| = \|(\bm{\Delta}_{i \cdot} - \bm{\Delta}_{j\cdot})\bm{\Psi}\bm{O}\bm{\Lambda}^{-1}\| \leq \|(\bm{\Delta}_{i \cdot} - \bm{\Delta}_{j\cdot})\bm{\Psi}\|\|\bm{\Lambda}^{-1}\|,\\
    &\leq (\|\bm{\Delta}_{i \cdot}\bm{\Psi}\| + \|\bm{\Delta}_{j\cdot}\bm{\Psi}\|)\|\bm{\Lambda}^{-1}\| \lesssim \sqrt{(p(1-p)+q)\log(nd)} \cdot (pn)^{-1}\\
    &= \frac{\eta}{\sqrt{n}} \overset{(a)}{\leq} \eta \max_j \|\bm{\Phi}_{j\cdot}\|
    \label{eq:bound_f_psi_i_j}
\end{aligned}
\end{equation}
w.h.p., where $\|\bm{\Delta}_{i \cdot}\bm{\Psi}\|$ and $\|\bm{\Delta}_{j\cdot}\bm{\Psi}\|$ are bounded by Lemma~\ref{lemma:Delta_M}, $\|\bm{\Lambda}^{-1}\|$ is bounded in \eqref{eq:bound_lambda},
$(a)$ uses $\max_j \|\bm{\Phi}_{j\cdot}\| \geq 1/\sqrt{n}$ in Lemma~\ref{lemma:phi_i_psi_j}. Combining \eqref{eq:bound_f_psi_i_j} with Lemma~\ref{the:Phi_f_Phi} yields
\begin{equation}
    \|\bm{\Phi}_{i\cdot} - \bm{\Phi}_{j\cdot}\| \lesssim \eta \max_j\|\bm{\Phi}_{j\cdot}\|
    \label{eq:bound_phi_i_phi_j}
\end{equation}
w.h.p. This further leads to $\max_j\|\bm{\Phi}_{j\cdot}\| = O(1/\sqrt{n})$ given in Lemma~\ref{lemma:upper_bound_phi_j}. Plugging this back into \eqref{eq:bound_phi_i_phi_j} completes the proof.
\end{proof}

Now we are ready to prove Theorem~\ref{the:phi_psi_i}.
\begin{proof}[Proof of Theorem~\ref{the:phi_psi_i}]
By defining $\bm{N}_{C_1}$, $\bm{N}_i$ and $\bm{N}_{\bm{\Delta}, i}$ as in \eqref{eq:def_N_C_1} such that
\begin{equation*}
    \bm{N}_{C_1} :=
    \begin{bmatrix}
    \bm{\Phi}_{1\cdot} - \bm{\Psi}_{1\cdot}\bm{O}\\
    \vdots\\
    \bm{\Phi}_{m\cdot} - \bm{\Psi}_{m\cdot}\bm{O}\\
    \end{bmatrix}
    = 
    \underbrace{\begin{bmatrix}
    \bm{\Phi}_{i\cdot} - \bm{\Psi}_{1\cdot}\bm{O}\\
    \vdots\\
    \bm{\Phi}_{i\cdot}- \bm{\Psi}_{m\cdot}\bm{O}\\
    \end{bmatrix}}_{=:\bm{N}_i} + 
    \underbrace{
    \begin{bmatrix}
    \bm{\Phi}_{1\cdot} - \bm{\Phi}_{i\cdot}\\
    \vdots\\
    \bm{\Phi}_{m\cdot} - \bm{\Phi}_{i\cdot}\\
    \end{bmatrix}}_{=: \bm{N}_{\bm{\Delta}, i}}
    = \bm{N}_i + \bm{N}_{\bm{\Delta}, i}, 
\end{equation*}
we have $\|\bm{N}_i\| = \sqrt{m}\|\bm{\Phi}_{i\cdot} - \bm{\Psi}_{i\cdot}\bm{O}\|$ since $\bm{\Psi}_{1\cdot} = \cdots = \bm{\Psi}_{m\cdot}$, and 
\begin{equation*}
    \|\bm{N}_{\bm{\Delta}, i}\| \overset{(a)}{\leq} \sqrt{\sum_{j = 1}^m \|\bm{\Phi}_{i\cdot} - \bm{\Phi}_{j\cdot}\|^2} \overset{(b)}{\lesssim} \sqrt{m} \cdot \frac{\eta}{\sqrt{n}} = O(\eta)
\end{equation*}
w.h.p., where $(a)$ holds by definition of the operator norm, $(b)$ comes from Lemma~\ref{the:phi_i_phi_j}. Then the following satisfies
\begin{equation}
    \begin{aligned}
    \|\bm{\Phi} - \bm{\Psi}\bm{O}\| &\overset{(a)}{=} \max_{\|\bm{x}\| = 1}\|[\bm{N}_{C_1}^\top, \bm{N}_{C_2}^\top] \bm{x}\|  \geq \max_{\|\bm{y}\| = 1} \|\bm{N}_{C_1}^\top\bm{y}\| \overset{(b)}{=} \|\bm{N}_{C_1}\| \\ 
    & \overset{(c)}{\geq} \|\bm{N}_i\| - \|\bm{N}_{\bm{\Delta}, i}\| = \sqrt{m}\|\bm{\Phi}_{i\cdot} - \bm{\Psi}_{i\cdot}\bm{O}\| - O(\eta)
    \end{aligned}
    \label{eq:phi_psi_1}
\end{equation}
w.h.p., where both $(a)$ and $(b)$ hold by definition of the operator norm, $(c)$ comes from the triangle inequality. On the other hand, from Lemma~\ref{lemma:Psi_Phi} one can see that $\|\bm{\Phi} - \bm{\Psi}\bm{O}\| \lesssim \tau \lesssim \eta/\log(nd)$ w.h.p.
Combining this with \eqref{eq:phi_psi_1} yields 
\begin{equation*}
    \begin{aligned}
        \|\bm{\Phi}_{i\cdot} - \bm{\Psi}_{i\cdot}\bm{O}\| = \frac{1}{\sqrt{m}}\left(O(\eta) + O\left(\frac{\eta}{\log(nd)}\right)\right) \lesssim \frac{\eta}{\sqrt{n}}
    \end{aligned}
\end{equation*}
w.h.p. The bound $\|\bm{\Phi}_{i\cdot} - \bm{\Psi}_{i\cdot}\bm{O}\|$ for $i \in C_2$ is obtained in the same way as above.
\end{proof}

\subsection{Proof of Theorem~\ref{the:cond}}  
\label{sec:proof_the_2}
In this section, we prove Theorem~\ref{the:cond} which provides the performance guarantee of Algorithm~\ref{alg:spectral}. Again, we start from listing several important inequalities, where most of them directly come from Theorem~\ref{the:phi_psi_i}. All the proofs of the lemmas are deferred to Appendix~\ref{sec:proof_lemmas}.

To begin with, recall that Theorem~\ref{the:phi_psi_i} bounds the difference between the noisy eigenvector block $\bm{\Phi}_{i\cdot}$ and the clean one $\bm{\Psi}_{i\cdot} \bm{O}$, then the following lemma further bounds the difference between other statistics such as the polar factors $\mathcal{P}(\bm{\Phi}_{i\cdot})$ and $\mathcal{P}(\bm{\Psi}_{i\cdot}\bm{O})$.

\begin{lemma}
Under the condition \eqref{eq:p_q_assumption} and as $n$ is sufficiently large, for any $i, j = 1,\ldots, n$ such that $\kappa(i) = \kappa(j)$, the following satisfy
\begin{align}
    \|\sigma_l(\bm{\Phi}_{i\cdot}) - \sqrt{2/n}\|  &\lesssim \eta/\sqrt{n},\quad l = 1,\ldots, d, \label{eq:sigma_phi_new} \\
    \|\mathcal{P}(\bm{\Phi}_{i\cdot}) - \mathcal{P}(\bm{\Psi}_{i\cdot}\bm{O})\| &\lesssim \eta, \label{eq:P_Phi_P_Psi} \\
    \|\bm{\Phi}_{j\cdot} \mathcal{P}(\bm{\Phi}_{i\cdot})^\top - \bm{\Psi}_{j\cdot} \mathcal{P}(\bm{\Psi}_{i\cdot})^\top\| &\lesssim \eta/\sqrt{n}, \nonumber\\
    \|\mathcal{P}(\bm{\Phi}_{j\cdot} \mathcal{P}(\bm{\Phi}_{i\cdot})^\top) - \mathcal{P}(\bm{\Psi}_{j\cdot} \mathcal{P}(\bm{\Psi}_{i\cdot})^\top)\| &\lesssim \eta \label{eq:P_P_Phi_P_Psi}
\end{align}
with probability $1 - O(n^{-1})$, where $\bm{O} = \mathcal{P}(\bm{\Psi}^\top \bm{\Phi})$. 
\label{lemma:sigma_phi_i} 
\end{lemma}

Lemma~\ref{lemma:sigma_phi_i} is sufficient for showing the exact recovery of the cluster memberships by Algorithm~\ref{alg:spectral}. For showing the stable recovery of the orthogonal group elements $\{\bm{O}_{i}\}_{i=1}^n$, we need the following result.

\begin{lemma}
Under the condition \eqref{eq:p_q_assumption} and the assumption that $p_1 \in C_1$ is the first pivot selected by Algorithm~\ref{alg:spectral}, as $n$ is sufficiently large, for $\bm{Q}_{\cdot 1}$ and $\bm{Q}_{\cdot 2}$ defined in \eqref{eq:def_Q_12}, the following satisfies for any $j \in C_2$
\begin{align}
    \|\bm{Q}_{\cdot 2} - \mathcal{P}(\bm{\Psi}_{j\cdot}\bm{O})^\top \bar{\bm{O}}_2\| &\lesssim \eta, \nonumber\\
    \|\bm{\Phi}_{j\cdot}\bm{Q}_{\cdot 2} - \bm{\Psi}_{j\cdot}\mathcal{P}(\bm{\Psi}_{j\cdot})^\top \bar{\bm{O}}_2\| &\lesssim \eta/\sqrt{n}, \nonumber\\
    \|\mathcal{P}(\bm{\Phi}_{j\cdot}\bm{Q}_{\cdot 2}) - \mathcal{P}(\bm{\Psi}_{j\cdot}\mathcal{P}(\bm{\Psi}_{j\cdot})^\top) \bar{\bm{O}}_2\| &\lesssim \eta \label{eq:P_Q_2_Psi}
\end{align}
with probability $1 - O(n^{-1})$, where $\bm{O} = \mathcal{P}(\bm{\Psi}^\top \bm{\Phi})$ and $\bar{\bm{O}}_2 = \mathcal{P}(\mathcal{P}(\bm{\Phi}_{j\cdot}\bm{O})\bm{Q}_{\cdot 2})$.
\label{lemma:Q_2} 
\end{lemma}

\begin{proof}[Proof of Theorem~\ref{the:cond}]
\textit{Exact recovery of the cluster memberships}:
According to the blockwise CPQR in Section~\ref{sec:block_CPQR}, with out loss of generality we assume $p_1 \in C_1$ is the first pivot column selected by Algorithm~\ref{alg:spectral}. Then for any node $j$, the two blocks $\bm{R}_{1j}$ and $\bm{R}_{2j}$ satisfy
\begin{equation*}
    \|\bm{R}_{1j}\|_\mathrm{F}^2 = \|\mathcal{P}(\bm{\Phi}_{p_1\cdot})(\bm{\Phi}_{j\cdot})^\top\|^2_\mathrm{F}, \quad
    \|\bm{R}_{2j}\|_\mathrm{F}^2 = \|(\bm{\Phi}_{j\cdot})^\top\|_\mathrm{F}^2 - \|\mathcal{P}(\bm{\Phi}_{p_1\cdot})(\bm{\Phi}_{j\cdot})^\top\|^2_\mathrm{F},
\end{equation*}
which correspond to the projection of $(\bm{\Phi}_{j\cdot})^\top$ onto the column space $\mathcal{R}((\bm{\Phi}_{p_1\cdot})^\top)$ and the complement of $\mathcal{R}((\bm{\Phi}_{p_1\cdot})^\top)$ respectively. According to Algorithm~\ref{alg:spectral}, node $j$ would be assigned to $C_1$ if $\|\bm{R}_{1j}\|_\mathrm{F} > \|\bm{R}_{2j}\|_\mathrm{F}$, which is equivalent to
\begin{equation}
    \|\mathcal{P}(\bm{\Phi}_{p_1\cdot})(\bm{\Phi}_{j\cdot})^\top\|_\mathrm{F} > \frac{\sqrt{2}}{2}\|\bm{\Phi}_{j\cdot}\|_\mathrm{F}
    \label{eq:assign_cluster_equal}.
\end{equation}
We first show that for any $j \in C_1$, \eqref{eq:assign_cluster_equal} is satisfied w.h.p. To this end, by using $\bm{\Psi}_{i\cdot}\bm{O}$ as a surrogate of $\bm{\Phi}_{i\cdot}$ for each node $i$, $\|\mathcal{P}(\bm{\Phi}_{p_1\cdot})(\bm{\Phi}_{j\cdot})^\top\|_{\mathrm{F}}$ can be bounded as
\begin{align}
    &\|\mathcal{P}(\bm{\Phi}_{p_1\cdot})(\bm{\Phi}_{j\cdot})^\top\|_{\mathrm{F}} = \|\left[\mathcal{P}(\bm{\Psi}_{p_1\cdot}\bm{O}) + \mathcal{P}(\bm{\Phi}_{p_1\cdot}) - \mathcal{P}(\bm{\Psi}_{p_1\cdot}\bm{O})\right](\bm{\Phi}_{j\cdot})^\top\|_{\mathrm{F}} \nonumber\\
    &\geq \|\mathcal{P}(\bm{\Psi}_{p_1\cdot}\bm{O})(\bm{\Psi}_{j\cdot}\bm{O}  + \bm{\Phi}_{j\cdot} - \bm{\Psi}_{j\cdot}\bm{O})^\top\|_{\mathrm{F}} - \|\left[\mathcal{P}(\bm{\Phi}_{p_1\cdot}) - \mathcal{P}(\bm{\Psi}_{p_1\cdot}\bm{O})\right](\bm{\Phi}_{j\cdot})^\top\|_\mathrm{F}\nonumber\\
    &\geq \underbrace{\|\mathcal{P}(\bm{\Psi}_{p_1\cdot})(\bm{\Psi}_{j\cdot})^\top\|_\mathrm{F}}_{=: T_1} - \underbrace{\|\mathcal{P}(\bm{\Psi}_{p_1\cdot})\|_\mathrm{F}\|\bm{\Phi}_{j\cdot} - \bm{\Psi}_{j\cdot}\bm{O}\|}_{=: T_2} -  \underbrace{\|\mathcal{P}(\bm{\Phi}_{p_1\cdot}) - \mathcal{P}(\bm{\Psi}_{p_1\cdot}\bm{O})\|\|\bm{\Phi}_{j\cdot}\|_\mathrm{F}}_{=: T_3} \nonumber\\
    &= T_1 - T_2 - T_3 \label{eq:bound_P_phi_i_phi_j}
\end{align}
where we use the fact $\|\bm{X}\bm{Y}\|_{\mathrm{F}} \leq \|\bm{X}\|\|\bm{Y}\|_\mathrm{F}$ for any $\bm{X}$ and $\bm{Y}$. To proceed, $T_1$, $T_2$ and $T_3$ can be bounded separately as 
\begin{align*}
    T_1 &= \|\mathcal{P}(\bm{\Psi}_{p_1\cdot})(\bm{\Psi}_{j\cdot})^\top\|_\mathrm{F} = \sqrt{m}\|\bm{\Psi}_{p_1\cdot}(\bm{\Psi}_{j\cdot})^\top\|_\mathrm{F} = \sqrt{2d/n},\\
    T_2 &= \|\mathcal{P}(\bm{\Psi}_{p_1\cdot})\|_\mathrm{F}\|\bm{\Phi}_{j\cdot} - \bm{\Psi}_{j\cdot}\bm{O}\| \leq \sqrt{d} \cdot O\left(\eta/\sqrt{n}\right) \lesssim \eta\sqrt{d/n}, \\
    T_3 &\leq \sqrt{d}\|\bm{\Phi}_{j\cdot}\|\|\mathcal{P}(\bm{\Phi}_{p_1\cdot}) - \mathcal{P}(\bm{\Psi}_{p_1\cdot}\bm{O})\| \lesssim \sqrt{d} \cdot \left(\sqrt{2/n} + O\left(\eta/\sqrt{n}\right)\right) \cdot \eta \lesssim \eta\sqrt{d/n}.
\end{align*}
For $T_1$ we use $\bm{\Psi}_{p_1\cdot}(\bm{\Psi}_{j\cdot})^\top = \bm{I}_d/m$; for $T_2$ we bound $\|\bm{\Phi}_{j\cdot} - \bm{\Psi}_{j\cdot}\bm{O}\|$ by Theorem~\ref{the:phi_psi_i}; for $T_3$, both $\|\bm{\Phi}_{j\cdot}\|$ and $\|\mathcal{P}(\bm{\Phi}_{p_1\cdot}) - \mathcal{P}(\bm{\Psi}_{p_1\cdot}\bm{O})\|$ are bounded by Lemma~\ref{lemma:sigma_phi_i}.
Plugging these into \eqref{eq:bound_P_phi_i_phi_j} yields
\begin{equation}
    \|\mathcal{P}(\bm{\Phi}_{p_1\cdot})(\bm{\Phi}_{j\cdot})^\top\|_{\mathrm{F}} = \sqrt{\frac{2d}{n}} - O\bigg(\sqrt{\frac{d}{n}}\eta\bigg),
\end{equation}
w.h.p.
On the other hand, the RHS of \eqref{eq:assign_cluster_equal} satisfies
\begin{equation*}
    \frac{\sqrt{2}}{2}\|\bm{\Phi}_{j\cdot}\|_\mathrm{F} \leq \frac{\sqrt{2d}}{2}\|\bm{\Phi}_{j\cdot}\| \leq \frac{\sqrt{2d}}{2}(\|\bm{\Psi}_{j\cdot}\bm{O}\| +\|\bm{\Phi}_{j\cdot} -\bm{\Psi}_{j\cdot}\bm{O}\|) =
    \left(1 + O(\eta)\right)\sqrt{\frac{d}{n}},
\end{equation*}
w.h.p. Therefore, as $\eta$ is sufficiently small, \eqref{eq:assign_cluster_equal} is satisfied and $j \in C_1$ is correctly assigned to $C_1$. On the other hand, for any $j \in C_2$, similar to \eqref{eq:bound_P_phi_i_phi_j} we have 
\begin{equation*}
    \begin{aligned}
        &\|\mathcal{P}(\bm{\Phi}_{p_1\cdot})(\bm{\Phi}_{j\cdot})^\top\|_\mathrm{F} = \|\left[\mathcal{P}(\bm{\Psi}_{p_1\cdot}\bm{O}) + \mathcal{P}(\bm{\Phi}_{p_1\cdot}) - \mathcal{P}(\bm{\Psi}_{p_1\cdot}\bm{O})\right](\bm{\Phi}_{j\cdot})^\top\|_{\mathrm{F}} \\
        &\leq \|\mathcal{P}(\bm{\Psi}_{p_1\cdot}\bm{O})(\bm{\Psi}_{j\cdot}\bm{O} + \bm{\Phi}_{j\cdot} - \bm{\Psi}_{j\cdot}\bm{O})^\top\|_{\mathrm{F}} + \|\left[\mathcal{P}(\bm{\Phi}_{p_1\cdot}) - \mathcal{P}(\bm{\Psi}_{p_1\cdot}\bm{O})\right](\bm{\Phi}_{j\cdot})^\top\|_\mathrm{F}\\
        &\leq \underbrace{\|\mathcal{P}(\bm{\Psi}_{p_1\cdot})(\bm{\Psi}_{j\cdot})^\top\|_\mathrm{F}}_{ = 0} + \underbrace{\|\mathcal{P}(\bm{\Psi}_{p_1\cdot})\|_\mathrm{F}\|\bm{\Phi}_{j\cdot} \!-\! \bm{\Psi}_{j\cdot}\bm{O}\|_{\mathrm{2}}}_{= T_2} + \underbrace{\|\mathcal{P}(\bm{\Phi}_{p_1\cdot}) \!-\! \mathcal{P}(\bm{\Psi}_{p_1\cdot}\bm{O})\|_\mathrm{2}\|\bm{\Phi}_{j\cdot}\|_\mathrm{F}}_{ = T_3}\\
        &=  O\left(\eta\sqrt{d/n}\right)
    \end{aligned}
\end{equation*}
w.h.p., where $T_2$ and $T_3$ are defined in \eqref{eq:bound_P_phi_i_phi_j}.  On the other hand, 
\begin{equation*}
    \frac{\sqrt{2}}{2}\|\bm{\Phi}_{j\cdot}\|_\mathrm{F} = \frac{\sqrt{2}}{2}\sqrt{\sum_{l = 1}^d\sigma_{l}^2(\bm{\Phi}_{j\cdot})} \geq \frac{\sqrt{2}}{2} \cdot \sqrt{d}\bigg(\sqrt{\frac{2}{n}} - O\bigg(\frac{\eta}{\sqrt{n}}\bigg)\bigg) = \sqrt{\frac{d}{n}} -  O\bigg(\sqrt{\frac{d}{n}}\eta\bigg)
\end{equation*}
w.h.p., where $\sigma_l(\bm{\Phi}_{j\cdot})$ is bounded by Lemma~\ref{lemma:sigma_phi_i}. Therefore, as $\eta$ is small, the inequality in \eqref{eq:assign_cluster_equal} does not hold and node $j$ is assigned to $C_2$. This leads to the exact recovery of the cluster memberships.

\vspace{0.1cm}
\textit{$\bullet$ Stable recovery of orthogonal transformations:}
To bound the estimation error of $\{\bm{O}_i\}_{i = 1}^n$, recall the orthogonal matrix $\bm{Q}$ from the blockwise CPQR in \eqref{eq:alg_QR}, then in the case of two clusters, it can be written as $\bm{Q} = [\bm{Q}_{\cdot 1} ,\bm{Q}_{\cdot 2}]$ where $\bm{Q}_{\cdot 1}, \bm{Q}_{\cdot 2} \in \mathbb{R}^{2d \times d}$. We follow the previous assumption that $p_1 \in C_1$ is the first pivot, then $\bm{Q}_{\cdot 1} \in \mathbb{R}^{2d \times d}$ is the polar factor of $(\bm{\Phi}_{p_1 \cdot})^\top$ up to some orthogonal transformation, and $\bm{Q}_{\cdot 2} \in \mathbb{R}^{2d \times d}$ is orthogonal to $\bm{Q}_{\cdot 1}$. 
As a result, for any node $j \in C_1$ we have 
\begin{equation}
    \bm{R}_{1j} = (\bm{Q}_{\cdot 1})^\top (\bm{\Phi}_{j\cdot})^\top, \quad \bm{Q}_{\cdot 1} = \mathcal{P}(\bm{\Phi}_{p_1\cdot}^\top) \bar{\bm{O}}_1
    \label{eq:R_1_j_1_j_2}
\end{equation}
where $\bar{\bm{O}}_1 \in \mathbb{R}^{d \times d}$ is some orthogonal matrix. Then, according to \eqref{eq:cluster}, our estimation $\hat{\bm{O}}_{j}$ is given as $\hat{\bm{O}}_{j} = \mathcal{P}(\bm{R}_{1j})^\top =   \mathcal{P}(\bm{\Phi}_{j\cdot}\mathcal{P}(\bm{\Phi}_{p_1\cdot})^\top)\bar{\bm{O}}_1.$
Meanwhile, the ground truth $\bm{O}_{j}$ satisfies $\bm{O}_{j} = \mathcal{P}(\bm{\Psi}_{j\cdot}\mathcal{P}(\bm{\Psi}_{p_1\cdot})^\top) = \bm{I}_d$ by assumption, then the estimation error of $\bm{O}_j$ can be bounded as
\begin{equation*}
    \|\hat{\bm{O}}_j - \bm{O}_j \bar{\bm{O}} _1\| = \|\hat{\bm{O}}_j\bar{\bm{O}}_1^\top - \bm{O}_j\| = \|\mathcal{P}(\bm{\Phi}_{j\cdot}\mathcal{P}(\bm{\Phi}_{i\cdot})^\top) - \mathcal{P}(\bm{\Psi}_{j\cdot}\mathcal{P}(\bm{\Psi}_{i\cdot})^\top)\| \overset{(a)}{\lesssim} \eta
\end{equation*}
w.h.p., where $(a)$ comes from Lemma~\ref{lemma:sigma_phi_i}. This completes the proof for $j \in C_1$. Next, we check $\hat{\bm{O}}_j$ for $j \in C_2$, similar to \eqref{eq:R_1_j_1_j_2} we have 
\begin{equation}
    \bm{R}_{2j} = (\bm{Q}_{\cdot 2})^\top (\bm{\Phi}_{j\cdot})^\top, \quad  (\bm{Q}_{\cdot 1})^\top \bm{Q}_{\cdot 2} = \bm{0}.
    \label{eq:R_Q_2}
\end{equation}
Also the ground truth $\bm{O}_j$ satisfies $\bm{O}_{j} = \mathcal{P}(\bm{\Psi}_{j\cdot}\mathcal{P}(\bm{\Psi}_{i\cdot})^\top) = \bm{I}_d$. Then our estimation $\hat{\bm{O}}_j$ is given as $\hat{\bm{O}}_{j} =  \mathcal{P}(\bm{R}_{2j})^\top = \mathcal{P}(\bm{\Phi}_{j\cdot}\bm{Q}_{\cdot 2})$. Furthermore, by defining $\bar{\bm{O}}_2 = \mathcal{P}(\mathcal{P}(\bm{\Phi}_{j\cdot}\bm{O})\bm{Q}_{\cdot 2})$ where $\bm{O} = \mathcal{P}(\bm{\Psi}^\top\bm{\Phi})$, we have
\begin{equation*}
    \|\hat{\bm{O}}_j - \bm{O}_j \bar{\bm{O}}_2\| = \|\mathcal{P}(\bm{\Phi}_{j\cdot}\bm{Q}_{\cdot 2}) - \mathcal{P}(\bm{\Psi}_{j\cdot}\mathcal{P}(\bm{\Psi}_{j\cdot})^\top) \bar{\bm{O}}_2\| \overset{(a)}{\lesssim} \eta
\end{equation*} 
w.h.p., where $(a)$ comes from  \eqref{eq:P_Q_2_Psi} in Lemma~\ref{lemma:Q_2}. This completes the proof.
\end{proof}

\section{Proof of the Lemmas in Appendix~\ref{sec:proof_main_theorems} }
\label{sec:proof_lemmas}

\subsection{Proof of Lemma~\ref{lemma:delta_norm_2}}
\label{sec:proof_lemma_delta_norm_2}
The proof relies on the following two theorems.
\begin{theorem}
Let $\bm{A} \in \mathbb{R}^{n \times n}$ be a symmetric matrix whose entries $a_{ij}$ for $1 \leq i < j \leq n$ are i.i.d. and satisfy 
\begin{equation*}
    a_{ij} = 
    \begin{cases}
    1 - p, &\quad \textrm{w.p. $p$},\\
    -p, &\quad \textrm{w.p. $1-p$}.
    \end{cases}
\end{equation*}
Then, for any $c > 0$, there exists constants $c_1, c_2 > 0$ such that 
\begin{equation*}
    \|\bm{A}\| \leq c_1\sqrt{p(1-p)n} + c_2\sqrt{\log n}
\end{equation*}
with probability $1 - n^{-c}$.
\label{the:bound_norm_2_A}
\end{theorem}

\begin{proof}
We resort to the moment method which is commonly used in random matrix theory (see e.g.~\cite{anderson2010introduction, tao2012topics}). The idea is to bound $\mathbb{E}[\|\bm{A}\|^{2k}]$ for any $k \in \mathbb{N}$ and apply Markov inequality for getting a tail bound. We start by denoting
\begin{align*}
\sigma_{k} & = \Bigg(\sum_{i = 1}^{n} \Bigg(\sum_{j = 1}^{n} \mathbb{E}\left[a_{ij}^2\right]\Bigg)^k \Bigg)^{1/2k} = n^{1/2k} \sqrt{n p(1-p)}, \\
\sigma_k^* & = \Bigg(\sum_{i = 1}^{n} \sum_{j = 1}^{n} \|a_{ij} \|_\infty^{2k} \Bigg)^{1/2k} \leq n^{1/k}.
\end{align*}
Then by applying \cite[Theorem 4.8]{latala2018dimension}, we obtain
\begin{align*}
    \mathbb{E}\left[\|\bm{A}\|^{2k}\right]^{1/2k} & \leq 2\sigma_{k} + C \sqrt{k} \sigma_{k}^* \leq 2n^{1/2k} \sqrt{n p(1-p)} + C\sqrt{k}n^{1/k} 
\end{align*}
for some universal constant $C > 0$. By further setting $k = \lceil \gamma \log n \rceil$ for some constant $\gamma > 0$, we have
\begin{equation*}
     \mathbb{E}\left[\|\bm{A}\|^{2k}\right]^{1/2k} \leq 2e^{1/2\gamma}\sqrt{np(1-p)} + C^\prime\sqrt{\log n}
\end{equation*}
for some $C^\prime > 0$. To get a tail bound, by Markov inequality~\cite{boucheron2013concentration} we obtain,
\begin{align*}
     \mathbb{P}\{ \| \bm{A} \| \geq t \} & =  \mathbb{P}\{ \| \bm{A} \|^{2k} \geq t^{2k} \} \leq t^{-2k}     \mathbb{E}\left[\|\bm{A}\|^{2k}\right]  \\
    &\leq \left( \frac {2e^{1/2\gamma}\sqrt{np(1-p)} + C^\prime\sqrt{\log n}}{t} \right)^{2\lceil \gamma \log n \rceil}, 
\end{align*}
for any $t > 0$. By setting $\gamma = \frac{c}{2} + 1$, for any constant $c > 0$, we can identify some constants $c_1, c_2 > 0$ such that $t = c_1\sqrt{np(1-p)}+ c_2 \sqrt{\log n}$ and $\mathbb{P}\{\|\bm{A}\| \geq t \} \leq n^{-c}$, which completes the proof.
\end{proof}

\begin{theorem}
Let $\bm{S} \in \mathbb{R}^{n_1 d \times n_2 d}$ be an $n_1 \times n_2$ random block matrix where each block $\bm{S}_{ij}$ is i.i.d. and satisfies
\begin{equation*}
    \bm{S}_{ij} = 
    \begin{cases}
    \bm{O}_{ij}, &\quad \text{w.p. $q$},\\
    \bm{0}, &\quad \text{w.p. $1-q$},
    \end{cases}
    \label{eq:S_ij_definition}
\end{equation*}
where $\bm{O}_{ij}$ is uniformly drawn from $\mO(d)$. Let $n = n_1 + n_2$. Then, for any $c > 0$, there exists $c_1, c_2 > 0$ such that
\begin{equation*}
    \|\bm{S}\| \leq c_1 (\sqrt{q n_1} + \sqrt{q n_2}) + c_2 \sqrt{\log n}
\end{equation*}
with probability $1 - n^{-c}$.
\label{lemma:spec_bound_S_12}
\end{theorem}
\begin{proof}
The proof is similar to the one of~\cite[Theorem~A.7]{fan2021joint}, with the only difference on the orthogonal group $\mO(d)$ rather than $\SO(d)$ considered in~\cite{fan2021joint}. 
\end{proof}

\begin{proof}[Proof of Lemma~\ref{lemma:delta_norm_2}]
By definition, $\bm{\Delta}$ can be written into four blocks as 
\begin{equation*}
    \bm{\Delta} = 
    \begin{bmatrix}
    \widetilde{\bm{\Delta}}_{11} & \widetilde{\bm{\Delta}}_{12}\\
    \widetilde{\bm{\Delta}}_{12}^\top & \widetilde{\bm{\Delta}}_{22}
    \end{bmatrix} = 
    \begin{bmatrix}
    \widetilde{\bm{\Delta}}_{11} & \bm{0}\\
    \bm{0} & \widetilde{\bm{\Delta}}_{22}
    \end{bmatrix} + 
    \begin{bmatrix}
    \bm{0} & \widetilde{\bm{\Delta}}_{12}\\
    \widetilde{\bm{\Delta}}_{12}^\top & \bm{0}
    \end{bmatrix} =: \bm{\Delta}_{\text{in}} + \bm{\Delta}_{\text{out}},
    \label{eq:Delta_four_blocks}
\end{equation*}
where $\widetilde{\bm{\Delta}}_{11}, \widetilde{\bm{\Delta}}_{22} \in \mathbb{R}^{md \times md}$ correspond to the two clusters $C_1, C_2$ respectively.
By using the fact $\|\bm{\Delta}_{\text{out}}\| =  \|\widetilde{\bm{\Delta}}_{12}\|$ and the triangle inequality that $\|\bm{\Delta}_{\text{in}}\| \leq \|\widetilde{\bm{\Delta}}_{11}\| + \|\widetilde{\bm{\Delta}}_{22}\|$ and $\|\bm{\Delta}\| \leq \|\bm{\Delta}_{\text{in}}\| + \|\bm{\Delta}_{\text{out}}\|$, we have $\|\bm{\Delta}\| \leq \|\widetilde{\bm{\Delta}}_{11}\| + \|\widetilde{\bm{\Delta}}_{22}\| + \|\widetilde{\bm{\Delta}}_{12}\|$. For $\|\widetilde{\bm{\Delta}}_{11}\|$, by definition in \eqref{eq:Delta_ij} we can denote $\bm{\Delta}_{ij} = r_{ij}\bm{I}_d$, where $r_{ij}$ is a random variable such that
$\mathbb{P}\{r_{ij} = 1 - p\} = p$ and $\mathbb{P}\{r_{ij} = -p\} = 1-p$. Then let $\bm{E}_1 \in \mathbb{R}^{m \times m}$ be a matrix that contains all $r_{ij}$ for $i, j \in C_1$ such that $r_{ij} = r_{ji}$ and $r_{ii} = 0$, one can see that $\widetilde{\bm{\Delta}}_{11} = \bm{E}_1 \otimes \bm{I}_d$ where $\otimes$ denotes the Kronecker product, and further
\begin{equation*}
    \|\widetilde{\bm{\Delta}}_{11}\| = \|\bm{E}_1 \otimes \bm{I}_d\| = \|\bm{E}_1\|.
\end{equation*}
Then from Theorem~\ref{the:bound_norm_2_A} we get $\|\widetilde{\bm{\Delta}}_{11}\| \lesssim \sqrt{p(1-p)n}$ w.h.p., where the residual term $c_2\sqrt{\log m}$ is absorbed owe to the assumption $p = \Omega(\log n/n)$ and $1-p = \Omega(\log n/n)$.  Similarly $\|\widetilde{\bm{\Delta}}_{22}\|$ is bounded as $\|\widetilde{\bm{\Delta}}_{22}\| \lesssim \sqrt{p(1-p)n}$ w.h.p. For $\|\widetilde{\bm{\Delta}}_{12}\|$, applying Theorem~\ref{lemma:spec_bound_S_12} gives $\|\widetilde{\bm{\Delta}}_{12}\| \lesssim \sqrt{qn}$ w.h.p. Putting all together yields the bound on $\|\bm{\Delta}\|$. 

For the block row $\bm{\Delta}_{i \cdot }$, let us consider another matrix $\bm{\Delta}^\prime \in \mathbb{R}^{nd \times nd}$ which is all zero but only the $i$-th block row is equal to $\bm{\Delta}_{i\cdot}$ (i.e. $\bm{\Delta}_{i\cdot}^\prime = \bm{\Delta}_{i\cdot}$). Notice that $\bm{\Delta}^\prime$ is of the same size as $\bm{\Delta}$, then by the definition of the operator norm $\|\bm{X}\| = \max_{\|v\|_2 = 1} \|\bm{X}v\|_2$, one can see that $\|\bm{\Delta}_{i\cdot}\| = \|\bm{\Delta}^\prime\| \leq \|\bm{\Delta}\|$.   

For $\bm{\Delta}^{(i)}$, let us consider another matrix $\bm{\Delta}^{(i), \text{col}}$ which is identical to $\bm{\Delta}$ but only the $i$-th block column is all zero, i.e. $\bm{\Delta}^{(i), \text{col}}_{\cdot i} = \bm{0}$. Then by definition of the operator norm we have $\|\bm{\Delta}^{(i), \text{col}}\| \leq \|\bm{\Delta}\|$. Furthermore, since $\bm{\Delta}^{(i)}$ is identical to $\bm{\Delta}^{(i), \text{col}}_{\cdot i}$ but only the $i$-th block row is all zero, then by following the same idea as above we have $\|\bm{\Delta}^{(i)}\| \leq \|\bm{\Delta}^{(i), \text{col}}_{\cdot i}\|$, which completes the proof.
\end{proof}

\subsection{Proof of Lemma~\ref{lemma:Delta_M}}
The proof relies on the following inequality.
\begin{theorem}[Matrix Bernstein inequality~\cite{tropp2015introduction}]
Let $\bm{X}_1, \ldots, \bm{X}_n \in \mathbb{R}^{d_1 \times d_2}$ be independent matrix such that $\mathbb{E}[\bm{X}_i] = \bm{0}$ and $\|\bm{X}_i\| \leq L$ for $i = 1,\ldots, n$. Let $\bm{Z} = \sum_{i = 1}^n \bm{X}_i$ and $v(\bm{Z}) := \max\{\|\mathbb{E}(\bm{Z}\bm{Z}^\top)\|, \; \|\mathbb{E}(\bm{Z}^\top\bm{Z})\|\}$.
Then for any $t > 0$ 
\begin{equation}
    \mathbb{P}\left\{\|\bm{Z}\| \geq t \right\} \leq (d_1 + d_2) \exp\left(\frac{-t^2/2}{v(\bm{Z}) + Lt/3}\right).
    \label{eq:matrix_bernstein}
\end{equation}
In other words, for any $c > 0$, 
\begin{equation}
    \|\bm{Z}\| \leq \sqrt{2cv(\bm{Z})\log(n(d_1 + d_2))} + \frac{2cL\log(n(d_1 + d_2))}{3}
    \label{eq:matrix_bernstein_bound}
\end{equation}
with probability $1 - n^{-c}$.
\label{lemma:matrix_bernstein}
\end{theorem}

To obtain \eqref{eq:matrix_bernstein_bound}, by setting the RHS of \eqref{eq:matrix_bernstein} to be $n^{-c}$ for some $c > 0$ we get 
\begin{equation*}
    t = \sqrt{2v(\bm{Z})\gamma}\left(\sqrt{1 + \frac{L^2\gamma}{18v(\bm{Z})}} + \sqrt{\frac{L^2\gamma}{18v(\bm{Z})}}\right), \quad \gamma := c\log n  + \log(d_1 + d_2).
\end{equation*}
Then when $c > 1$ it satisfies $\gamma \leq c\log(n(d_1+d_2))$ and $t \leq \sqrt{2v(\bm{Z})\gamma}\left(1 + 2\sqrt{\frac{L^2\gamma}{18v(\bm{Z})}}\right) \leq \sqrt{2v(\bm{Z})\gamma} + \frac{2L\gamma}{3}$,
which leads to \eqref{eq:matrix_bernstein_bound}.

\begin{proof}[Proof of Lemma~\ref{lemma:Delta_M}]
We directly apply Theorem~\ref{lemma:matrix_bernstein} by letting $\bm{X}_j = \bm{\Delta}_{ij}\bm{M}_{j\cdot}$. Clearly, $\mathbb{E}[\bm{X}_j] = 0$ and $\|\bm{X}_j\| \leq \|\bm{\Delta}_{ij}\|\|\bm{M}_{j \cdot}\| \leq \max_j \|\bm{M}_{j \cdot}\|$, for $j = 1,\ldots, n$,  
where $\|\bm{\Delta}_{ij}\| \leq 1$ by definition in \eqref{eq:Delta_ij}. Then $L = \max_j \|\bm{M}_{j \cdot}\|$.
For $v(\bm{Z})$ in Theorem~\ref{lemma:matrix_bernstein} where $\bm{Z} := \sum_{j = 1}^n\bm{X}_j$, we have
\begin{equation*}
    \begin{aligned}
    \mathbb{E}(\bm{Z}\bm{Z}^\top) &= \sum_{j = 1}^n \mathbb{E}\left[\bm{\Delta}_{ij}\bm{M}_{j\cdot }\bm{M}_{j \cdot}^\top\bm{\Delta}_{ij}^\top\right]\overset{(a)}{\preceq} \sum_{j = 1}^n \mathbb{E}\left[\bm{\Delta}_{ij}\bm{\Delta}_{ij}^\top\right]\|\bm{M}_{j \cdot}\|^2_2\\
    &\preceq \sum_{j = 1}^n \mathbb{E}\left[\bm{\Delta}_{ij}\bm{\Delta}_{ij}^\top\right] \max_{j}\|\bm{M}_{j \cdot}\|^2_2 \overset{(b)}{=} \left[\frac{n}{2}(p(1-p) + q)\max_{j}\|\bm{M}_{j \cdot}\|^2_2\right] \bm{I}_d
    \end{aligned}
\end{equation*}
where $(a)$ holds since $\bm{M}_{j \cdot}\bm{M}_{j \cdot}^\top \preceq \|\bm{M}_{j \cdot}\|^2 \bm{I}_d$ and $(b)$ holds by \eqref{eq:Delta_ij}. Similarly, one can see that $\mathbb{E}(\bm{Z}^\top\bm{Z}) =  \left[\frac{n}{2}(p(1-p) + q)\max_{j}\|\bm{M}_{j \cdot}\|^2_2\right] \bm{I}_d$. 
Then, by definition $v(\bm{Z}) = \frac{n}{2}(p(1-p)+q)\max_j\|\bm{M}_{j \cdot}\|^2$. In this way, Theorem~\ref{lemma:matrix_bernstein} gives us
\begin{equation*}
\begin{aligned}
    \|\bm{\Delta}_{i \cdot}\bm{M}\| &\leq \sqrt{2c(p(1-p)+q)n\log(2nd)}\max_j\|\bm{M}_{j \cdot}\| + \frac{2c}{3}\log(2nd) \max_j\|\bm{M}_{j \cdot}\| \\
    &\lesssim \sqrt{(p(1-p) + q)n\log(nd)}\max_j\|\bm{M}_{j \cdot}\|
\end{aligned}    
\end{equation*}
with probability $1 - n^{-c}$ for $c > 0$, 
where the last inequality holds because of the assumption $p(1-p) + q = \Omega(\log n/n)$ and then the second term $\frac{2c}{3}\log(2nd) \max_j\|\bm{M}_{j \cdot}\|$ is grouped into the first one. This completes the proof.
\end{proof}

\subsection{Proof of the lemmas in Appendix~\ref{sec:proof_the_1}}
\label{sec:proof_lemmas_the_1}

\begin{proof}[Proof of Lemma~\ref{lemma:Psi_Phi}]
For $\|\bm{\Phi} - \bm{\Psi}\bm{O}\|$, applying Theorem~\ref{the:davis_kahan} with \eqref{eq:Davis_kahan_bound} yields
\begin{equation*}
    \begin{aligned}
    \|\bm{\Phi} - \bm{\Psi}\bm{O}\| \lesssim \frac{\|\bm{\Delta}\bm{\Psi}\|}{\delta}
    \leq \frac{\|\bm{\Delta}\|\|\bm{\Psi}\|}{\delta}
    \end{aligned}
\end{equation*}
where $\delta = |\sigma_{2d}(\mathbb{E}[\bm{A}]) - \sigma_{2d+1}(\bm{A})|$. Here, $\sigma_{2d}(\mathbb{E}[\bm{A}]) = pn/2$ by definition, and $\sigma_{2d+1}(\bm{A}) \leq \sigma_{2d+1}(\mathbb{E}[\bm{A}]) + \|\bm{\Delta}\| = \|\bm{\Delta}\|$ by Theorem~\ref{the:weyl}. Then applying Lemma~\ref{lemma:delta_norm_2} gives $\delta \geq \sigma_{2d}(\mathbb{E}[\bm{A}]) -  \sigma_{2d+1}(\bm{A}) = \Omega(pn)$ w.h.p. Also, by definition $\|\bm{\Phi}\| = 1$. Combining these gives $\|\bm{\Phi} - \bm{\Psi}\bm{O}\| \lesssim \tau$ w.h.p. For $1 - \sigma_{\mathrm{min}}(\bm{\Psi}^\top\bm{\Phi})$, notice that
\begin{equation*}
    \|\bm{\Psi}^\top\bm{\Phi} - \bm{O}\| = \|\bm{\Psi}^\top(\bm{\Phi} - \bm{\Psi}\bm{O})\| \leq \|\bm{\Psi}\|\|\bm{\Phi} - \bm{\Psi}\bm{O}\| \lesssim \tau
\end{equation*}
w.h.p., this leads to 
\begin{equation*}
    \sigma_{\text{min}}(\bm{O}) - \sigma_{\text{min}}(\bm{\Psi}^\top\bm{\Phi}) = 1 - \sigma_{\text{min}}(\bm{\Psi}^\top\bm{\Phi}) \overset{(a)}{\leq} \|\bm{\Psi}^\top\bm{\Phi} - \bm{O}\| \lesssim \tau
\end{equation*}
w.h.p., where $(a)$ comes from Theorem~\ref{the:weyl}. The bounds on $\|\bm{\Phi}^{(i)} - \bm{\Psi}\bm{O}^{(i)}\|$ and $1 - \sigma_{\text{min}}(\bm{\Psi}^\top\bm{\Phi}^{(i)})$ are derived in a similar manner, therefore we do not repeat. 
\end{proof}

\vspace{0.1cm}
\begin{proof}[Proof of Lemma~\ref{lemma:Phi_Phi_i}]
Applying Theorem~\ref{the:davis_kahan} on $\|\bm{\Phi} - \bm{\Phi}^{(i)}\bm{O}^{(i)}\|$ yields
\begin{equation*}
    \|\bm{\Phi} - \bm{\Phi}^{(i)}\bm{O}^{(i)}\| \leq \frac{\|(\bm{A} - \bm{A}^{(i)})\bm{\Phi}^{(i)}\|}{\delta},
\end{equation*}
where $\delta = |\sigma_{2d}(\bm{A}) - \sigma_{2d+1}(\bm{A}^{(i)})|$.
To bound $\delta$, similar to the proof in Lemma~\ref{lemma:Psi_Phi} we have $\sigma_{2d}(\bm{A})\geq \sigma_{2d}(\mathbb{E}[\bm{A}]) - \|\bm{\Delta}\|$ and $\sigma_{2d+1}(\bm{A}^{(i)}) \leq \sigma_{2d+1}(\mathbb{E}[\bm{A}]) + \|\bm{\Delta}^{(i)}\| = \|\bm{\Delta}^{(i)}\|$ w.h.p., then applying Lemma~\ref{lemma:delta_norm_2} gives $\delta \geq \sigma_{2d}(\mathbb{E}[\bm{A}]) - 2\|\bm{\Delta}\| = \Omega(pn)$ w.h.p. To bound $\|(\bm{A} - \bm{A}^{(i)})\bm{\Phi}^{(i)}\|$, notice that $\bm{A} - \bm{A}^{(i)} = \bm{\Delta} - \bm{\Delta}^{(i)}$ which is only non-zero on the $i$-th block row and the $i$-th block column, then
\begin{equation*}
    \begin{aligned}
    \|(\bm{A} - \bm{A}^{(i)})\bm{\Phi}^{(i)}\| &=  
    \|(\bm{\Delta} - \bm{\Delta}^{(i)})\bm{\Phi}^{(i)}\| 
    \overset{(a)}{\leq} \|\bm{\Delta}_{i\cdot}\bm{\Phi}^{(i)}\| + \|\bm{\Delta}_{i\cdot}^\top\bm{\Phi}_{i\cdot}^{(i)}\| \\
    &\leq \|\bm{\Delta}_{i\cdot}\bm{\Phi}^{(i)}\| + \|\bm{\Delta}_{i\cdot}\|\|\bm{\Phi}_{i\cdot}^{(i)}\|, 
    \end{aligned}
\end{equation*}
where $(a)$ holds by separating the $i$-th block row and block column with the triangle inequality follows. To proceed, since $\bm{\Delta}_{i\cdot}$ is independent of $\bm{\Phi}^{(i)}$,  $\|\bm{\Delta}_{i\cdot}\bm{\Phi}^{(i)}\|$ is bounded by Lemma~\ref{lemma:Delta_M} and yields
\begin{equation*}
    \begin{aligned}
    \|\bm{\Delta}_{i\cdot}\bm{\Phi}^{(i)}\| 
    &\lesssim \sqrt{(p(1-p)+q)n \log (nd)}\max_j\|\bm{\Phi}_{j\cdot}^{(i)}\|
    \end{aligned}
\end{equation*}
w.h.p. Also, $\|\bm{\Delta}_{i\cdot}\|\|\bm{\Phi}_{i\cdot}^{(i)}\| \leq (\sqrt{p(1-p)n} + \sqrt{qn})\max_j\|\bm{\Phi}_{j\cdot}^{(i)}\|$ w.h.p., where $\|\bm{\Delta}_{i\cdot}\|$ is bounded by Lemma~\ref{lemma:delta_norm_2}. Combining the result above gives
\begin{equation}
    \|\bm{\Phi} - \bm{\Phi}^{(i)}\bm{O}^{(i)}\| \lesssim \eta \max_j \|\bm{\Phi}_{j\cdot}^{(i)}\|
    \label{eq:phi_phi_i_O}
\end{equation}
w.h.p. Next we show that $\max_j \|\bm{\Phi}_{j\cdot}^{(i)}\| \lesssim \max_j \|\bm{\Phi}_{j\cdot}\|$. Let $j^\prime$ be the index such that $\|\bm{\Phi}^{(i)}_{j^\prime \cdot}\| = \max_j \|\bm{\Phi}_{j\cdot}^{(i)}\|$, then we have 
\begin{equation}
    \begin{aligned}
    \|\bm{\Phi}_{j^\prime \cdot}^{(i)}\| - \|\bm{\Phi}_{j^\prime \cdot}\| &\overset{(a)}{\leq} \|\bm{\Phi}_{j^\prime \cdot} - \bm{\Phi}^{(i)}_{j^\prime \cdot}\bm{O}^{(i)}\| \overset{(b)}{\leq} \|\bm{\Phi} - \bm{\Phi}^{(i)}\bm{O}^{(i)}\| \lesssim \eta \max_j \|\bm{\Phi}_{j\cdot}^{(i)}\| = \eta \|\bm{\Phi}_{j^\prime \cdot}^{(i)}\|
    \end{aligned}
\label{eq:phi_j_bound}
\end{equation}
w.h.p., where $(a)$ comes from the triangle inequality, and $(b)$ holds since $\bm{\Phi}_{j^\prime \cdot} - \bm{\Phi}^{(i)}_{j^\prime \cdot}\bm{O}^{(i)}$ is the $j^\prime$-th block row of $\bm{\Phi} - \bm{\Phi}^{(i)}\bm{O}^{(i)}$. This implies if the condition \eqref{eq:p_q_assumption} that $\eta \leq c_0$ holds for a sufficiently small $c_0$, then $\max_j \|\bm{\Phi}_{j\cdot}^{(i)}\| = \|\bm{\Phi}_{j^\prime \cdot}^{(i)}\| \lesssim \|\bm{\Phi}_{j^\prime \cdot}\| \leq \max_j\|\bm{\Phi}_{j\cdot}\|$
w.h.p. 
Plugging this into \eqref{eq:phi_phi_i_O} completes the proof.
\end{proof}

\vspace{0.1cm}
\begin{proof}[Proof of Lemma~\ref{lemma:phi_i_psi_j}]
From the triangle inequality we obtain
\begin{align}
    \max_{j}\|\bm{\Phi}_{j\cdot}^{(i)} - \bm{\Psi}_{j\cdot}\bm{S}^{(i)}\| &\leq \max_j(\|\bm{\Phi}_{j\cdot}^{(i)}\| + \|\bm{\Psi}_{j\cdot}\bm{S}^{(i)}\|) \leq \max_j\|\bm{\Phi}_{j\cdot}^{(i)}\| + \max_j\|\bm{\Psi}_{j\cdot}\| \nonumber \\
    &\overset{(a)}{\lesssim} \max_j\|\bm{\Phi}_{j\cdot}\| + (1/\sqrt{n})
    \label{eq:phi_j_1_sqrtn}
\end{align}
w.h.p., 
where $(a)$ holds since $\max_j\|\bm{\Phi}_{j\cdot}^{(i)}\| \lesssim \max_j\|\bm{\Phi}_{j\cdot}\|$ in Lemma~\ref{lemma:Phi_Phi_i}, and $\|\bm{\Psi}_{j\cdot}\| = 1/\sqrt{m}$ by definition. 
It remains to show $\max_j\|\bm{\Phi}_{j\cdot}\| \geq 1/\sqrt{n}$, which comes from
\begin{equation*}
    1 = \|\bm{\Phi}\| \leq \sqrt{\sum_{i = 1}^n\|\bm{\Phi}_{i\cdot}\|^2} \leq \sqrt{n} \max_j \|\bm{\Phi}_{j\cdot}\|.
\end{equation*}
Plugging $\max_j\|\bm{\Phi}_{j\cdot}\| \geq 1/\sqrt{n}$ back into \eqref{eq:phi_j_1_sqrtn} completes the proof.
\end{proof}

\vspace{0.1cm}
\begin{proof}[Proof of Lemma~\ref{lemma:S_i_O_i_O}]
By definition, 
\begin{align}
    \|\bm{S}^{(i)}\bm{O}^{(i)} - \bm{O}\| &= \|\mathcal{P}(\bm{\Psi}^\top\bm{\Phi}^{(i)})\bm{O}^{(i)} - \mathcal{P}(\bm{\Psi}^\top\bm{\Phi})\| = \|\mathcal{P}(\bm{\Psi}^\top\bm{\Phi}^{(i)}\bm{O}^{(i)}) - \mathcal{P}(\bm{\Psi}^\top\bm{\Phi})\|  \nonumber\\
    &\overset{(a)}{\leq} 2\sqrt{2}\min\{\sigma_{\text{min}}(\bm{\Psi}^\top\bm{\Phi}^{(i)}), \; \sigma_{\text{min}}(\bm{\Psi}^\top\bm{\Phi})\}\|\bm{\Psi}^\top(\bm{\Phi}^{(i)}\bm{O}^{(i)} - \bm{\Phi})\|. \label{eq:sigma_phi}
\end{align}
where $(a)$ comes from Lemma~\ref{lemma:PX_PY}. To proceed, from Lemma~\ref{lemma:Psi_Phi} we have $1 - \sigma_{\text{min}}(\bm{\Psi}^\top\bm{\Phi}^{(i)}) \lesssim \tau$ and $1 - \sigma_{\text{min}}(\bm{\Psi}^\top\bm{\Phi}) \lesssim \tau$ w.h.p. Moreover, under the condition \eqref{eq:p_q_assumption} that $\eta \leq c_0$, it satisfies 
\begin{equation*}
    \tau = \frac{\sqrt{p(1-p)} + \sqrt{q}}{p\sqrt{n}} \overset{(a)}{\leq} \frac{\sqrt{2(p(1-p) + q)}}{p\sqrt{n} } \leq \frac{\sqrt{2}c_0}{\log(nd)} = o(1)
\end{equation*}
where $(a)$ uses the fact $\sqrt{x} + \sqrt{y} \leq \sqrt{2(x + y)}$ for any $x, y \geq 0$. 
This implies $\sigma_{\text{min}}(\bm{\Psi}^\top\bm{\Phi}^{(i)}) = \Omega(1), \sigma_{\text{min}}(\bm{\Psi}^\top\bm{\Phi}) = \Omega(1)$ w.h.p. Plugging this back into \eqref{eq:sigma_phi} yields
\begin{equation*}
    \begin{aligned}
    \|\bm{S}^{(i)}\bm{O}^{(i)} - \bm{O}\| &\lesssim \|\bm{\Psi}^\top(\bm{\Phi}^{(i)}\bm{O}^{(i)} - \bm{\Phi})\| \leq \|\bm{\Psi}\|\|\bm{\Phi} - \bm{\Phi}^{(i)}\bm{O}^{(i)}\| \lesssim \eta\max_j \|\bm{\Phi}_{j\cdot}\|
    \end{aligned}
\end{equation*}
w.h.p., 
where Lemma~\ref{lemma:Phi_Phi_i} is applied to bound $\|\bm{\Phi} - \bm{\Phi}^{(i)}\bm{O}^{(i)}\|$. 
\end{proof}

\vspace{0.1cm}
\begin{proof}[Proof of Lemma~\ref{lemma:upper_bound_phi_j}]
Let $i^\prime \in C_1$ and $j^\prime \in C_2$ be two nodes such that 
\begin{equation*}
    \|\bm{\Phi}_{i^\prime \cdot}\| = \max_{i \in C_1}\|\bm{\Phi}_{i\cdot}\|, \quad \|\bm{\Phi}_{j^\prime \cdot}\| = \max_{j \in C_2}\|\bm{\Phi}_{j\cdot}\|.
\end{equation*}
Without loss of generality we assume $\|\bm{\Phi}_{j^{\prime}}\| = \max_j\|\bm{\Phi}_{j\cdot}\|$ which is the largest block among all nodes. Then, we define a matrix $\bm{\Phi}^\prime \in \mathbb{R}^{nd \times 2d}$ which has the same size of $\bm{\Phi}$ and is formed by $\bm{\Phi}_{i^\prime \cdot}, \bm{\Phi}_{j^\prime \cdot}$ as
\begin{equation*}
    (\bm{\Phi}^\prime)^\top = 
    \begin{bmatrix}
    \smash[b]{\underbrace{
    \begin{matrix}
    \bm{\Phi}_{i^\prime \cdot}^\top &\cdots &\bm{\Phi}_{i^\prime \cdot}^\top
    \end{matrix}}_{m \times \bm{\Phi}_{i^\prime \cdot}^\top}
    }
    &
    \smash[b]{\underbrace{
    \begin{matrix}
    \bm{\Phi}_{j^\prime \cdot}^\top &\cdots &\bm{\Phi}_{j^\prime \cdot}^\top
    \end{matrix}}_{m \times \bm{\Phi}_{j^\prime \cdot}^\top}
    }
    \end{bmatrix}
    \vspace{0.35cm}
\end{equation*}
As a result, $\bm{\Phi}^\prime$ is close to $\bm{\Phi}$ such that 
\begin{equation}
    \|\bm{\Phi} - \bm{\Phi}^\prime\| \overset{(a)}{\leq} \sqrt{\sum_{i = 1}^m \|\bm{\Phi}_{i\cdot} - \bm{\Phi}_{i^\prime \cdot}\|^2 + \sum_{j = m+1}^n \|\bm{\Phi}_{j\cdot} - \bm{\Phi}_{j^\prime \cdot}\|^2} \overset{(b)}{\lesssim} \sqrt{n}\eta\|\bm{\Phi}_{j^\prime \cdot}\|
    \label{eq:phi_phi_prime}
\end{equation}
w.h.p., where $(a)$ holds by definition of the operator norm, and $(b)$ comes from \eqref{eq:bound_phi_i_phi_j} in the proof of Lemma~\ref{the:phi_i_phi_j}. This leads to 
\begin{equation}
    |\sigma_{\text{max}}(\bm{\Phi}) - \sigma_{\text{max}}(\bm{\Phi}^\prime)| \overset{(a)}{=} |1 - \sigma_{\text{max}}(\bm{\Phi}^\prime)| \overset{(b)}{\leq} \|\bm{\Phi} - \bm{\Phi}^\prime\| \lesssim \sqrt{n}\eta\|\bm{\Phi}_{j^\prime \cdot}\|
    \label{eq:1_sigma}
\end{equation}
w.h.p., 
where (a) uses $\sigma_{\text{max}}(\bm{\Phi}) = 1$ since $\bm{\Phi}^\top\bm{\Phi} = \bm{I}_{2d}$ by definition, (b) comes from Weyl's inequality. On the other hand, by definition of the operator norm
\begin{equation*}
\begin{aligned}
    \sigma_{\text{max}}(\bm{\Phi}^\prime) &= \max_{\|\bm{x}\| = 1}\|\bm{\Phi}^\prime\bm{x}\| = \max_{\|\bm{x}\| = 1} \sqrt{\sum_{i = 1}^m \|\bm{\Phi}_{i^\prime \cdot}\bm{x}\|^2 + \sum_{j = m+1}^n \|\bm{\Phi}_{j^\prime \cdot}\bm{x}\|^2} \\
    &\geq \max_{\|\bm{x}\| = 1}
    \sqrt{\sum_{j = m+1}^n\|\bm{\Phi}_{j^\prime \cdot}\bm{x}\|^2} = \sqrt{m}\|\bm{\Phi}_{j^\prime \cdot}\|.
\end{aligned}
\end{equation*}
Combining this with \eqref{eq:1_sigma} gives
\begin{equation*}
    \sqrt{m}\|\bm{\Phi}_{j^\prime \cdot}\| - 1 \leq \sigma_{\text{max}}(\bm{\Phi}^\prime) - 1 \lesssim \sqrt{n}\eta\|\bm{\Phi}_{j^\prime \cdot}\|
\end{equation*}
w.h.p. This implies as long as the condition \eqref{eq:p_q_assumption} that $\eta \leq c_0$ for a sufficiently small $c_0$, it satisfies $\|\bm{\Phi}_{j^\prime \cdot}\| = O(1/\sqrt{n})$, which completes the proof. 
\end{proof}

\subsection{Proof of the lemmas in Appendix~\ref{sec:proof_the_2}}
\label{sec:proof_lemma_the_2}

\begin{proof}[Proof of Lemma~\ref{lemma:sigma_phi_i}]
By definition, $\sigma_l(\bm{\Psi}_{i\cdot}\bm{O}) = \sqrt{2/n}, l = 1, \ldots, d$. Then,
\begin{equation*}
    \|\sigma_l(\bm{\Phi}_{i\cdot}) - \sqrt{2/n}\| \overset{(a)}{\leq} \|\bm{\Phi}_{i\cdot} - \bm{\Psi}_{i\cdot}\bm{O}\| \overset{(b)}{\lesssim} \eta/\sqrt{n},\quad l = 1,\ldots, d
\end{equation*}
w.h.p., where $(a)$ holds by Weyl's inequality in Theorem~\ref{the:weyl} and $(b)$ comes from Theorem~\ref{the:phi_psi_i}. $\|\mathcal{P}(\bm{\Phi}_{i\cdot}) - \mathcal{P}(\bm{\Psi}_{i\cdot}\bm{O})\|$ is bounded by Lemma~\ref{lemma:PX_PY} as
\begin{equation*}
    \begin{aligned}
        &\|\mathcal{P}(\bm{\Phi}_{i\cdot}) - \mathcal{P}(\bm{\Psi}_{i\cdot}\bm{O})\| \leq 2\sqrt{2}\min\{\sigma_{\text{min}}^{-1}(\bm{\Phi}_{i\cdot}),\; \sigma_{\text{min}}^{-1}(\bm{\Psi}_{i\cdot}\bm{O})\}\|\bm{\Phi}_{i\cdot} - \bm{\Psi}_{i\cdot}\bm{O}\|\\
        &\lesssim \min\left\{(\sqrt{2/n} - O(\eta/\sqrt{n}))^{-1}, \; (\sqrt{2/n})^{-1}\right\} \cdot O(\eta/\sqrt{n}) = O\left(\eta\right)
    \end{aligned}
\end{equation*}
w.h.p. For $\|\bm{\Phi}_{j\cdot} \mathcal{P}(\bm{\Phi}_{i\cdot})^\top - \bm{\Psi}_{j\cdot} \mathcal{P}(\bm{\Psi}_{i\cdot})^\top\|$, by substituting $\bm{\Phi}_{j\cdot}$ and $\mathcal{P}(\bm{\Psi}_{i\cdot})$ with $\bm{\Psi}_{j\cdot}\bm{O} + \bm{\Phi}_{j\cdot} - \bm{\Psi}_{j\cdot}\bm{O} = \bm{\Psi}_{j\cdot}\bm{O} $ and $\mathcal{P}(\bm{\Psi}_{i\cdot}\bm{O}) + \mathcal{P}(\bm{\Phi}_{i\cdot}) - \mathcal{P}(\bm{\Psi}_{i\cdot}\bm{O})$ respectively, we denote $\bm{\Delta}_{\bm{\Phi}_{j\cdot}} := \bm{\Phi}_{j\cdot} - \bm{\Psi}_{j\cdot}\bm{O}$ and $\bm{\Delta}_{\mathcal{P}(\bm{\Phi}_{i\cdot})} := \mathcal{P}(\bm{\Phi}_{i\cdot}) - \mathcal{P}(\bm{\Psi}_{i\cdot}\bm{O})$, then we get
\begin{align*}
    &\|\bm{\Phi}_{j\cdot} \mathcal{P}(\bm{\Phi}_{i\cdot})^\top - \bm{\Psi}_{j\cdot} \mathcal{P}(\bm{\Psi}_{i\cdot})^\top\| \\
    &= \|(\bm{\Psi}_{j\cdot}\bm{O} + \bm{\Delta}_{\bm{\Phi}_{j\cdot}})(\mathcal{P}(\bm{\Psi}_{i\cdot}\bm{O}) + \bm{\Delta}_{\mathcal{P}(\bm{\Phi}_{i\cdot})})^\top - \bm{\Psi}_{j\cdot} \mathcal{P}(\bm{\Psi}_{i\cdot})^\top\| \\
    &= \|\bm{\Delta}_{\bm{\Phi}_{j\cdot}}\mathcal{P}(\bm{\Psi}_{j\cdot}\bm{O})^\top + \bm{\Psi}_{j\cdot} \bm{O}\bm{\Delta}_{\mathcal{P}(\bm{\Phi}_{i\cdot})}^\top + \bm{\Delta}_{\bm{\Phi}_{j\cdot}}\bm{\Delta}_{\mathcal{P}(\bm{\Phi}_{i\cdot})}^\top\| \\
    &\leq \|\bm{\Delta}_{\bm{\Phi}_{j\cdot}}\|\|\mathcal{P}(\bm{\Psi}_{j\cdot})\| + \|\bm{\Psi}_{j\cdot}\|\|\bm{\Delta}_{\mathcal{P}(\bm{\Phi}_{i\cdot})}\| + \|\bm{\Delta}_{\bm{\Phi}_{j\cdot}}\|\|\bm{\Delta}_{\mathcal{P}(\bm{\Phi}_{i\cdot})}\|\\
    &\lesssim \frac{\eta}{\sqrt{n}} \cdot 1 + \sqrt{\frac{2}{n}} \cdot \eta + \frac{\eta}{\sqrt{n}} \cdot \eta \lesssim \frac{\eta}{\sqrt{n}}
\end{align*}
w.h.p., where $\|\bm{\Delta}_{\bm{\Phi}_{j\cdot}}\|$ is bounded by Theorem~\ref{the:phi_psi_i}. It remains to bound the last term $\|\mathcal{P}(\bm{\Phi}_{j\cdot} \mathcal{P}(\bm{\Phi}_{i\cdot})^\top) - \mathcal{P}(\bm{\Psi}_{j\cdot} \mathcal{P}(\bm{\Psi}_{i\cdot})^\top)\|$, by Lemma~\ref{lemma:PX_PY} we have
\begin{align*}
    &\|\mathcal{P}(\bm{\Phi}_{j\cdot} \mathcal{P}(\bm{\Phi}_{i\cdot})^\top) - \mathcal{P}(\bm{\Psi}_{j\cdot} \mathcal{P}(\bm{\Psi}_{i\cdot})^\top)\| \\
    &\lesssim \min\{\sigma_{\text{min}}^{-1}(\bm{\Phi}_{j\cdot} \mathcal{P}(\bm{\Phi}_{i\cdot})^\top),\; \sigma_{\text{min}}^{-1}(\bm{\Psi}_{j\cdot} \mathcal{P}(\bm{\Psi}_{i\cdot})^\top)\}\|\bm{\Phi}_{j\cdot} \mathcal{P}(\bm{\Phi}_{i\cdot})^\top - \bm{\Psi}_{j\cdot} \mathcal{P}(\bm{\Psi}_{i\cdot})^\top\| \\
    &\lesssim \min\left\{(\sqrt{2/n} - O(\eta/\sqrt{n}))^{-1}, \; (\sqrt{2/n})^{-1}\right\} \cdot O(\eta/\sqrt{n}) = O\left(\eta\right)
\end{align*}
w.h.p., where $\sigma_{\text{min}}(\bm{\Psi}_{j\cdot} \mathcal{P}(\bm{\Psi}_{i\cdot})^\top) = \sqrt{2/n}$, and 
\begin{equation*}
    \begin{aligned}
    \sigma_{\text{min}}(\bm{\Phi}_{j\cdot} \mathcal{P}(\bm{\Phi}_{i\cdot})^\top) &\geq \sigma_{\text{min}}(\bm{\Psi}_{j\cdot} \mathcal{P}(\bm{\Psi}_{i\cdot})^\top) - \|\bm{\Phi}_{j\cdot} \mathcal{P}(\bm{\Phi}_{i\cdot})^\top - \bm{\Psi}_{j\cdot} \mathcal{P}(\bm{\Psi}_{i\cdot})^\top\|\\ &\geq \sqrt{2/n} - O(\eta/\sqrt{n})
    \end{aligned}
\end{equation*}
w.h.p., which is bounded by Theorem~\ref{the:weyl}. This completes the proof.
\end{proof}

\vspace{0.1cm}
\begin{proof}[Proof of Lemma~\ref{lemma:Q_2}]
For $\|\bm{Q}_{\cdot 2} - \mathcal{P}(\bm{\Psi}_{j\cdot}\bm{O})^\top \bar{\bm{O}}_2\|$, the following satisfies
\begin{align*}
    &\|\bm{Q}_{\cdot 2} - \mathcal{P}(\bm{\Psi}_{j\cdot}\bm{O})^\top \bar{\bm{O}}_2\| = \|\bm{Q}_{\cdot 1}^\perp - \mathcal{P}(\bm{\Psi}_{j\cdot}\bm{O})^\top \bar{\bm{O}}_2\| \overset{(a)}{\lesssim} \|\bm{Q}_{\cdot 1}(\bm{Q}_{\cdot 1})^\top \mathcal{P}(\bm{\Psi}_{j\cdot}\bm{O})^\top \bar{\bm{O}}_2\|\\
    &= \|(\bm{Q}_{\cdot 1})^\top \mathcal{P}(\bm{\Psi}_{j\cdot}\bm{O})^\top \bar{\bm{O}}_2\| = \|\bar{\bm{O}}_1^\top\mathcal{P}(\bm{\Phi}_{i\cdot})\mathcal{P}(\bm{\Psi}_{j\cdot}\bm{O})^\top \bar{\bm{O}}_2\| = \|\mathcal{P}(\bm{\Phi}_{i\cdot})\mathcal{P}(\bm{\Psi}_{j\cdot}\bm{O})^\top\| \\
    &= \|\mathcal{P}(\bm{\Phi}_{i\cdot})\mathcal{P}(\bm{\Psi}_{j\cdot}\bm{O})^\top\| \leq \|\mathcal{P}(\bm{\Psi}_{i\cdot}\bm{O})\mathcal{P}(\bm{\Psi}_{j\cdot}\bm{O})^\top\| + \|(\mathcal{P}(\bm{\Psi}_{i\cdot}) - \mathcal{P}(\bm{\Psi}_{i\cdot}\bm{O}))\mathcal{P}(\bm{\Psi}_{j\cdot})^\top\| \\
    &\overset{(b)}{\leq} \|\mathcal{P}(\bm{\Psi}_{i\cdot})\mathcal{P}(\bm{\Psi}_{j\cdot})^\top\| + \|\mathcal{P}(\bm{\Psi}_{i\cdot}) - \mathcal{P}(\bm{\Psi}_{i\cdot}\bm{O})\|\|\mathcal{P}(\bm{\Psi}_{j\cdot})^\top\| \lesssim \eta
\end{align*}
w.h.p., where $(a)$ comes from Lemma~\ref{lemma:error2inner} with $\bm{X} = \bm{Q}_{\cdot 1}^\perp$ then $\bm{I}_{2d} - \bm{X}\bm{X}^\top = \bm{Q}_{\cdot 1}(\bm{Q}_{\cdot 1})^\top$,  and $(b)$ uses \eqref{eq:P_Phi_P_Psi} in Lemma~\ref{lemma:sigma_phi_i}. Next, for $\|\bm{\Phi}_{j\cdot}\bm{Q}_{\cdot 2} - \bm{\Psi}_{j\cdot}\mathcal{P}(\bm{\Psi}_{j\cdot})^\top \bar{\bm{O}}_2\|$ we have 
\begin{align*}
    &\|\bm{\Phi}_{j\cdot}\bm{Q}_{\cdot 2} - \bm{\Psi}_{j\cdot}\mathcal{P}(\bm{\Psi}_{j\cdot})^\top \bar{\bm{O}}_2\| \leq \|\bm{\Psi}_{j\cdot}\bm{O}\bm{Q}_{\cdot 2} -  \bm{\Psi}_{j\cdot}\mathcal{P}(\bm{\Psi}_{j\cdot})^\top \bar{\bm{O}}_2\| + \|(\bm{\Phi}_{j\cdot} - \bm{\Psi}_{j\cdot}\bm{O})\bm{Q}_{\cdot 2}\|\\
    & \leq \|\bm{\Psi}_{j\cdot}\|\|\bm{Q}_{\cdot 2} - \mathcal{P}(\bm{\Psi}_{j\cdot}\bm{O})^\top \bar{\bm{O}}_2\| + \|\bm{\Phi}_{j\cdot} - \bm{\Psi}_{j\cdot}\bm{O}\|\|\bm{Q}_{\cdot 2}\| \lesssim \eta/\sqrt{n}
\end{align*}
w.h.p. The bounds on $\|\mathcal{P}(\bm{\Phi}_{j\cdot}\bm{Q}_{\cdot 2}) - \mathcal{P}(\bm{\Psi}_{j\cdot}\mathcal{P}(\bm{\Psi}_{j\cdot})^\top \bar{\bm{O}}_2)\|$ can be obtained by applying Lemma~\ref{lemma:PX_PY} with Theorem~\ref{the:weyl}, which is similar to \eqref{eq:P_P_Phi_P_Psi} in Lemma~\ref{lemma:sigma_phi_i} and thus we do not repeat. 
\end{proof}

	\section{A more involved noise model}
\label{sec:new_noise_model}
{\color{black}
In this section, we study a more involved noise model that extends the one introduced in Section~\ref{sec:pre}. Recall that the orthogonal transform $\bm{A}_{ij}$ is measured exactly when nodes $i$ and $j$ belong to the same cluster. To extend from the measurement model in~\eqref{eq:clean_observation}, we include additive noise perturbation to~\eqref{eq:clean_observation} and the new noisy measurement $\widetilde{\bm{A}}_{ij}$ 
for any $i < j$ follows 
\begin{equation}
  \widetilde{\bm{A}}_{ij} =  \bm{A}_{ij} + \bm{W}_{ij}, 
    \label{eq:clean_observation_additive_noise}
\end{equation}
where $\bm{W}_{ij}$ denotes the additive noise that is independent for each pair of nodes $(i,j)$. For now there is no need to specify the statistics of $\bm{W}_{ij}$, but only make sure that $\bm{W}_{ii} = \bm{0}$ for $i = 1,\ldots, n$; $\mathbb{E}[\bm{W}_{ij}] = \bm{0}$ for any pair of nodes; and $\bm{W}_{ij} = \bm{W}_{ji}$. Notably, such additive noise model was also broadly considered in orthogonal synchronization problem, e.g.~\cite{ling2020near,ling2020solving}. 

In this case, our proposed Algorithm~\ref{alg:spectral} still applies and is able to recover the cluster memberships and the orthogonal transforms, as we show both theoretically and empirically in the following.

\subsection{Analysis} 
Our analysis for this noise model is still based on the setting of two clusters with equal cluster sizes. First, let us denote $\bm{W} = [\bm{W}_{ij}]_{i,j = 1}^n \in \mathbb{R}^{nd \times nd}$ as the whole symmetric matrix of additive noise, then we require the following assumptions on $\bm{W}$:

\begin{assumption}[Operator norm]
It satisfies $\|\bm{W}\| \leq \epsilon_0$ for some $$\epsilon_0 = O\big(\sqrt{p(1-p)n} + \sqrt{qn}\big),$$ with probability at least $1 - O(n^{-1})$.
\label{assump:operator_norm}
\end{assumption}
\begin{assumption}[Block row sum concentration]
Given any block matrix $\bm{M} \in \mathbb{R}^{nd \times r}$ with $n$ block rows, for each block row $\bm{W}_{i \cdot}$ it satisfies
\begin{equation*}
    \|\bm{W}_{i \cdot}\bm{M}\| \lesssim \epsilon \max_j\|\bm{M}_{j\cdot }\|
\end{equation*}
for some $\epsilon > 0$, with probability at least $1 - O(n^{-1})$.
\label{assump:block_row_sum}
\end{assumption}

Both Assumptions~\ref{assump:operator_norm} and \ref{assump:block_row_sum} are reasonable as Assumption~\ref{assump:operator_norm} states an overall concentration on $\bm{W}_{ij}$, and Assumption~\ref{assump:block_row_sum} further confines on each block row of $\bm{W}$. In particular, such a block row concentration is necessary for getting a blockwise analysis. Given the above, we are able to derive the following blockwise error bound between the noisy eigenvectors $\bm{\Phi}$ and the clean ones $\bm{\Psi}$:
\begin{theorem}[Blockwise error bound]
Under Assumptions~\ref{assump:operator_norm},\ref{assump:block_row_sum}, and the setting of two equal-sized clusters with model parameters $(n, p, q, d, \epsilon_0, \epsilon)$, for a sufficiently large $n$, suppose
    \begin{equation}
    \tilde{\eta} := \frac{\sqrt{(p(1-p)+q)n\log (nd)} + \epsilon}{pn} \leq c_0
    \label{eq:p_q_assumption_new_model}
\end{equation}
for some small constant $c_0 \geq 0$. Then with probability $1 - O(n^{-1})$,
\begin{equation}
    \max_{1 \leq i \leq n} \|\bm{\Phi}_{i\cdot} - \bm{\Psi}_{i\cdot}\bm{O}\| \lesssim \frac{\tilde{\eta}}{\sqrt{n}}
    \label{eq:bound_max_block_new_model}
\end{equation}
where $\bm{O} = \mathcal{P}(\bm{\Psi}^\top\bm{\Phi})$. 
\label{the:phi_psi_i_new_model}
\end{theorem}

\begin{proof}
The proof structure essentially follows the one of Theorem~\ref{the:phi_psi_i}, with the only difference on the analysis of the perturbation $\widetilde{\bm{\Delta}} = \widetilde{\bm{A}} - \mathbb{E}[\widetilde{\bm{A}}]$. In this case, 
$\widetilde{\bm{\Delta}} = \bm{\Delta} + \bm{W}$ where $\bm{\Delta}$ 
represents the original perturbation defined in \eqref{eq:Delta_ij}. 
As a result, by applying the triangular inequality, Lemma~\ref{lemma:delta_norm_2} for the operator norm bound $\|\widetilde{\bm{\Delta}}\|$ can be bounded as 
\begin{equation*}
    \|\widetilde{\bm{\Delta}}\| \leq \|\bm{\Delta}\| + \|\bm{W}\| \leq \sqrt{p(1-p)n} + \sqrt{qn} + \epsilon_0.
\end{equation*}
Similarly, the block row sum inequality in Lemma~\ref{lemma:Delta_M} can be given as 
\begin{equation*}
    \| \widetilde{\bm{\Delta}}_{i \cdot}\bm{M}\| \leq \|\bm{\Delta}_{i \cdot} \bm{M}\| + \|\bm{W}_{i \cdot}\bm{M}\| \leq \sqrt{(p(1-p) + q)n\log (nd)} + \epsilon.
\end{equation*}
Then, the remaining proof is almost identical to the one presented in Appendix~\ref{sec:proof_the_1}, but only replacing the original bounds on  $\| \widetilde{\bm{\Delta}}\|$ and $\| \widetilde{\bm{\Delta}}_{i \cdot}\bm{M}\|$ with the updated ones, and we leave the detailed proof to interested readers. 
\end{proof}

As a result, the performance guarantee of Algorithm~\ref{alg:spectral} on the new noise model is identical to Theorem~\ref{the:cond}, and we re-state here:
\begin{theorem}[Performance guarantee]
Under the assumption of Theorem~\ref{the:phi_psi_i_new_model}, for $i = 1,\ldots, n$, with probability $1 - O(n^{-1})$, Algorithm~\ref{alg:spectral} exactly recovers the cluster memberships $\kappa(i)$ defined in Section~\ref{sec:pre}, and $\hat{\bm{O}}_i$ satisfies
\begin{equation}
    \|\hat{\bm{O}}_i - \bm{O}_i\bar{\bm{O}}_{\kappa(i)}\| \lesssim \tilde{\eta}
    \label{eq:sync_bounded_new_model}
\end{equation}
where $\bar{\bm{O}}_{\kappa(i)}$ is orthogonal and only depends on the cluster that $i$ belongs to.
\label{the:cond_new_model}
\end{theorem}

\begin{proof}
The proof is identical to the one of Theorem~\ref{the:cond} and therefore we do not repeat.
\end{proof}

\begin{figure}[t!]
    \centering
    \subfloat[\scriptsize{Success rate of exact recovery, $\sigma = 0.5$}]{\includegraphics[width = 0.4\textwidth]{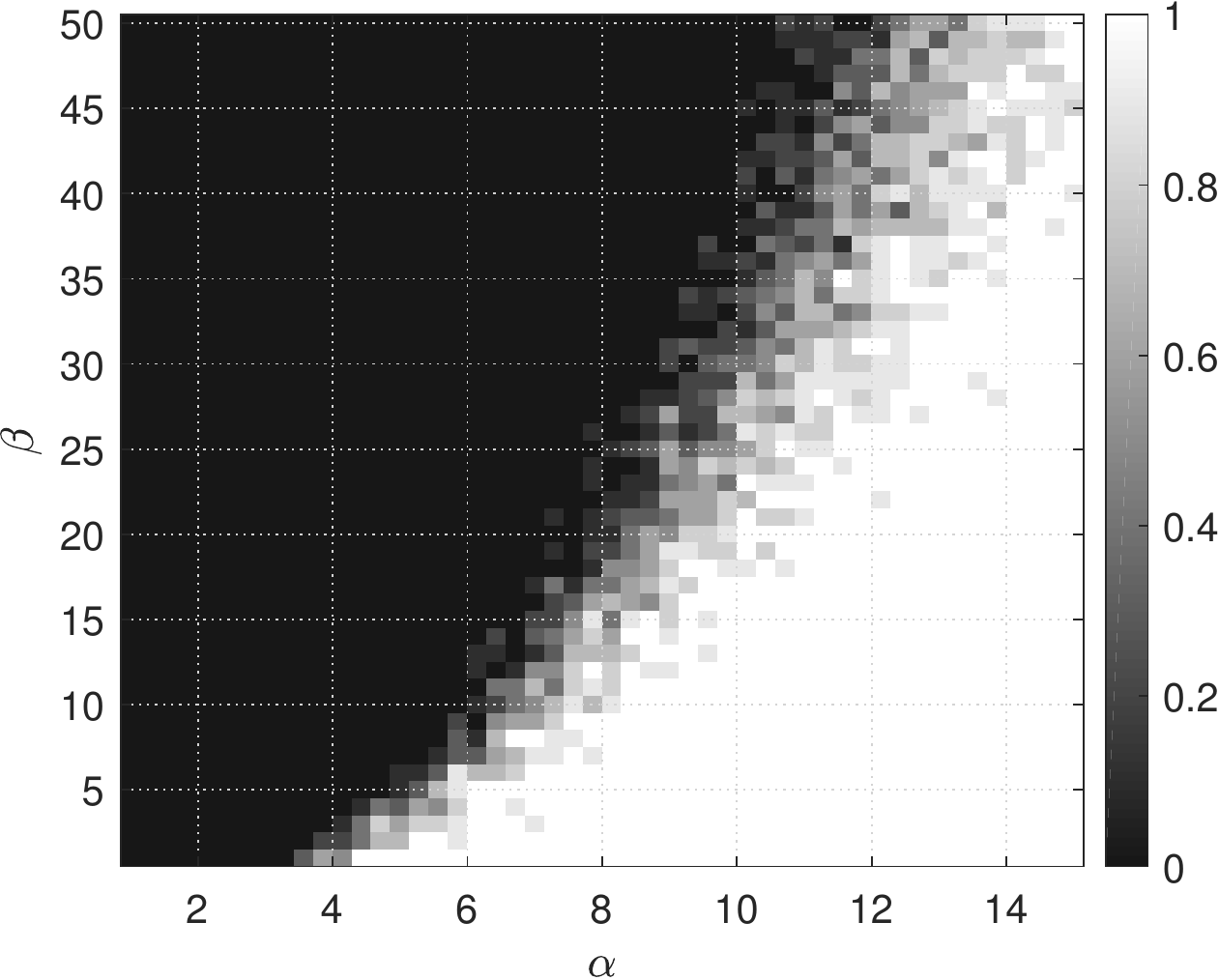}
    \label{fig:exp_noise_a}
    }\hskip 0.2cm
    \subfloat[\scriptsize{Error of synchronization, $\sigma = 0.5$}]{\includegraphics[width = 0.405\textwidth]{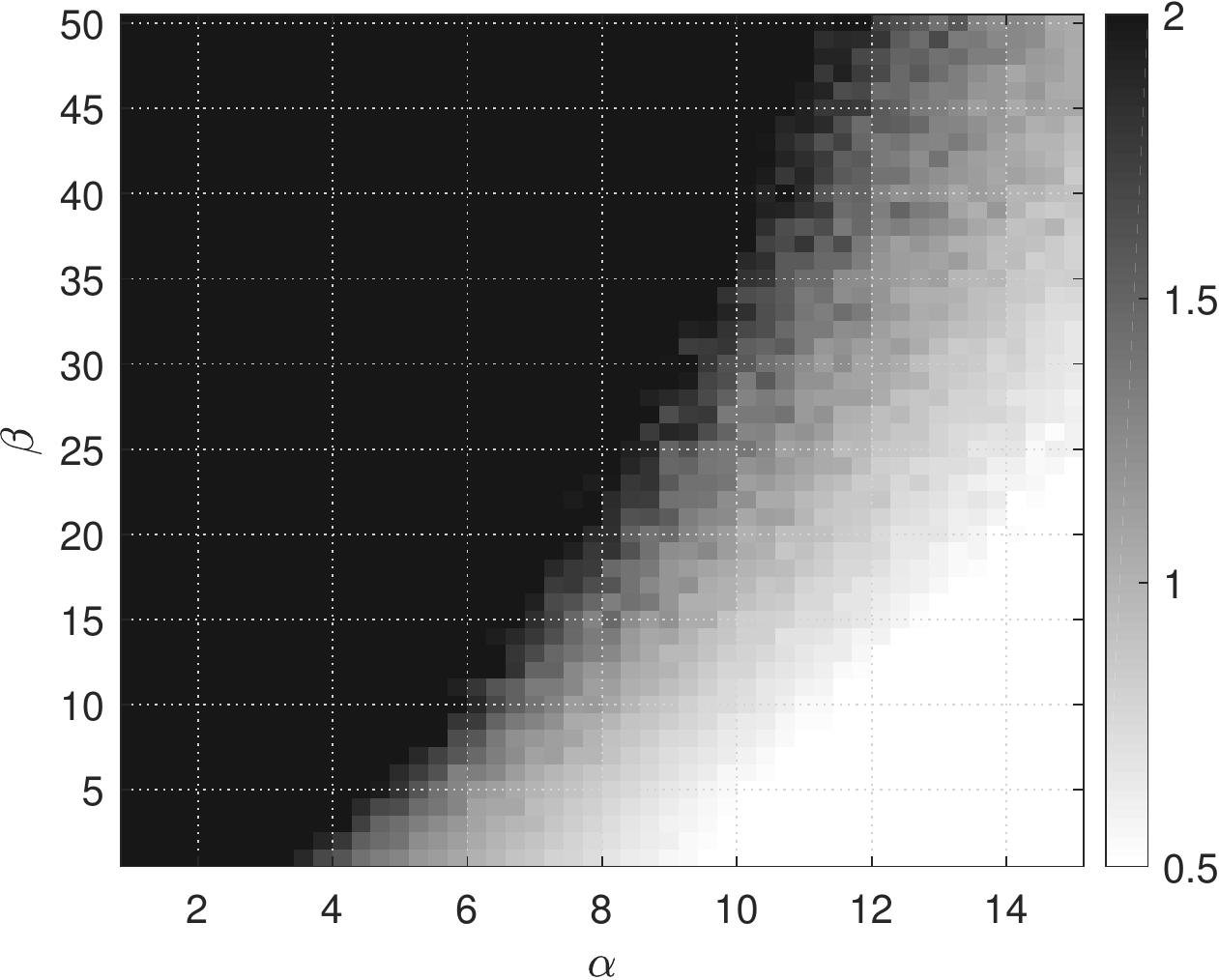}
    \label{fig:exp_noise_b}
    }\\[-2pt]
    \subfloat[\scriptsize{Success rate of exact recovery, $\sigma = 1$}]{\includegraphics[width = 0.4\textwidth]{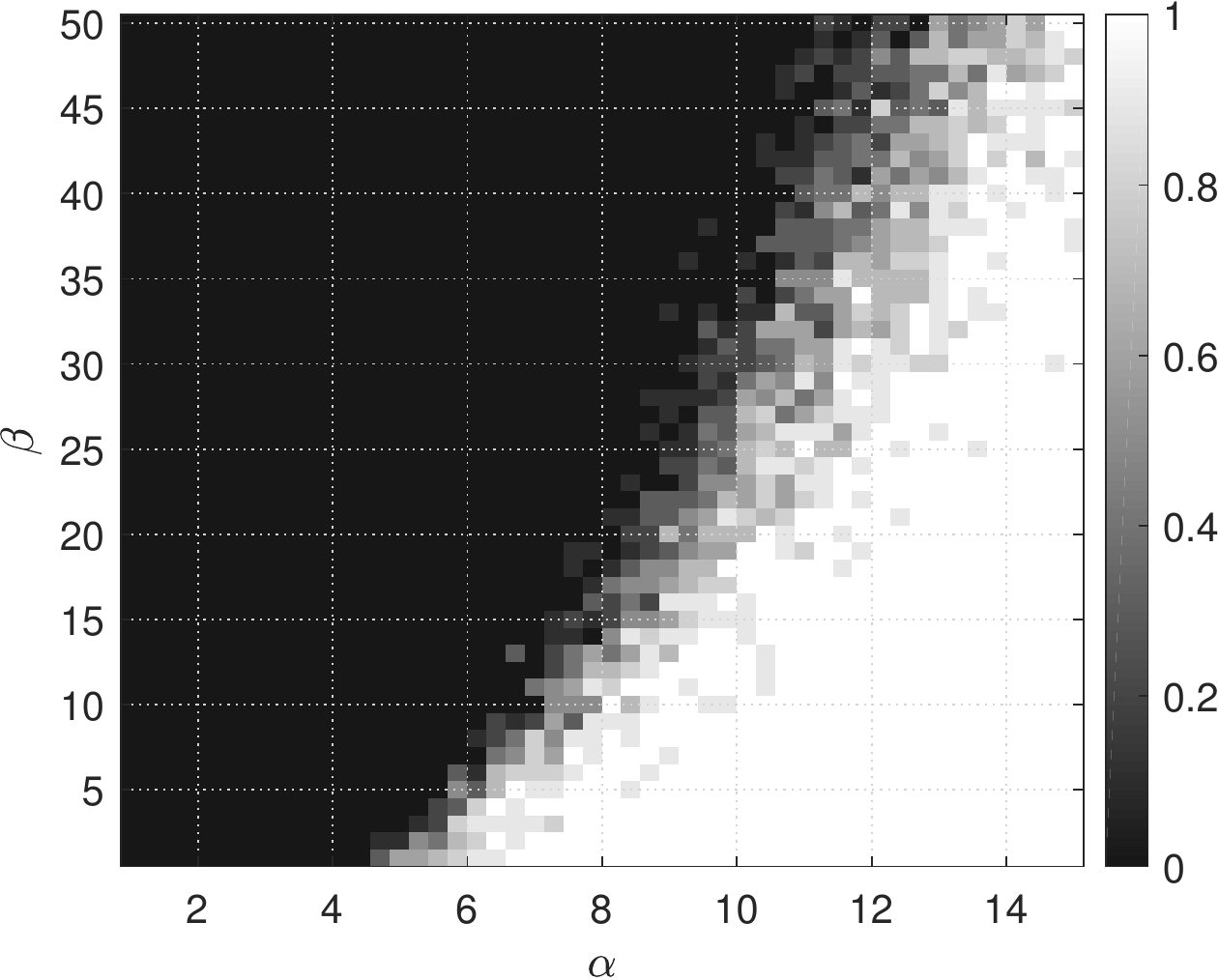}
    \label{fig:exp_noise_c}
    }\hskip 0.2cm
    \subfloat[\scriptsize{Error of synchronization, $\sigma = 1$}]{\includegraphics[width = 0.405\textwidth]{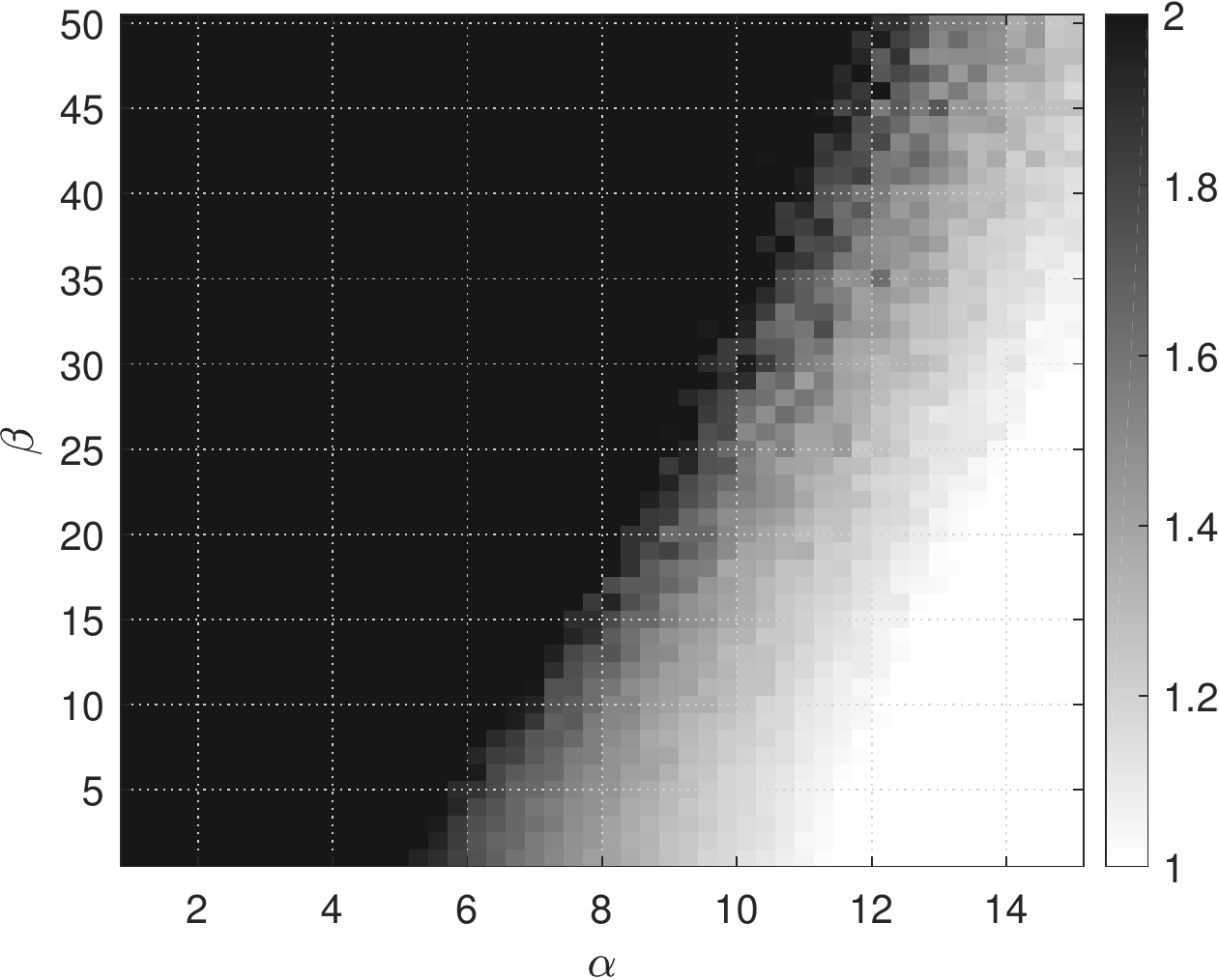}
    \label{fig:exp_noise_d}
    }\\[-2pt]
    \subfloat[\scriptsize{Success rate of exact recovery, $\sigma = 2$}]{\includegraphics[width = 0.4\textwidth]{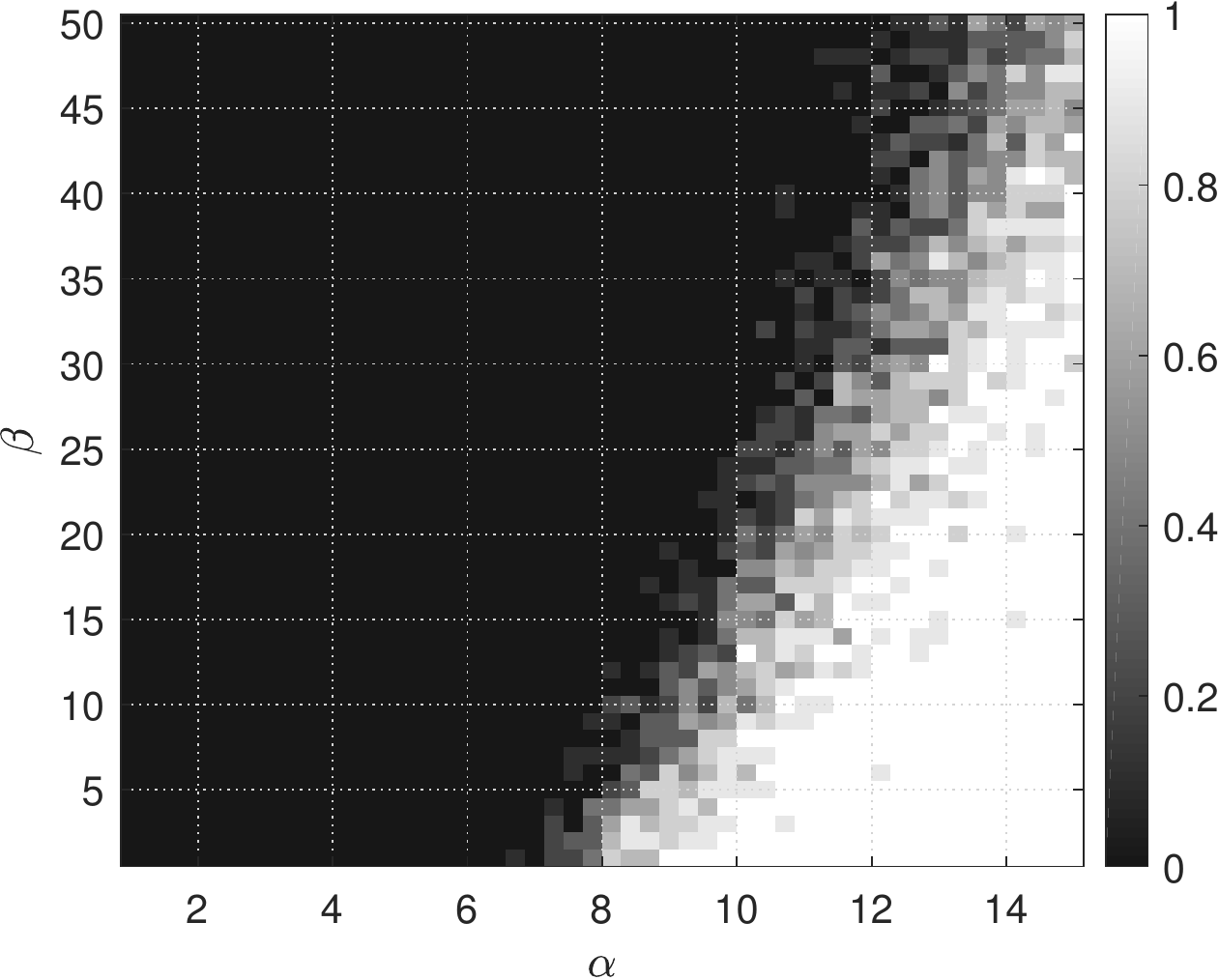}
    \label{fig:exp_noise_e}
    }\hskip 0.2cm
    \subfloat[\scriptsize{Error of synchronization, $\sigma = 2$}]{\includegraphics[width = 0.405\textwidth]{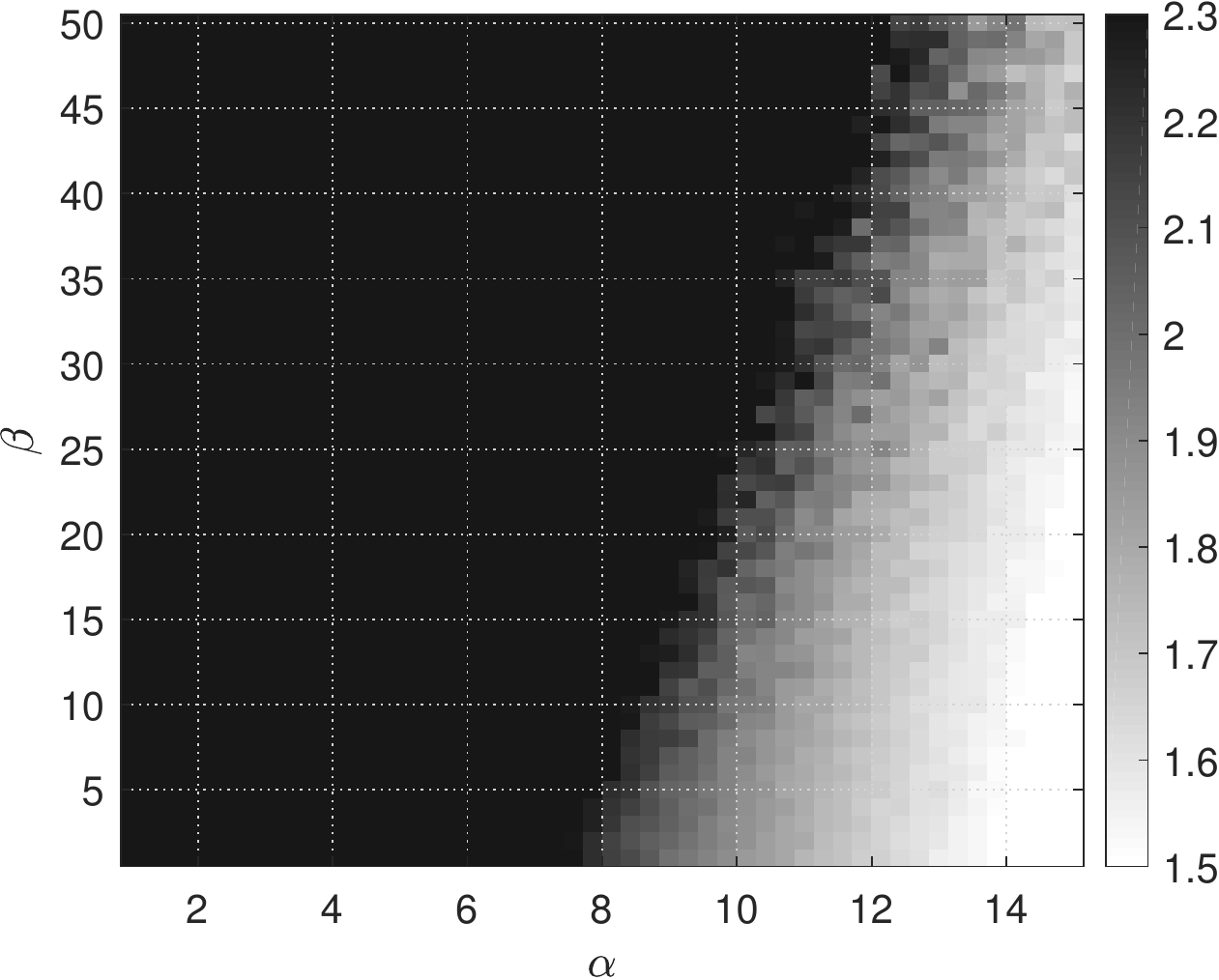}
    \label{fig:exp_noise_f}
    }
    \caption{\textit{Results on the probabilistic model with additive Gaussian noise.}
    We test under the setting $m_1 = m_2 = 500$, $d = 2$, the noise levels $\sigma \in \{0.5, 1, 2\}$.  
    \protect \subref{fig:exp_noise_a}, \protect \subref{fig:exp_noise_c}, and \protect \subref{fig:exp_noise_e}: the success rate of exact recovery by \eqref{eq:def_failure_rate}, under varying $\alpha$ in $p = \alpha\log n/n$ and $\beta$ in $q = \beta\log n/n$; \protect \subref{fig:exp_noise_b}, \protect \subref{fig:exp_noise_d}, and \protect \subref{fig:exp_noise_f}: the synchronization error by \eqref{eq:def_error_sync} under varying $\alpha$ and $\beta$.
    }
    \label{fig:exp_noise}
\end{figure}

\begin{remark}
As an example, let us consider the case of i.i.d. Gaussian noise such that all the entries in the off-diagonal blocks $\bm{W}_{ij}$ are i.i.d. Gaussian random variables with mean zero and variance $\sigma^2$. In other words, $\bm{W}_{ij} \sim \mathcal{N}(0, \sigma^2 \bm{I}_d)$ for any pair of nodes $(i,j)$ that $i \neq j$. Then, by using an existing result on the operator norm (e.g. \cite[Theorem 4.4.5]{vershynin2018high}) we have $$\|\bm{W}\| \lesssim \sigma \sqrt{nd}$$ with high probability. For the bound on $\| \widetilde{\bm{\Delta}}_{i \cdot}\bm{M}\|$, we have
\begin{equation*}
    \|\widetilde{\bm{\Delta}}_{i \cdot}\bm{M}\| \leq \|\widetilde{\bm{\Delta}}_{i \cdot}\|\|\bm{M}\| \overset{(a)}{\lesssim} \sigma \sqrt{nd} \cdot \|\bm{M}\| \overset{(b)}{\leq} \sigma n\sqrt{nd} \max_{j}\|\bm{M}_{j \cdot}\|.
\end{equation*}
where $(a)$ holds since $\|\widetilde{\bm{\Delta}}_{i \cdot}\| \lesssim \sigma\sqrt{nd}$, and $(b)$ comes from the fact that $\|\bm{M}\| \leq n\max_{j}\|\bm{M}_{j\cdot}\|$. As a result, we can set $\epsilon_0 = \sigma\sqrt{nd}$ and $\epsilon = \sigma n\sqrt{nd}$, then the exact recovery condition $\eqref{eq:p_q_assumption_new_model}$ becomes
\begin{equation*}
    \tilde{\eta} = \frac{\sqrt{(p(1-p)+q)\log (nd)} + \sigma n\sqrt{d}}{p\sqrt{n}} \leq c_0
\end{equation*}
which indicates that exact recovery is available with high probability as long as $\sigma$ is less than a certain threshold.
\label{remark:gaussian_noise}
\end{remark}

\subsection{Experiments}
In this part, we empirically test Algorithm~\ref{alg:spectral} on the model with additive Gaussian noise discussed in Remark~\ref{remark:gaussian_noise}. We following the same evaluation process as in Section~\ref{sec:exp} and the result is shown in Fig.~\ref{fig:exp_noise}. We test on the case of $m_1 = m_2 = 500$, $d = 2$ and different noise levels $\sigma \in \{0,5, 1, 2\}$. In Fig.~\ref{fig:exp_noise} we still observe sharp phase transition phenomenon on the exact recovery of cluster memberships, and the error of synchronization is also bounded when the clusters are perfectly identified. This agrees with our theoretical analysis and demonstrates the efficacy of our proposed algorithm.  

}

	\bibliographystyle{plain}
	\bibliography{ref}

\end{document}